\documentclass[10pt]{article}
\usepackage{floatrow}

\usepackage{amsfonts,amsmath,amsthm}
\usepackage{algorithm,algorithmic}
\usepackage{bm}
\usepackage[dvips]{graphicx}
\usepackage[margin=1in]{geometry}
\usepackage{setspace}

\usepackage{url}
\usepackage{subfigure}

\usepackage{mcr}
\usepackage{pkg}
\usepackage{algE}
\usepackage{thmE}

\usepackage{cite}
\usepackage{wrapfig}
\usepackage{multirow}
\usepackage{multicol}
\usepackage{ascmac}
\usepackage{comment}

\newcommand{\vxi}{\bm{x}_{i}}

\newcommand{\vhat}[1]{{\hat{\bm{#1}}}}

\newcommand{\pkA}{_{pk_A}}
\newcommand{\pkB}{_{pk_B}}

\newcommand{\argmin}{\mathop{\rm argmin}\limits}
\newcommand{\defequal}{:=}

\title{Secure Approximation Guarantee for\\Cryptographically Private Empirical Risk Minimization}
\author{
Toshiyuki Takada \\
Nagoya Institute of Technology \\
Nagoya, Aichi, Japan \\
{\tt takada.t.mllab.nit@gmail.com} \\
\and
Hiroyuki Hanada \\
Nagoya Institute of Technology \\
Nagoya, Aichi, Japan \\
{\tt hanada.hiroyuki@nitech.ac.jp} \\
\and
Yoshiji Yamada \\
Mie University \\
Tsu, Mie, Japan \\
{\tt yamada@gene.mie-u.ac.jp} \\
\and
Jun Sakuma \\
University of Tsukuba \\
Tsukuba, Ibaraki, Japan \\
{\tt jun@cs.tsukuba.ac.jp} \\
\and
Ichiro Takeuchi\thanks{Corresponding author} \\
Nagoya Institute of Technology \\
Nagoya, Aichi, Japan \\
{\tt takeuchi.ichiro@nitech.ac.jp} \\
}
\date{\today}

\begin{document}
\maketitle

\begin{abstract} 
Privacy concern has been increasingly important
in many machine learning (ML) problems.
We study empirical risk minimization (ERM) problems
under secure multi-party computation (MPC) frameworks. 
%
Main technical tools 
for MPC
have been developed based on cryptography. 
%
One of limitations
in current cryptographically private ML 
is that
it is computationally intractable to evaluate non-linear functions
such as logarithmic functions or exponential functions. 
Therefore, 
for a class of ERM problems
such as logistic regression 
in which non-linear function evaluations are required, 
one can only obtain approximate solutions. 
In this paper,
we introduce a novel cryptographically private tool called 
\emph{secure approximation guarantee (SAG)} method. 
The key property of SAG method is that, 
given an arbitrary approximate solution, 
it can provide a non-probabilistic assumption-free bound on the approximation quality 
under cryptographically secure computation framework. 
We demonstrate the benefit of the SAG method
by applying it to several problems 
including a practical privacy-preserving data analysis task on genomic and clinical information.

\end{abstract}

\section{Introduction} \label{sec:introduction}

Privacy preservation
has been increasingly important
in many machine learning (ML) tasks. 
In this paper,
we consider 
empirical risk minimizations (ERMs)
when the data is distributed among multiple parties,
and these parties are unwilling to share their data to other parties. 
For example,
if two parties have different sets of features for the same group of people, 
they might want to combine
these two datasets
for more accurate predictive model building.
On the other hand, 
due to privacy concerns or legal regulations,
these two parties might want to keep their own data private. 
The problem of learning from multiple confidential databases
have been studied under the name of
\emph{secure multi-party computation (secure MPC)}. 
This paper is motivated 
by our recent secure MPC project 
on genomic and clinical data. 
Our task is to develop a model for predicting the risk of a disease 
based on genomic and clinical information of potential patients.
The difficulty of this problem is that 
genomic information were collected in a research institute,
while
clinical information were collected in a hospital,
and both institutes do not want to share their data to others.
However,
since the risk of the disease is dependent both on genomic and clinical features, 
it is quite valuable to use both types of information
for the risk modeling. 

Various tools for secure MPC 
have been taken from cryptography, 
and 
privacy-preserving ML approaches based on cryptographic techniques have been called
\emph{cryptographically private ML}. 
A key building block of cryptographically private ML is \emph{homomorphic encryption}
by which
sum or product of two encrypted values can be evaluated
without decryption.
Many cryptographically private ML algorithms have been developed, 
e.g., 
for 
linear regression~\cite{hall2011secure,nikolaenko2013privacy} 
and
SVM~\cite{laur2006cryptographically,yu2006privacy} 
by using homomorphic encryption property. 
One of limitations in current cryptographically private ML is that 
it is computationally intractable to evaluate non-linear functions 
such as logarithmic functions or exponential functions
in homomorphic encryption framework.
Since
non-linear function evaluations
are required
in many fundamental statistical analyses
such as logistic regression, 
it is crucially important to develop a method
that can alleviate this computational bottleneck. 
One way to circumvent this issue is to 
\emph{approximate}
non-linear functions.
For example,
in Nardi {\it et al.}'s work \cite{nardi2012achieving} for secure logistic regression,
the authors proposed to approximate
a logistic function
by sum of step functions,
which can be computed under secure computation framework. 

Due to the very nature of MPC,
even after the final solution is obtained, 
the users are not allowed to access to private data. 
When the resulting solution is an approximation,
it is important for the users to be able to check its approximation quality. 
Unfortunately,
most existing cryptographically private ML method
does not have such an approximation guarantee mechanism.
Although a probabilistic approximation guarantee
was provided in the aforementioned secure logistic regression study \cite{nardi2012achieving}, 
the approximation bound derived in that work 
depends on the unknown true solution,
meaning that the users cannot make sure
how much they can trust the approximate solution. 

The goal of this paper is to develop a practical method
for secure computations of ERM problems. 
To this end, 
we introduce a novel secure computation technique 
called
\emph{secure approximation guarantee (SAG)} method. 
Given an arbitrary approximate solution of an ERM problem, 
the SAG method provides non-probabilistic assumption-free bounds 
on how far the approximate solution is away from the true solution.
A key difference of our approach with existing ones is that 
our approximation bound is
not for theoretical justification of an approximation algorithm itself, 
but for practical decision making based on a given approximate solution. 
Our approximation bound can be obtained without any information about the true solution, 
and it can be computed
with a reasonable computational cost 
under secure computation framework, 
i.e.,
without the risk of disclosing private information.

The proposed SAG method can provide
non-probabilistic bounds on a quantity depending on the true solution of the ERM problem 
under cryptographically secure computation framework,
which is valuable for making decisions when only an approximate solution is available. 
In order to develop the SAG method, 
we introduce two novel technical contributions in this paper. 
We first introduce a novel algorithmic framework for 
computing approximation guarantee
that can be applied to a class of ERM problems
whose loss function is non-linear and its secure evaluation is difficult. 
In this framework,
we use a pair of surrogate loss functions 
that bounds the non-linear loss function from below and above. 
Our second contribution is to implement these surrogate loss functions 
by piecewise-linear functions,
and show that they can be cryptographically securely computed. 
Furthermore, 
we empirically demonstrate that 
the bounds obtained by the SAG method is much tighter than the bounds in Nardi {\it et al.}'s method
\cite{nardi2012achieving}
despite the former is non-probabilistic and assumption-free.
\figurename~\ref{fig:example} 
is an illustration of 
the SAG method 
in a simple logistic regression example. 

In machine learning literature,
significant amount of works on differential privacy
\cite{dwork2006differential} 
have been recently studied. 
The objective of differential privacy is to 
disclose an information from confidential database 
without taking a risk of revealing private information in the database,
and random perturbation is main technical tool for protecting differential privacy. 
We note that
the privacy concern studied in this paper 
is rather different from 
those in differential privacy.
Although it would be interesting to study
how the latter type of privacy concerns can be handled with the approach we discussed here,
we would focus in this paper on privacy regarding cryptographically private ML. 

\begin{figure*}[tp]
\begin{center}
\begin{tabular}{cc}
\begin{minipage}[t]{0.45\hsize}
\begin{center}
\includegraphics[width = .8\textwidth]{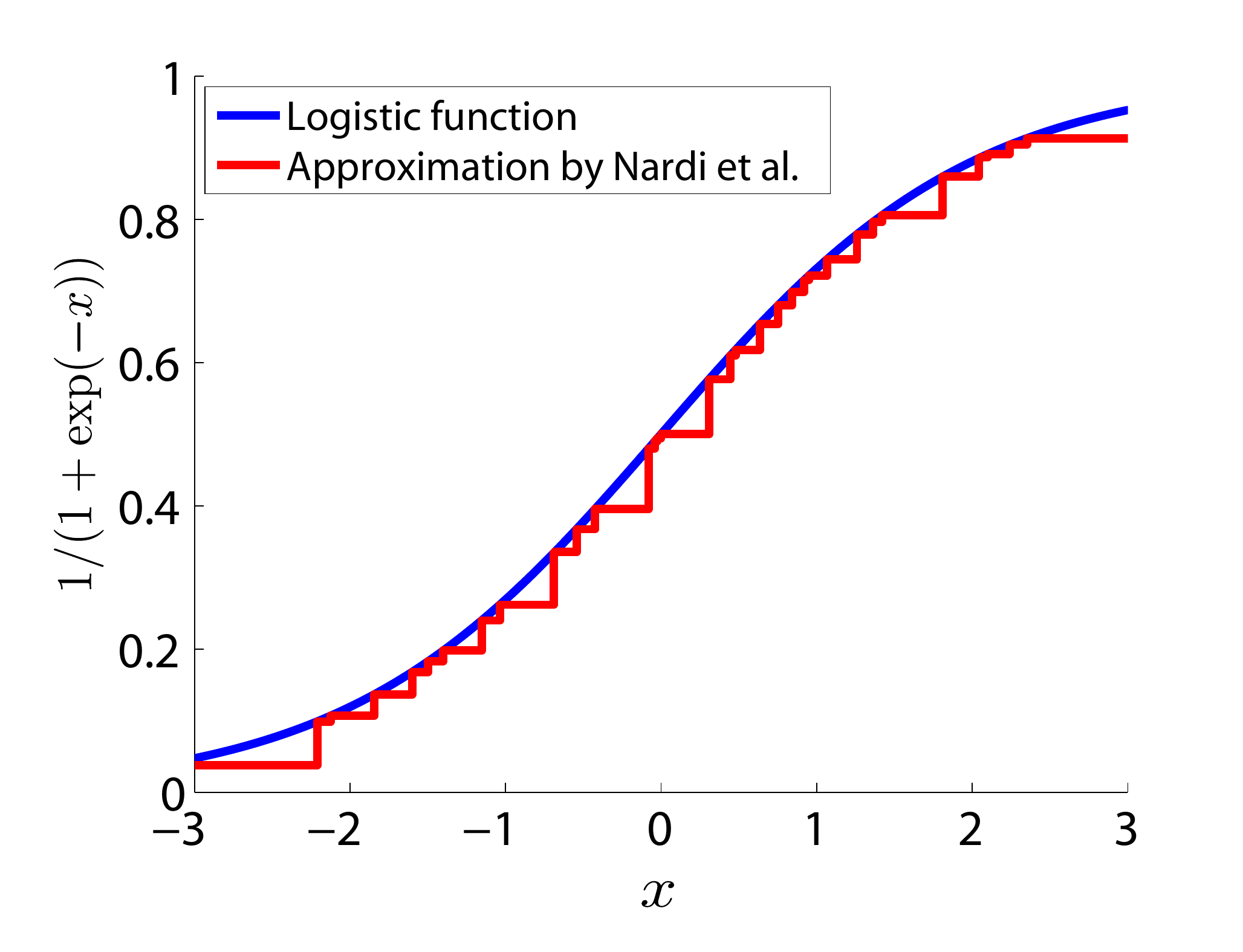}
\end{center}
(A) A non-linear function $1/(1 + \exp(-x))$ and its approximation with \cite{nardi2012achieving}
\end{minipage}
&
\begin{minipage}[t]{0.45\hsize}
\begin{center}
\includegraphics[width = 0.7\textwidth]{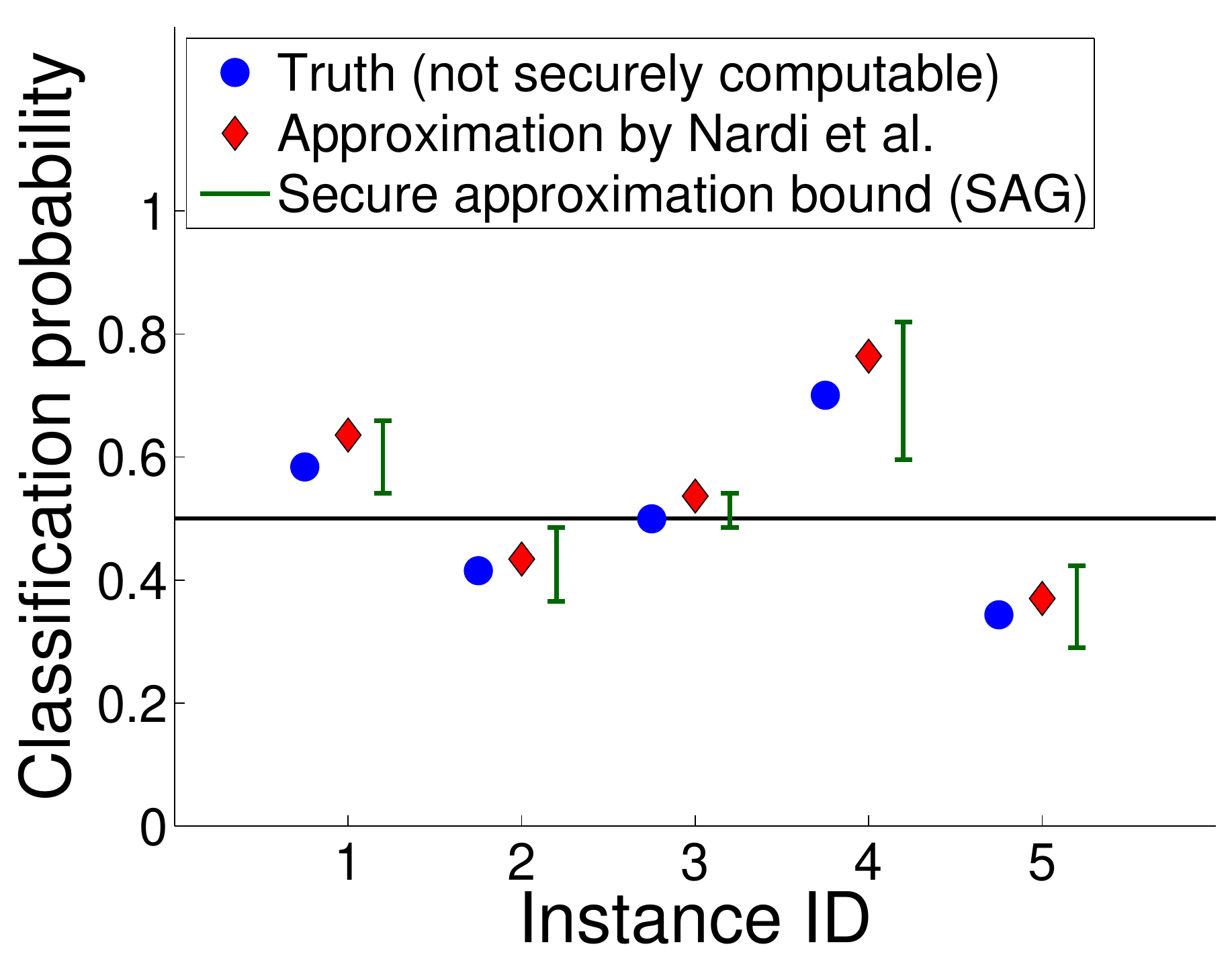}
\end{center}
(B) Class probabilities by true and approximate solutions and the bounds obtained by the SAG method.
\end{minipage}
\end{tabular}
\end{center}

\begin{minipage}{\hsize}
The left plot (A) shows 
the logistic function (blue)
and
its approximation (red) proposed in 
\cite{nardi2012achieving}. 
The right plot (B) shows 
the true (blue) and approximate (red) class probabilities 
of five training instances 
(the instance IDs $1, \ldots, 5$ are shown in the horizontal axis), 
where
the former is obtained with true logistic function, 
while 
the latter is obtained with the approximate logistic function. 

The green intervals in plot (B) are the approximation guarantee intervals 
provided by the SAG method.
The key property of the SAG method is that 
these intervals are guaranteed to contain the true class probabilities. 
Thus they can be used
for certainly classifying some of these five instances
to either positive or negative class.

Noting that 
the lower bounds of the class probabilities are greater than 0.5
in the instances 1 and 4, 
they would be certainly classified to positive class.
Similarly,
noting that 
the upper bounds of the class probabilities are smaller than 0.5
in the instances 3 and 5, 
they would be certainly classified to negative class.
\end{minipage}

\caption{
An illustration of the proposed SAG method in a simple logistic regression example
}
\label{fig:example}
\end{figure*}

\paragraph{Notations}
We use the following notations in the rest of the paper.
We denote the sets of real numbers and integers as
$\RR$
and
$\ZZ$,
respectively. 
For a natural number $N$,
we define 
$[N] := \{1, 2, \ldots, N\}$
and 
$\ZZ_{N} := \{0, 1, \ldots, N-1\}$. 
The Euclidean norm is written as
$\|\cdot\|$.
Indicator function is written as
$I_{\chi}$
i.e.,
$I_{\chi} = 1$ if $\chi$ is true, and $I_{\chi} = 0$ otherwise. 
For a protocol $\Pi$ between two parties,
we use the notation
$\Pi(\cI_A, \cI_B)\to(\cO_A, \cO_B)$,
where
$\cI_A$
and
$\cI_B$
are inputs from the parties A and B, respectively,
and
$\cO_A$ and $\cO_B$ are outputs
given to A and B,
respectively.

\section{Preliminaries}
\label{sec:Preliminary}
\subsection{Problem statement} \label{sect:problem-statement}

\paragraph{Empirical risk minimization (ERM)}

Let $\{(\vxi,y_i)\in \cX\times\cY\}_{i\in[n]}$ be the training set,
where
the input domain
$\cX \subset \RR^d$
is a compact region in $\RR^d$,
and the output domain $\cY$ is $\{-1, +1\}$ in classification problems and $\RR$ in regression problems. 
In this paper,
we consider the following class of empirical risk minimization problems: 
\begin{align}
\label{eq:optimization-problem}
\argmin_{\bm{w}}
 ~
 \frac{\lambda}{2}\|\bm{w}\|^2 +
\frac{1}{n} \sum_{i\in[n]} \ell(y_i,\vxi^\top\bm{w}),
\end{align}
where
$\ell$
is a loss function 
subdifferentiable 
and
convex 
with respect to 
$\bm{w}$, 
and
$\lambda>0$
is the regularization parameter.
$L_2$ regularization in \eq{eq:optimization-problem} 
ensures that the solution $\bm w$ is within a compact region $\cW \subset \RR^d$. 

We consider the cases where
$\ell$ is hard to compute in secure computation framework, 
i.e., 
$\ell$ includes non-linear functions such as $\log$ and $\exp$.
Popular examples includes 
logistic regression
\begin{align}
 \label{eq:Logisic-loss}
\ell(y, \bm x^\top \bm w) := \log(1+\exp(-\bm{x}^\top\bm{w}))-y\bm{x}^\top\bm{w},
\end{align}
Poisson regression
\begin{align}
 \label{eq:Poisson-loss}
\ell(y, \bm x^\top \bm w) := \exp(\bm{x}^\top\bm{w})-y\bm{x}^\top\bm{w}, 
\end{align}
and
exponential regression
\begin{align}
 \label{eq:Exponential-loss}
\ell(y, \bm x^\top \bm w) := (y\exp(-\bm{x}^\top\bm{w}))-\bm{x}^\top\bm{w}.
\end{align}

\paragraph{Secure two-party computation}
We consider secure two-party computation scenario
where 
the training set
$\{(\vxi,y_i)\}_{i\in[n]}$
is {\em vertically-partitioned}
between two parties A and B
\cite{vaidya2003privacy},
i.e.,
A and B own 
different sets of features for common set of $n$ instances.
More precisely,
let party A own the first
$d_A$ features
and
party B own 
the last $d_B$ features, 
i.e., $d_A + d_B = d$. 
We consider a scenario where 
the labels 
$\{y_i\}_{i \in [n]}$
are also owned by either party, 
and
we let party B own them here. 
We assume that
both parties can identify the instance index $i \in [n]$, 
i.e., 
it is possible for both parties to make communications with respect to a specified instance.
We denote the input data matrix owned by parties A and B as 
$X_A$
and
$X_B$,
respectively. 
Furthermore,
we denote the $n$-dimensional vector of the labels as
$\bm y := [y_1, \ldots, y_n]^\top$. 

\paragraph{Semi-honest model}
In this paper, 
we develop the SAG method 
so that it is secure 
(meaning that private data is not revealed to the other party) 
under the {\em semi-honest} model \cite{goldreich2001foundations}. 
In this security model, 
any parties are allowed to guess other party's data
as long as
they follow the specified protocol.
In other words,
we assume that all the parties do not modify the specified protocol.
The semi-honest model is standard security model in cryptographically private ML. 

\subsection{ Cryptographically Secure Computation}
\label{sect:secure-multi-party}

\paragraph{Paillier cryptosystem}
For secure computations,
we use {\em Paillier cryptosystem} \cite{paillier1999public}
as an additive {\em homomorphic encryption} tool,
i.e., 
we can obtain
$E(a+b)$
from
$E(a)$
and
$E(b)$
without decryption,
where
$a$
and
$b$
are plaintexts and
$E(\cdot)$
is the encryption function.
Paillier cryptosystem has
the {\em semantic security} \cite{goldreich2004foundations} (the {\em IND-CPA security}), which roughly means
that it is difficult to judge whether $a = b$ or $a \neq b$
by knowing 
$E(a)$
and
$E(b)$.

Paillier cryptosystem
is a public key cryptosystem
with additive homomorphism
over $\ZZ_{N}$ (i.e., ${\rm mod} N$).
In public key cryptosystem, 
the private key is two large prime numbers
$p$
and
$q$,
and the public key is
$(N, g)\in\ZZ\times\ZZ_{N^2}$,
where
$N = pq$
and
$g$
is an integer co-prime with
$N^2$.
Given a plaintext
$m\in\ZZ_N$,
a ciphertext of
$E(m)$
is obtained with a random integer
$R\in\ZZ_N$
as
follows:
\begin{align*}
E(m) = g^m R^N\mod N^2.
\end{align*}
Ciphertext 
$E(m)$
is decrypted with the private key 
whatever $R$ is chosen.
With the encryption,
the following additive homomorphism
holds for any plaintexts
$a,b \in\ZZ_N$:
\begin{align*}
&E(a)\cdot E(b) = E(a+b),\\
&E(a)^b = E(ab).
\end{align*}
Hereafter, 
we denote by 
$E\pkA(\cdot)$
and
$E\pkB(\cdot)$
the encryption functions with the public keys 
issued
by party A and B, respectively. 

Note that
we need computations of real numbers rather than integers
in data analysis tasks.
First, negative numbers can be treated with the similar technique to the two's complement. 
In order to handle real numbers, 
we multiply a magnification constant
$M$
for each input real number for expressing it with an integer.
Here,
there is a tradeoff between the accuracy and range of acceptable real number,
i.e.,
for large
$M$,
accuracy would be high,
but only possible to handle a limited range of real numbers. 

\subsection{Related works} \label{sect:existing}

The most general framework 
for cryptographically private ML is the 
Yao's garbled circuit 
\cite{yao86generate},
where any desired secure computation is expressed as an electronic circuit
with encrypted components. 
In principle, 
Yao's garbled circuit
can evaluate any function securely,
but its computational costs are usually extremely large. 
Unfortunately, 
it is impractical to use Yao's garbled circuit
for
secure computations of ERM problems. 

Nardi {\it et al.} \cite{nardi2012achieving} studied
cryptographically private approach for logistic regression. 
As briefly mentioned in \S\ref{sec:introduction},
in order to circumvent the difficulty of secure non-linear function evaluations,
the authors proposed to approximate logistic function 
by empirical cumulative density function (CDF) of logistic distributions
(see \figurename~\ref{fig:example}(A) as an example).
Denoting
the true solution
and
the approximate solution
as 
$\bm w^*$
and 
$\hat{\bm w}$,
respectively, 
the authors showed that
\begin{align}
\|\bm w^* - \hat{\bm w}\|\leq \frac{nc_1\max\|\bm{x}_i\|}{L^\gamma\lambda_{\rm min}}
&\text{\quad in probability~} 1-2\exp(-cL^{1-2\gamma}), \label{eq:nardi-bound}
\end{align}
where 
$L$ is the sample size for the empirical CDF,
$\lambda_{\rm min}$ is
the smallest eigenvalue of Fisher information matrix depending on $\bm w^*$,
and
$c > 0$, $c_1 > 0$, $\gamma \in (0, 1/2)$
are constants. 
This approximation error bound 
cannot be used 
for knowing the approximation quality of the given approximate solution
$\hat{\bm w}$:
the bound depends on the unknown true solution $\bm w^*$ because $\lambda_{\rm min}$ depends on it.
Furthermore,
in experiment section, 
we demonstrate that
the SAG method can provide much tighter non-probabilistic bounds 
than the above probabilistic bound in Nardi {\it et al.}'s method \cite{nardi2012achieving}. 

\section{Secure Approximation Guarantee(SAG)} \label{sec:secure-approximation-guarantee}
The basic idea behind the SAG method is to introduce two surrogate loss functions
$\phi$ and $\psi$
that bound the target non-linear loss function
$\ell$
from below and above.
In what follows,
we show that, 
given an arbitrary approximate solution $\hat{\bm w}$, 
if we can securely evaluate 
$\phi(\hat{\bm w})$, $\psi(\hat{\bm w})$ and a subgradient $\partial \phi / \partial \bm w \mid_{\bm w = \hat{\bm w}}$,
we can securely compute bounds on the true solution 
$\bm w^*$
which itself cannot be computed
under secure computation framework. 

First, 
the following theorem 
states that
we can obtain a ball in the solution space
in which 
the true solution $\bm w^*$ certainly exists.
\begin{theo} \label{th:bounds}
 Let $\phi: \RR \to \RR$ and $\psi: \RR \to \RR$ be functions that satisfy 
 $\phi(y, \bm x^\top \bm w) \le \ell(y, \bm x^\top \bm w) \le \phi(y, \bm x^\top \bm w)$
 $\forall y \in \cY, \bm x \in {\mathcal X}, \bm w \in {\mathcal W}$, 
 and assume that they are convex and subdifferentiable with respect to $\bm w$. 
 Then,
 for any $\hat{\bm w} \in {\mathcal W}$, 
 \begin{align*}
  \| \bm w^* - \bm m(\hat{\bm w}) \| \le r(\hat{\bm w}), 
 \end{align*}
 i.e.,
 the true solution $\bm w^*$ is located within a ball in $\cW$ 
 with the center 
\begin{align*}
\bm m(\hat{\bm w}) := \frac{1}{2}\left(\vhat{w}-\frac{1}{\lambda}\nabla\Phi(\vhat{w})\right)
\end{align*}
 and the radius 
\begin{align*}
 r(\hat{\bm w}) := \sqrt{\left\| \frac{1}{2}\left(\vhat{w}+\frac{1}{\lambda} \nabla\Phi(\vhat{w})\right) \right\|^2 +\frac{1}{\lambda}\left(\Psi(\vhat{w}) - \Phi (\vhat{w})\right)},
\end{align*}
 where
 $\Phi(\vhat{w}) := \frac{1}{n}\sum_{i\in [n]}\phi(y_i,\bm{x}_i^\top\vhat{w})$,
 $\Psi(\vhat{w}) := \frac{1}{n}\sum_{i\in [n]}\psi(y_i,\bm{x}_i^\top\vhat{w})$
 and 
 $\nabla \Phi(\hat{\bm w})$
 is a subgradient of
 $\Phi$
 at $\bm w = \hat{\bm w}$. 
\end{theo}
\noindent
The proof of Theorem~\ref{th:bounds} is presented in Appendix.

Using Theorem~\ref{th:bounds},
we can compute a pair of lower and upper bounds of any linear score in the form of
$\bm \eta^\top \bm w^*$
for an arbitrary $\bm \eta \in \RR^d$
as the following Corollary states. 
\begin{coro} \label{coro:output-bound}
 For an arbitrary $\bm \eta \in \RR^d$, 
 \begin{align}\label{eq:balltest}
  LB(\bm \eta^\top \bm w^*)
  \leq
  \bm{\eta}^\top\bm{w}^*
  \leq
  UB(\bm \eta^\top \bm w^*), 
 \end{align}
 where
 \begin{subequations} \label{eq:LB-UB}
 \begin{align} 
  LB(\bm \eta^\top \bm w^*)
  &:= \bm{\eta}^\top\bm{m}(\hat{\bm w})-\|\bm{\eta}\|r(\hat{\bm w}) \\
  UB(\bm \eta^\top \bm w^*)
  &:= \bm{\eta}^\top\bm{m}(\hat{\bm w})+\|\bm{\eta}\|r(\hat{\bm w}). 
 \end{align}
 \end{subequations}
\end{coro}
\noindent
The proof of Corollary~\ref{coro:output-bound} is presented in Appendix.

Many important quantities in data analyses are represented as a linear score.
For example,
in binary classification,
the classification result
$\tilde{y}$
of a test input
$\tilde{\bm x}$
is determined by the sign of 
the linear score
$\bm \tilde{\bm x}^\top \bm w^*$. 
It suggests that 
we can certainly classify the test instance as 
$LB(\tilde{\bm x}^\top \bm w^*) > 0 ~\Rightarrow~ \tilde{y} = +1$
and 
$UB(\tilde{\bm x}^\top \bm w^*) < 0 ~\Rightarrow~ \tilde{y} = -1$. 
%
Similarly,
if we are interested in each coefficient
$w^*_h, h \in [d], $
of the trained model,
by setting
$\bm \eta = \bm e_h$
where 
$\bm e_h$
is a $d$-dimensional vector of all 1s except 0 in the $h$-th component,
we can obtain a pair of lower and upper bounds on the coefficient as 
$LB(\bm e_h^\top \bm w^*) \le w^*_h \le UB(\bm e_h^\top \bm w^*)$. 
 
We note that 
Theorem~\ref{th:bounds} and Corollary~\ref{coro:output-bound}
are inspired by recent works on safe screening and related problems 
\cite{ghaoui2012safe,xiang2011learning,ogawa2013safe,liu2014safe,wang2014safe,xiang2014screening,fercoq2015mind,okumura2015quick},
where an approximate solution is used for bounding the optimal solution
without solving the optimization problem.

\section{SAG implementation with piecewise-linear functions}
\label{sec:piecewise-linear}

\begin{figure*}[t]
\begin{center}
\includegraphics[width=\textwidth]{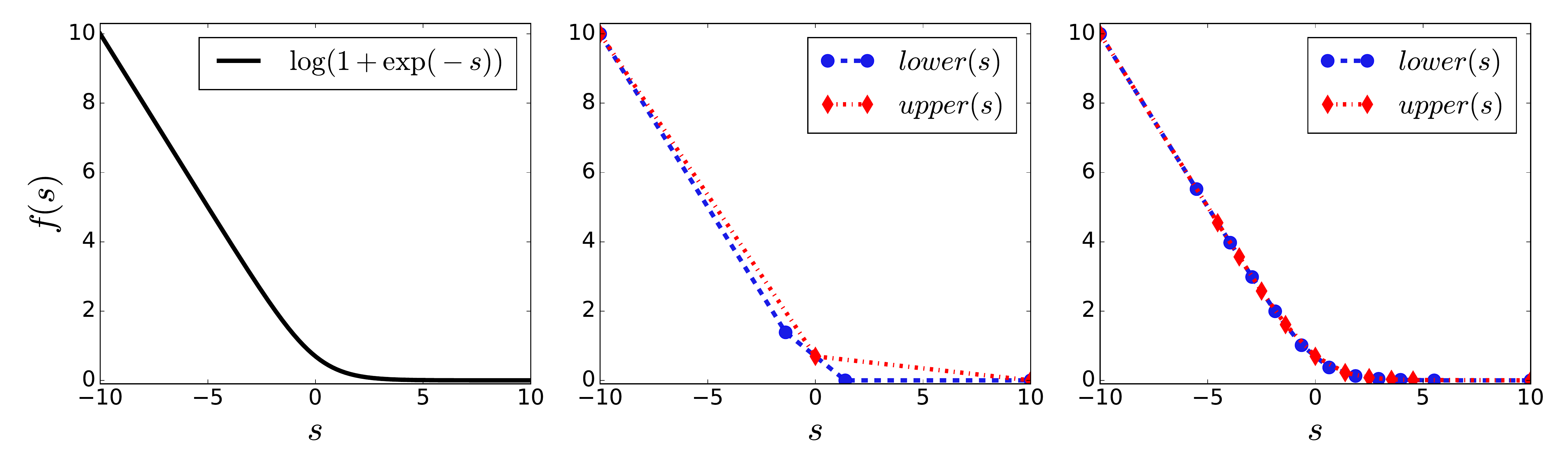}
\vspace{-1.5em}

\begin{tabular}{cccc}
\begin{minipage}{0.03\hsize}
~
\end{minipage}
&
\begin{minipage}{0.29\hsize}
A. $\log(1+\exp(-s))$
\end{minipage}
&
\begin{minipage}{0.29\hsize}
B. Bounds with $K = 2$
\end{minipage}
&
\begin{minipage}{0.29\hsize}
C. Bounds with $K = 10$
\end{minipage}
\end{tabular}
\end{center}
\vspace{-1em}
\caption{An example of bounding convex function of one variable $\log(1+\exp(-s))$ with piecewise linear functions with $K$ sections for $s\in[-10,10]$}
\label{fig:pwl-bound}
\end{figure*}

In this section,
we present how to compute the bounds on the true solution
discussed in 
\S\ref{sec:secure-approximation-guarantee} 
under secure computation framework.
Specifically,
we propose using piecewise-linear functions for the 
two surrogate loss functions 
$\phi$
and
$\psi$. 
In \S\ref{sect:piecewise-linear}, 
we present a protocol of secure piecewise-linear function evaluation ($SPL$). 
In \S\ref{sect:secure-bound},
we describe a protocol for securely computing the bounds. 
In the appendix,
we describe a specific implementation for logistic regression. 

\subsection{Secure piecewise-linear function computation} \label{sect:piecewise-linear}
Let us denote a piecewise-linear function with $K$ pieces defined in $s \in [T_0, T_K]$ as
\begin{align}
 g(s) = (\alpha_j s + \beta_j) I_{T_{j-1} \le s < T_j}, 
\end{align}
where
$\{(\alpha_j, \beta_j)\}_{j \in [K]}$
are the coefficients of the $j$-th linear segment
and 
$T_0 < T_1 < \ldots < T_{K-1} < T_K$
are breakpoints. 
For continuity,
we assume that
$\alpha_j T_j + \beta_j = \alpha_{j+1} T_j + \beta_{j+1}$ 
for all $j \in\{0, 1, \dots, K-1\}$. 

An advantage of piecewise-linear functions 
is that,
for any one-dimensional convex function, 
a lower bounding function can be easily obtained by using its tangents,
while 
an upper bounding function can be also easily obtained by using its chords. 
In addition,
we can easily control
the trade-off
between the accuracy and the computational complexity 
by changing the number of pieces $K$.
\figurename~\ref{fig:pwl-bound}
shows
examples of two piecewise-linear surrogate loss functions 
for a non-linear function 
$\log(1 + \exp(-s))$ 
for several values of $K$. 

The following theorem states that
a piecewise-linear function
$g(s)$
can be securely evaluated. 
\begin{theo} \label{th:plf-secure-computability}
 Suppose that
 party A has
 $E_{pk_B}(s_A)$
 and
 party B has 
 $E_{pk_A}(s_B)$
 such that
 $s = s_A + s_B$.
 Then,
 the two parties can securely evaluate the encrypted value of the piecewise-linear function value 
 $g(s)$
 in the sense that 
 there is a secure protocol that outputs 
 $E_{pk_B}(g_A)$
 and
 $E_{pk_A}(g_B)$
 respectively to 
 party A
 and
 party B
 such that 
 $g_A + g_B = g(s)$. 
\end{theo}
The proof of Theorem~\ref{th:plf-secure-computability} is presented in Appendix.
In the proof,
we develop such a protocol called $SPL$, 
whose input-output property is represented as 
\begin{align*}
 SPL(E_{pk_B}(s_A), E_{pk_A}(s_B))
 ~\to~
 (E_{pk_B}(g_A), E_{pk_A}(g_B)). 
\end{align*}

Let 
$o_j(s) := I_{s \in [T_{j-1}, T_j)}$,
$j \in [K]$, 
denote the indicator of an event that a scalar $s$ is in the $j$-th piece. 
The difficulty of secure piecewise-linear function evaluation is that 
we need to securely compute 
$E(o_j(s))$.
We use a protocol presented by Veugen {\it et al.} \cite{veugen2011comparing}
in order to compute
$E(I_{a < b})$
from
$E(a)$
and
$E(b)$, 
and then compute $E(o_j(s))$ as
\begin{eqnarray*}
 E(o_j(s)) = E(I_{s\geq T_{j-1}} - I_{s\geq T_j}) = E(I_{s \geq T_{j-1}})E(I_{s\geq T_j})^{-1}.
\end{eqnarray*}
Using the indicators 
$\{o_j(s)\}_{j \in [K]}$,
the piecewise-linear function value
$g(s)$
is written as 
\begin{eqnarray}
g(s) = \sum_{j\in[K]} o_j(s)(\alpha_j s + \beta_j), \label{eq:location-of-PLF-sum}
\end{eqnarray}
which can be securely computed
if 
$E(o_j(s))$
and
$E(s)$
are available. 

We finally note that, in Theorem \ref{th:bounds},
when
$\phi(s)$
is represented as a piecewise-linear function,
its subderivative
$\partial \phi(s)/\partial s$
is represented as a piecewise-constant function and so is the subgradient
$\nabla \Phi(\vhat{w})$.
We can develop a secure piecewise-constant function evaluation protocol
based on the same idea as above (detailed in the proof of Theorem \ref{th:plf-secure-computability} in Appendix).

\subsection{Secure bound computation} \label{sect:secure-bound}
We describe here how to compute the bounds on the true solution
in the form of
\eq{eq:balltest}
when the surrogate loss functions
$\phi$
and
$\psi$
are implemented with piecewise-linear functions. 
We consider a class of loss functions $\ell$ that can be decomposed as
\begin{align} \label{eq:computability-loss-function}
 \ell(y, \bm x^\top \bm w) = u(s(y, \bm x^\top \bm w)) + v(y, \bm x^\top \bm w), 
\end{align}
where
$u$ is a non-linear function whose secure evaluation is difficult,
while
$s(y, \bm x^\top \bm w)$,
$v(y, \bm x^\top \bm w)$,
and their subgradients are assumed to be securely evaluated. 
Note that most commonly-used loss functions
can be written in this form.
For example,
in the case of logistic regression \eqref{eq:Logisic-loss},
$u(s) = \log(1 + \exp(-s))$,
$s(y, \bm x^\top \bm w) = \bm x^\top \bm w$
and
$v(y, \bm x^\top \bm w) = - y \bm x^\top \bm w$. 

We consider a situation that 
two parties A and B
own
encrypted approximate solution $\hat{\bm w}$
separately for their own features,
i.e., 
parties A and B
own 
$E_{pk_B}(\hat{\bm w}_A)$
and 
$E_{pk_A}(\hat{\bm w}_B)$,
respectively,
where
$\hat{\bm w}_A$
and 
$\hat{\bm w}_B$
the first $d_A$ and the following $d_B$ components of $\hat{\bm w}$. 

\subsubsection{Secure computations of the ball}
The following theorem states that 
the center
$\bm m(\hat{\bm w})$
and
the radius 
$r(\hat{\bm w})$
can be securely computed. 
\begin{theo} \label{th:SAG}
 Suppose that
 party A has
 $X_A$
 and 
 $E_{pk_B}(\hat{\bm w}_A)$, 
 while 
 party B has 
 $X_B$,
 $\bm y$
 and 
 $E_{pk_A}(\hat{\bm w}_B)$. 
 Then,
 the two parties can securely compute 
 the center
 $\bm m(\hat{\bm w})$
 and
 the radius 
 $r(\hat{\bm w})$
 in the sense that 
 there is a secure protocol that outputs 
 $E_{pk_B}(\bm m_A(\hat{\bm w}))$
 and
 $E_{pk_B}(r_A(\hat{\bm w})^2)$
 to party A,
 and 
 $E_{pk_A}(\bm m_B(\hat{\bm w}))$
 and
 $E_{pk_A}(r_B(\hat{\bm w})^2)$
 to party B
 such that 
 $\bm m_A(\hat{\bm w}) + \bm m_B(\hat{\bm w}) = \bm m(\hat{\bm w})$
 and 
 $r_A(\hat{\bm w})^2 + r_B(\hat{\bm w})^2 = r(\hat{\bm w})^2$.
\end{theo}
We call such a protocol as secure ball computation ($SBC$) protocol.
whose input-output property is characterized as 
\begin{align*}
 \nonumber
 SBC&((X_A, E_{pk_B}(\hat{\bm w}_A)), (X_B, \bm y, E_{pk_A}(\hat{\bm w}_B))
 \\
 ~\to~
 &((E_{pk_B}(\bm m_A(\hat{\bm w})), E_{pk_B}(r_A(\hat{\bm w})^2)), (E_{pk_A}(\bm m_B(\hat{\bm w})), E_{pk_A}(r_B(\hat{\bm w})^2)))
\end{align*}
To prove Theorem~\ref{th:SAG}, 
we only describe secure computations of three components in the $SBC$ protocol. 
We omit the security analysis of the other components 
because
they can be easily derived 
from the security properties of 
Paillier cryptosystem \cite{paillier1999public},
comparison protocol \cite{veugen2011comparing}
and
multiplication protocol \cite{nissim2006communication}.
\footnote{We add that the trade-off of security strengths and computation times of Paillier cryptosystem and the comparison protocol are controlled by parameters ($N$ in {\S}\ref{sect:secure-multi-party} for Paillier cryptosystem; another parameter exists for the comparison protocol). Thus the total security depends on the weaker one of the two. The security of the multiplication protocol depends on the security of Paillier cryptosystem itself.}

\paragraph{Encrypted values of $\Psi(\vhat{w}) - \Phi(\vhat{w})$}
This quantity can be obtained 
by summing
$\psi(\bm x_i) - \phi(\bm x_i)$
for $i \in [n]$. 
Denoting 
$\phi := \underline{u}(s) + v$
and
$\psi := \overline{u}(s) + v$,
where
$\underline{u}$ and $\overline{u}$
are lower and upper bounds of
$u$
implemented with piecewise-linear functions,
respectively,
we can compute
$\psi(\bm x_i) - \phi(\bm x_i) = \overline{u} - \underline{u}$
by using SPL protocol for each of $\overline{u}$ and $\underline{u}$. 

\paragraph{Encrypted values of $\nabla \Phi(\vhat{w})$}
This quantity can be obtained by summing
$\nabla \phi$
at
$\bm w = \hat{\bm w}$.
Since
$
\nabla \phi
=
\frac{\partial}{\partial \bm w} \underline{u}(s)
+
\frac{\partial v}{\partial \bm w}
= 
\underline{u}^\prime(s) \frac{\partial s}{\partial \bm w}
+
\frac{\partial v}{\partial \bm w}
$,
its encrypted version can be written as
$
E(\nabla \phi)
=
E(\underline{u}^\prime(s) \frac{\partial s}{\partial \bm w})
E(\frac{\partial v}{\partial \bm w})
$.
Here, 
$\underline{u}^\prime(s)$
can be securely evaluated because
$\underline{u}^\prime$
is piecewise-constant function,
while 
$\frac{\partial s}{\partial \bm w}$
and 
$\frac{\partial v}{\partial \bm w}$
are securely computed from the assumption in \eq{eq:computability-loss-function}.
For computing
$E(\underline{u}^\prime(s) \frac{\partial s}{\partial \bm w})$
from
$E(\underline{u}^\prime(s))$
and
$E(\frac{\partial s}{\partial \bm w})$,
we can use the secure multiplication protocol in \cite{nissim2006communication}.

\paragraph{Encrypted value of $r(\hat{\bm w})^2$}
In order to compute this quantity,
we need the encrypted value of 
$\|\frac{1}{2}(\vhat{w}+1/\lambda\nabla\Phi)\|^2$,
which can be also computed by using the secure multiplication protocol in \cite{nissim2006communication}.

\subsubsection{Secure computations of the bounds}
Finally
we discuss here how to securely compute the upper and the lower bounds
in \eq{eq:balltest} 
from the encrypted
$\bm{m}(\hat{\bm w})$
and
$r(\hat{\bm w})^2$. 
The protocol depends on 
who owns the test instance
and
who receives the resulted bounds. 
We describe here a protocol for a particular setup where 
the test instance
$\tilde{\bm{x}}$
is owned by two parties A and B,
i.e.,
$\tilde{\bm{x}} = [\tilde{\bm{x}}_A^\top ~ \tilde{\bm{x}}_B^\top]^\top$
where
$\tilde{\bm{x}}_A$
and 
$\tilde{\bm{x}}_B$
are 
the first $d_A$ and the following $d_B$ components of 
$\tilde{\bm{x}}$, 
and
that
the lower and the upper bounds are given to either party.
Similar protocols can be easily developed for other setups. 

\begin{theo}\label{th:bound_evaluation}
 Let party A owns
 $\tilde{\bm{x}}_A$,
 $E_{pk_B}(\bm{m}_A(\hat{\bm w}))$
 and
 $E_{pk_B}(r_A(\hat{\bm w})^2)$,
 and
 party B owns 
 $\tilde{\bm{x}}_B$,
 $E_{pk_A}(\bm{m}_B(\hat{\bm w}))$
 and
 $E_{pk_A}(r_B(\hat{\bm w})^2)$,
 respectively.
 Then, 
 either party A or B can receive the lower and the upper bounds of 
 $\tilde{\bm{x}}^\top\bm{w}^*$
 in the form of \eq{eq:balltest}
 without revealing
 $\tilde{\bm{x}}_A$ and $\tilde{\bm{x}}_B$
 to the others. 
\end{theo}
\noindent
The proof of Theorem~\ref{th:bound_evaluation} is presented in Appendix. 
We note that
a party who receives bounds from the protocol
would get some information about
the center
$\bm{m}_B(\hat{\bm w})$
and
the radius
$\bm{r}_B(\hat{\bm w})$,
but no other information about the original dataset is revealed.

\section{Experiments} \label{sec:experiments}
We conducted experiments for illustrating the performances of the proposed SAG method.
The experimental setup is as follows.
We used
Paillier cryptosystem
with $N=1024$-bit public key
and
comparison protocol by Veugen {\it et al.} \cite{veugen2011comparing}
for 60 bits of integers.
The program is implemented with Java,
and the communications
between
two parties are
implemented with sockets
between two processes
working in the same computer.
We used a single computer
with 3.07GHz Xeon CPU and 48GB RAM. 
Except when we investigate computational costs, 
computations
were
done on unencrypted values.
Note that 
the proposed SAG method provide bounds on the true solution
$\bm w^*$
based on an arbitrary approximate solution
$\hat{\bm w}$. 
In all the experiments presented here, 
we used approximate solutions obtained by 
Nardi {\it et al.}'s approach \cite{nardi2012achieving}
as the approximate solution
$\hat{\bm w}$. 
In what follows,
we call the bounds or intervals obtained by the SAG method as
SAG bounds and SAG intervals,
respectively.

\subsection{Logistic regression}

\begin{table}[t]
\caption{Data sets used for the logistic regression. All are from UCI Machine Learning Repository.}
\label{tab:dataset}
\begin{center}
\begin{tabular}{c|rrr}\hline
data set&training set&validation set&$d$\\\hline\hline
Musk&3298&3300&166\\
MGT&9510&9510&10\\
Spambase&2301&2301&57\\
OLD&1268&1268&72\\\hline
\end{tabular}
\end{center}
\vspace{-2em}
\end{table}

\begin{figure*}[tp]
\begin{center}
\begin{tabular}{ccc}
\begin{minipage}{0.3\hsize}
\begin{center}
\includegraphics[width = \textwidth]{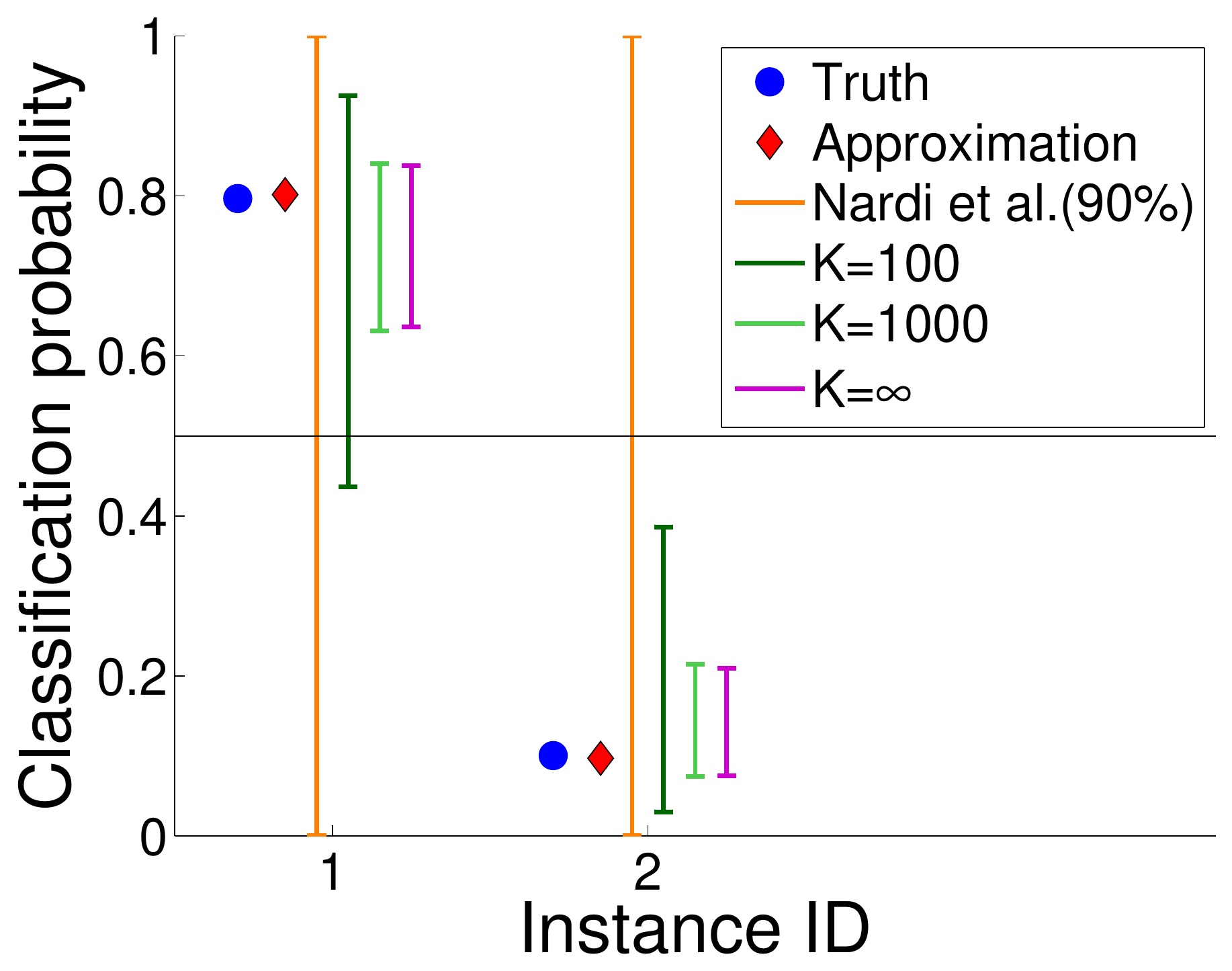}

Musk, $\lambda = 0.1$
\end{center}
\end{minipage}
&
\begin{minipage}{0.3\hsize}
\begin{center}
\includegraphics[width = \textwidth]{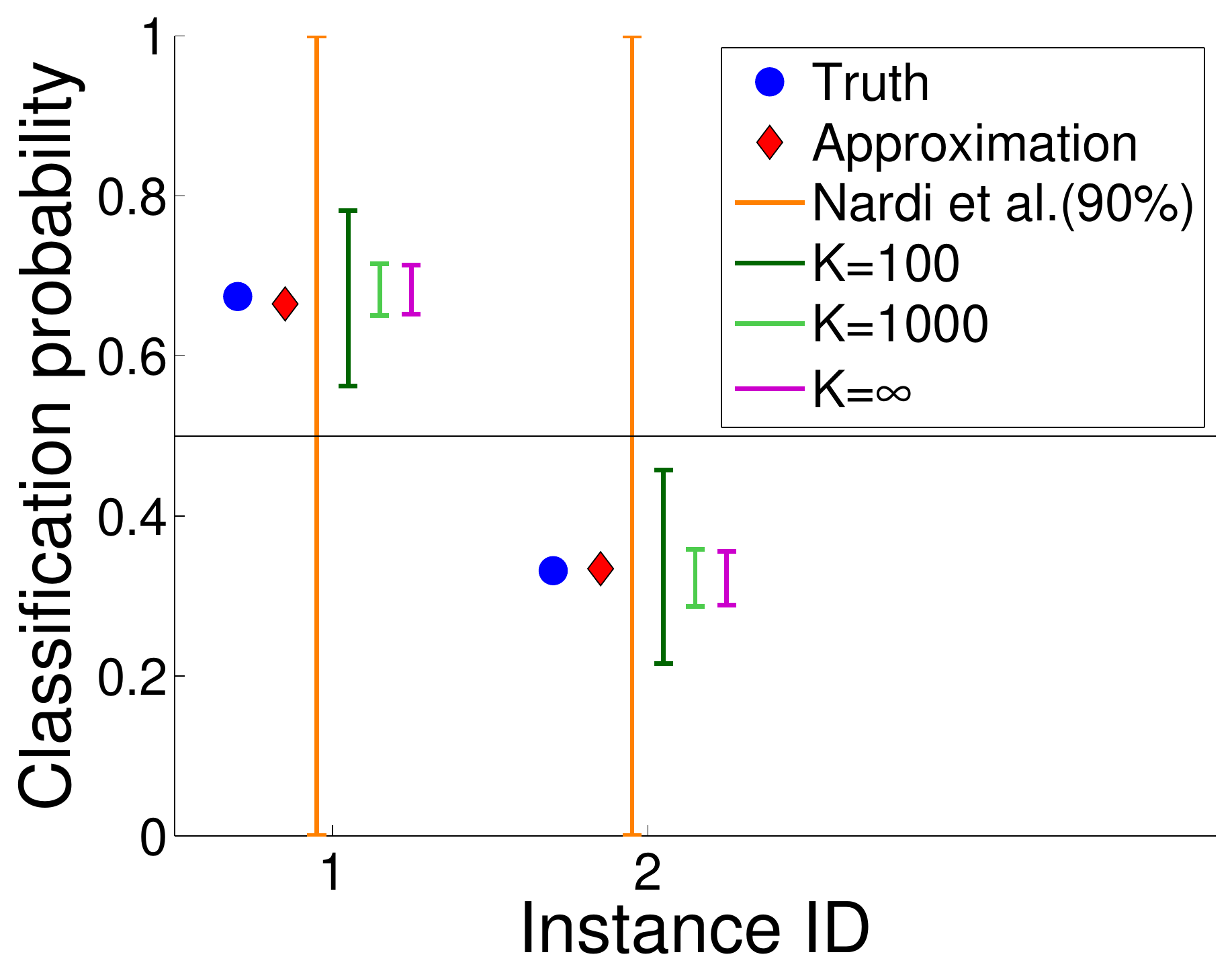}

Musk, $\lambda = 1$
\end{center}
\end{minipage}
&
\begin{minipage}{0.3\hsize}
\begin{center}
\includegraphics[width = \textwidth]{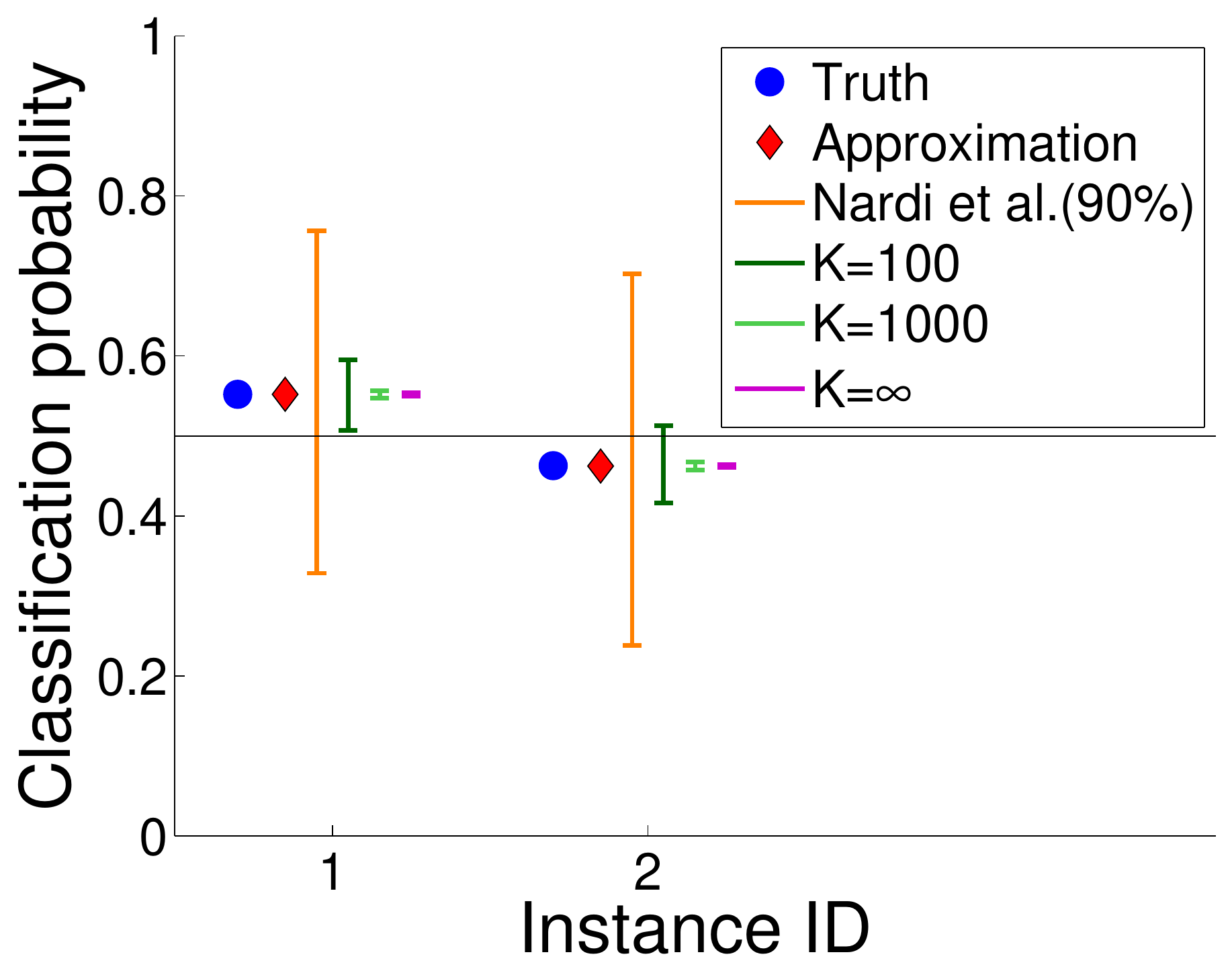}

Musk, $\lambda = 10$
\end{center}
\end{minipage}
\\
\begin{minipage}{0.3\hsize}
\begin{center}
\includegraphics[width = \textwidth]{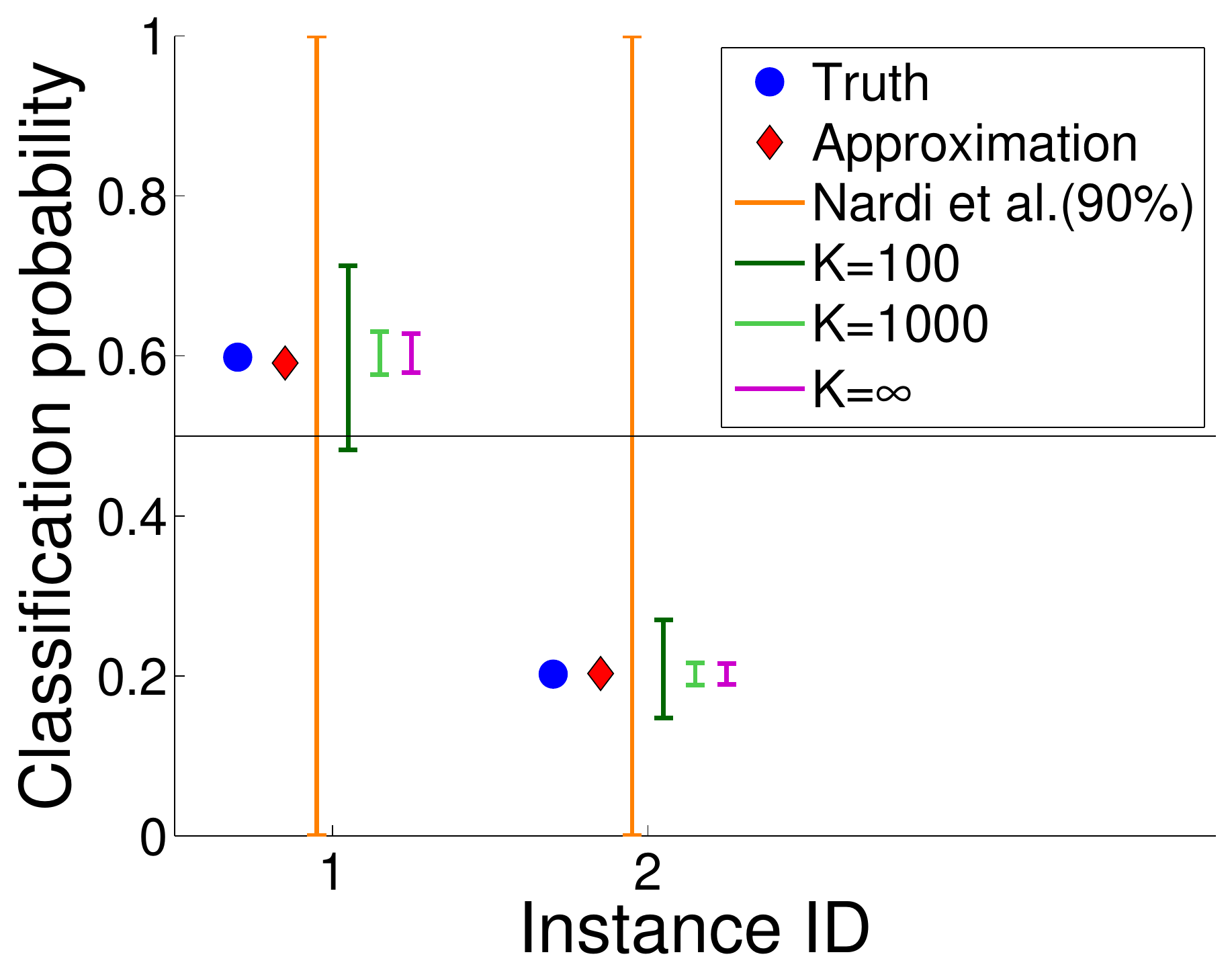}

MGT, $\lambda = 0.1$
\end{center}
\end{minipage}
&
\begin{minipage}{0.3\hsize}
\begin{center}
\includegraphics[width = \textwidth]{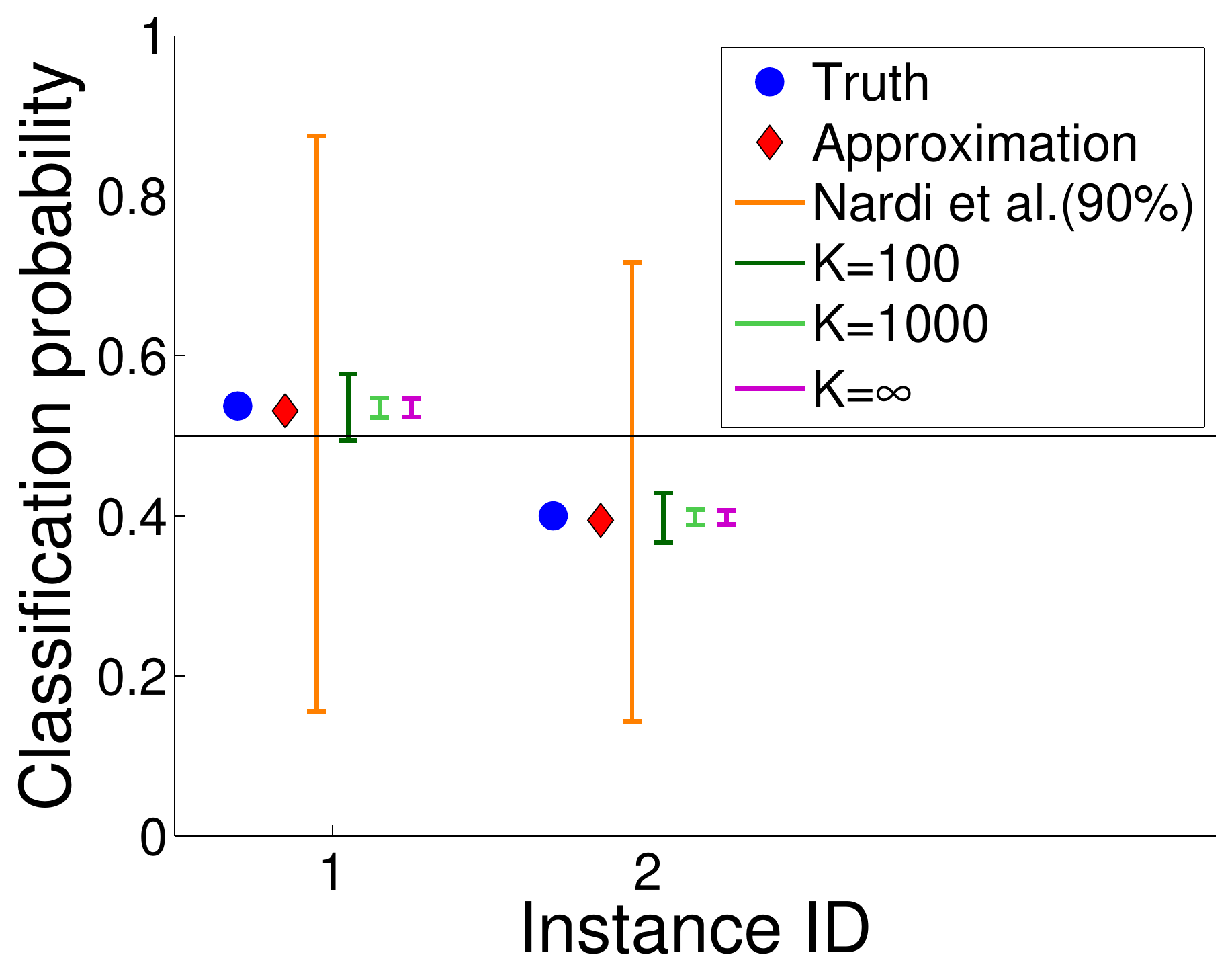}

MGT, $\lambda = 1$
\end{center}
\end{minipage}
&
\begin{minipage}{0.3\hsize}
\begin{center}
\includegraphics[width = \textwidth]{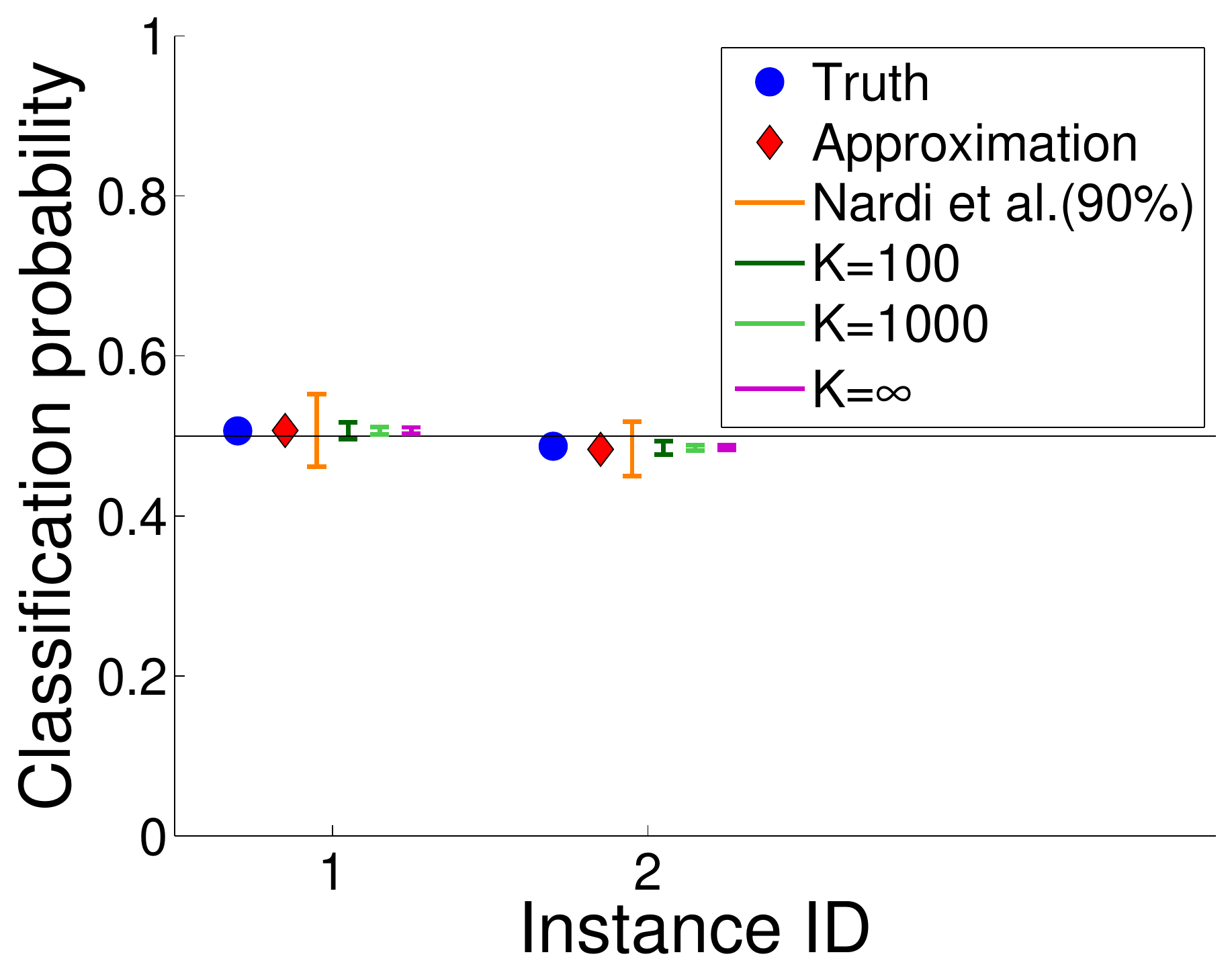}

MGT, $\lambda = 10$
\end{center}
\end{minipage}
\\
\begin{minipage}{0.3\hsize}
\begin{center}
\includegraphics[width = \textwidth]{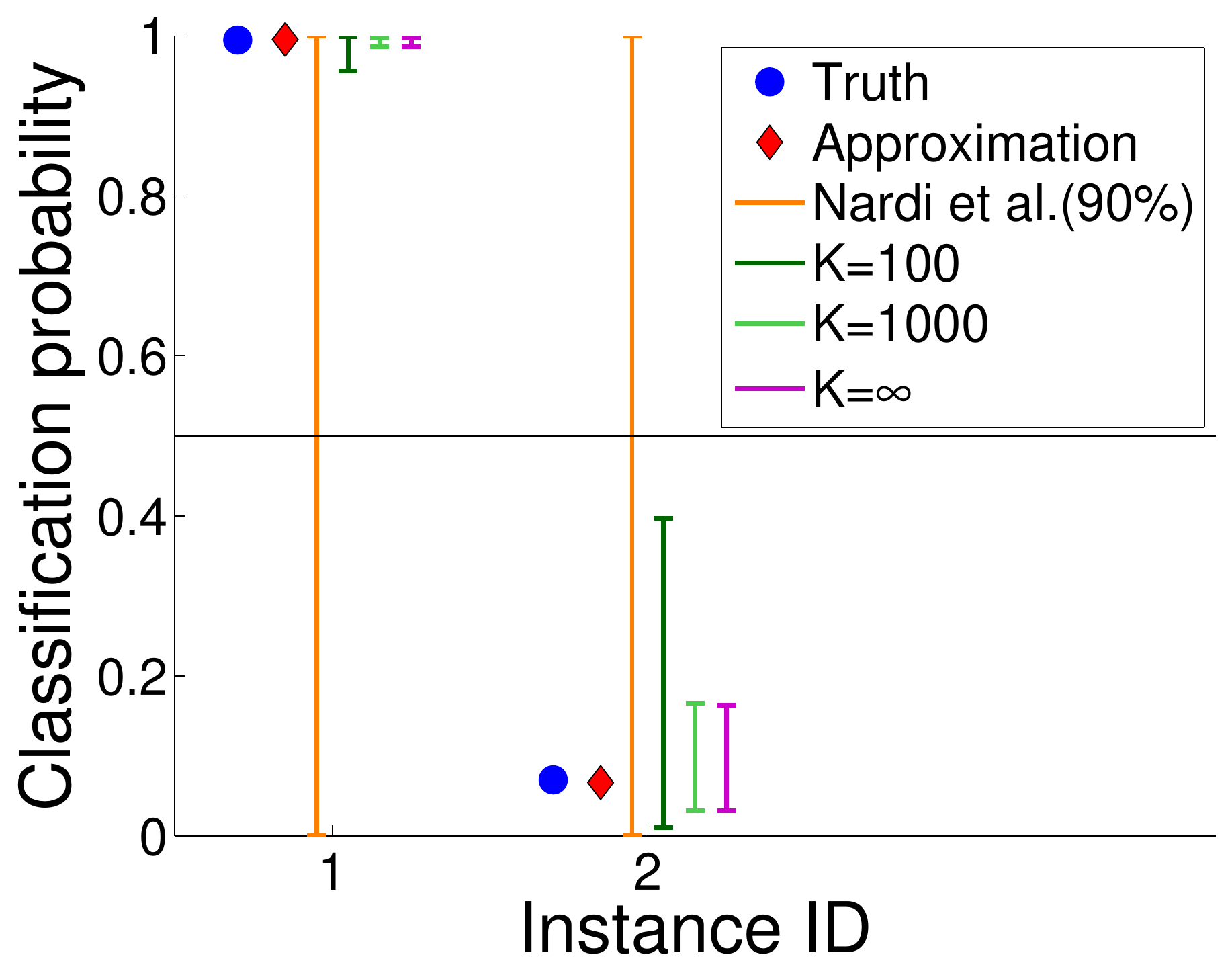}

Spambase, $\lambda = 0.1$
\end{center}
\end{minipage}
&
\begin{minipage}{0.3\hsize}
\begin{center}
\includegraphics[width = \textwidth]{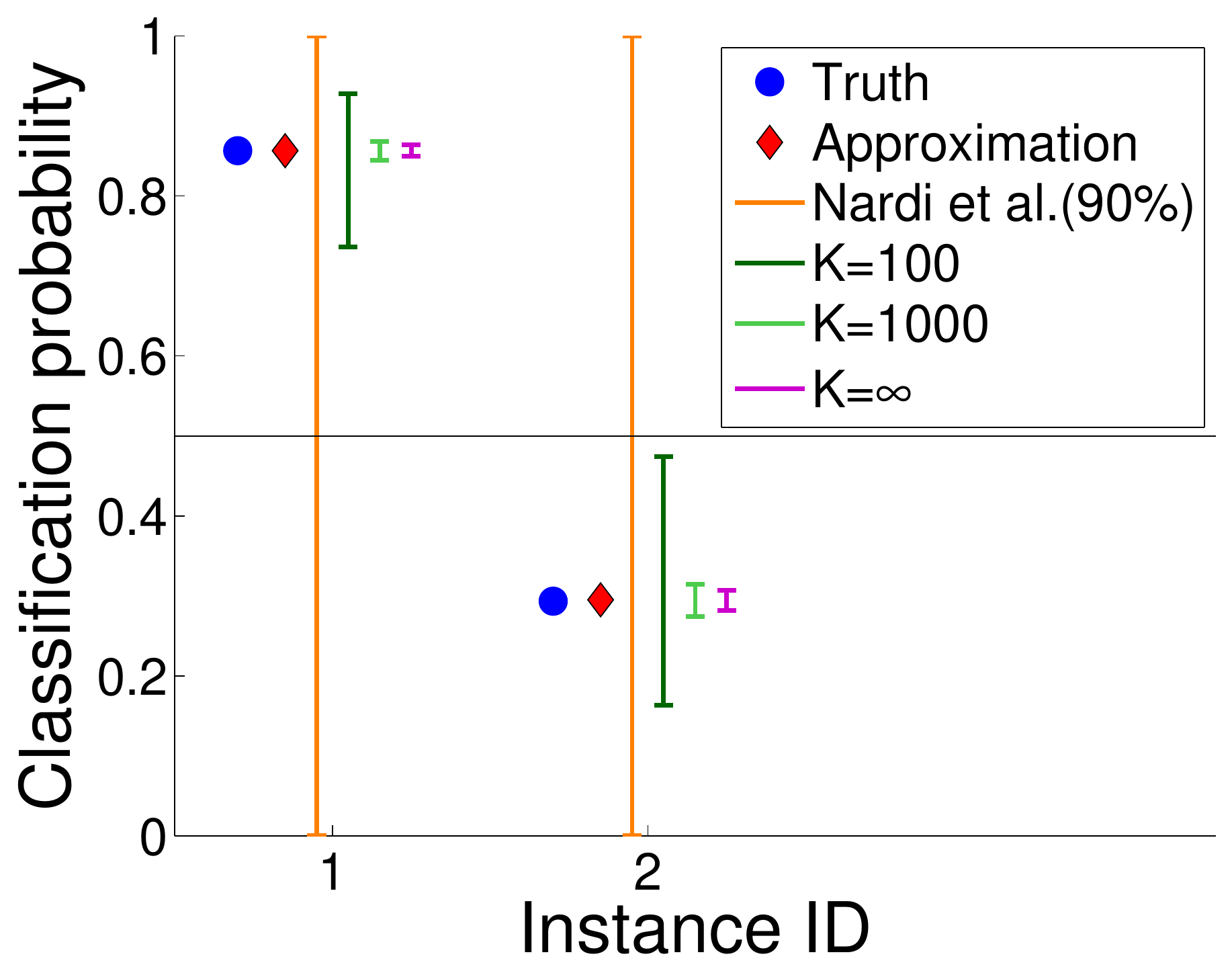}

Spambase, $\lambda = 1$
\end{center}
\end{minipage}
&
\begin{minipage}{0.3\hsize}
\begin{center}
\includegraphics[width = \textwidth]{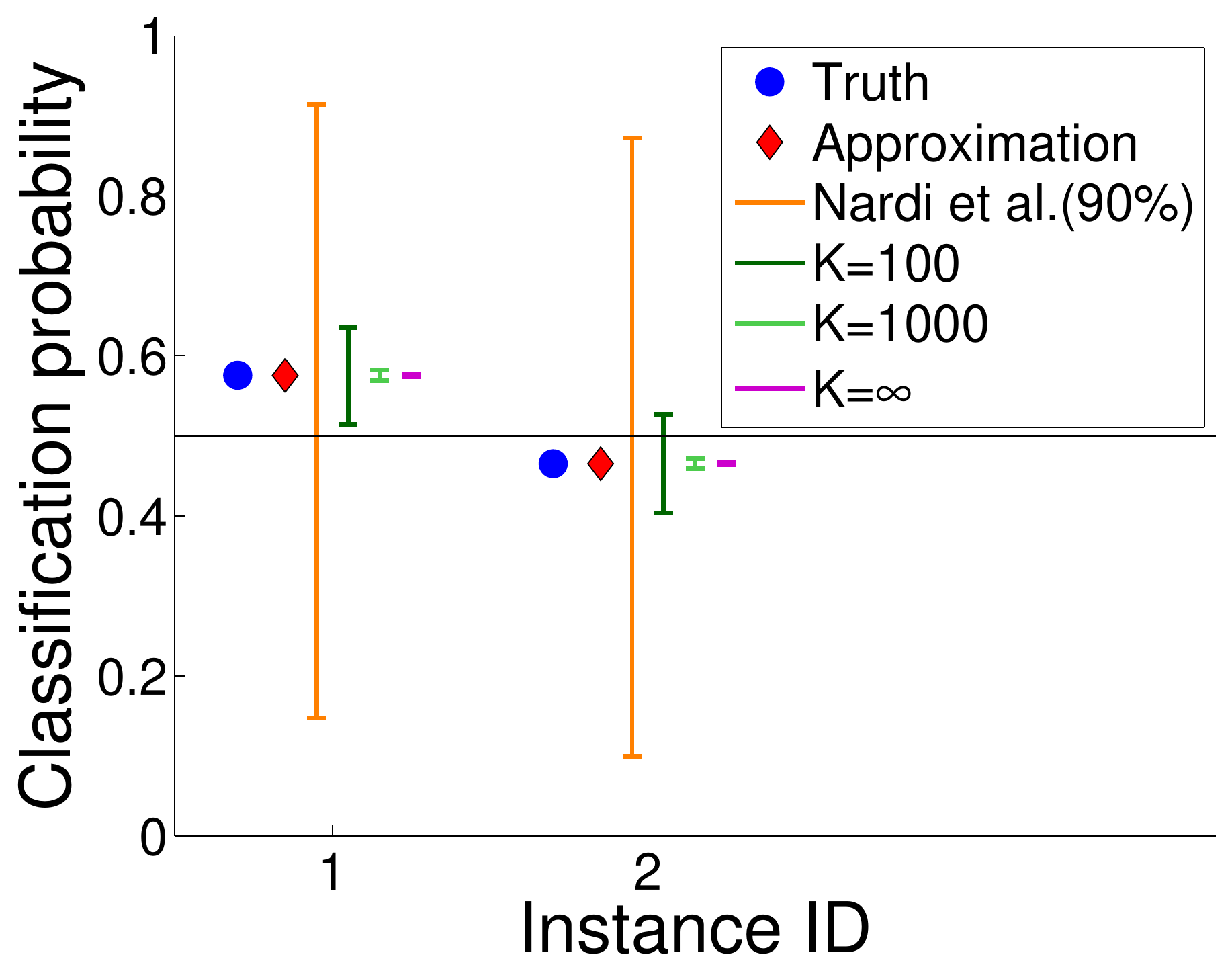}

Spambase, $\lambda = 10$
\end{center}
\end{minipage}
\\
\begin{minipage}{0.3\hsize}
\begin{center}
\includegraphics[width = \textwidth]{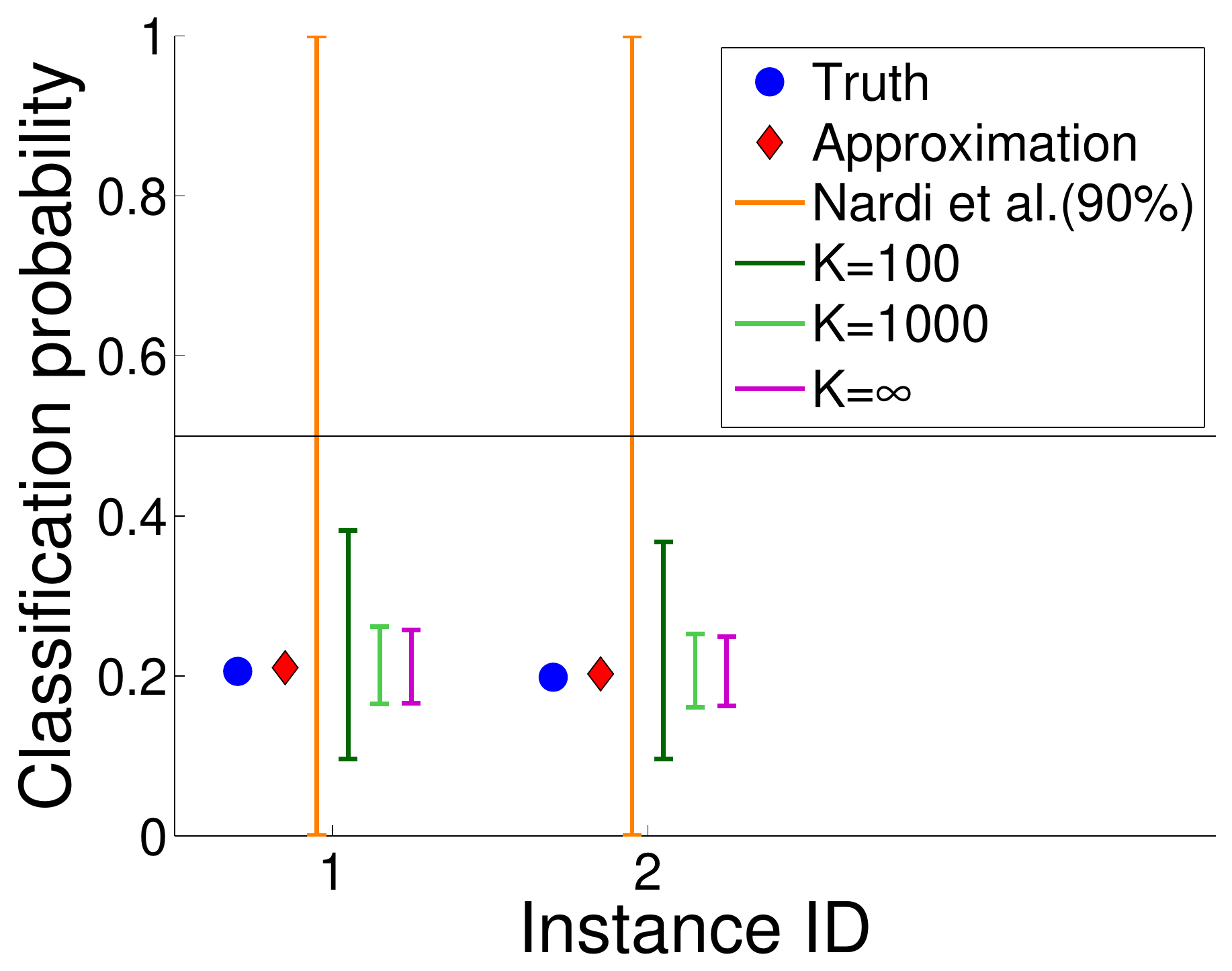}

OLD, $\lambda = 0.1$
\end{center}
\end{minipage}
&
\begin{minipage}{0.3\hsize}
\begin{center}
\includegraphics[width = \textwidth]{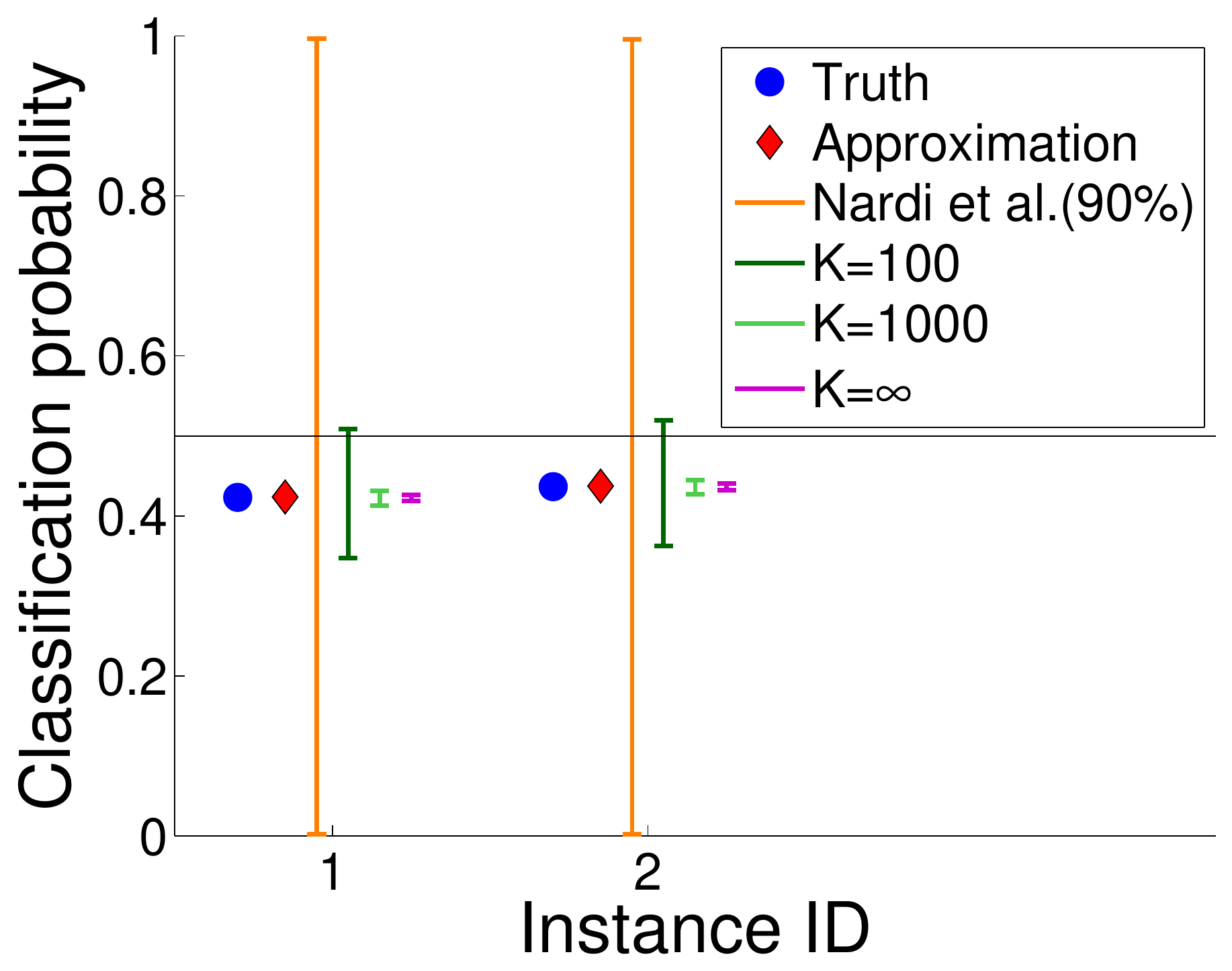}

OLD, $\lambda = 1$
\end{center}
\end{minipage}
&
\begin{minipage}{0.3\hsize}
\begin{center}
\includegraphics[width = \textwidth]{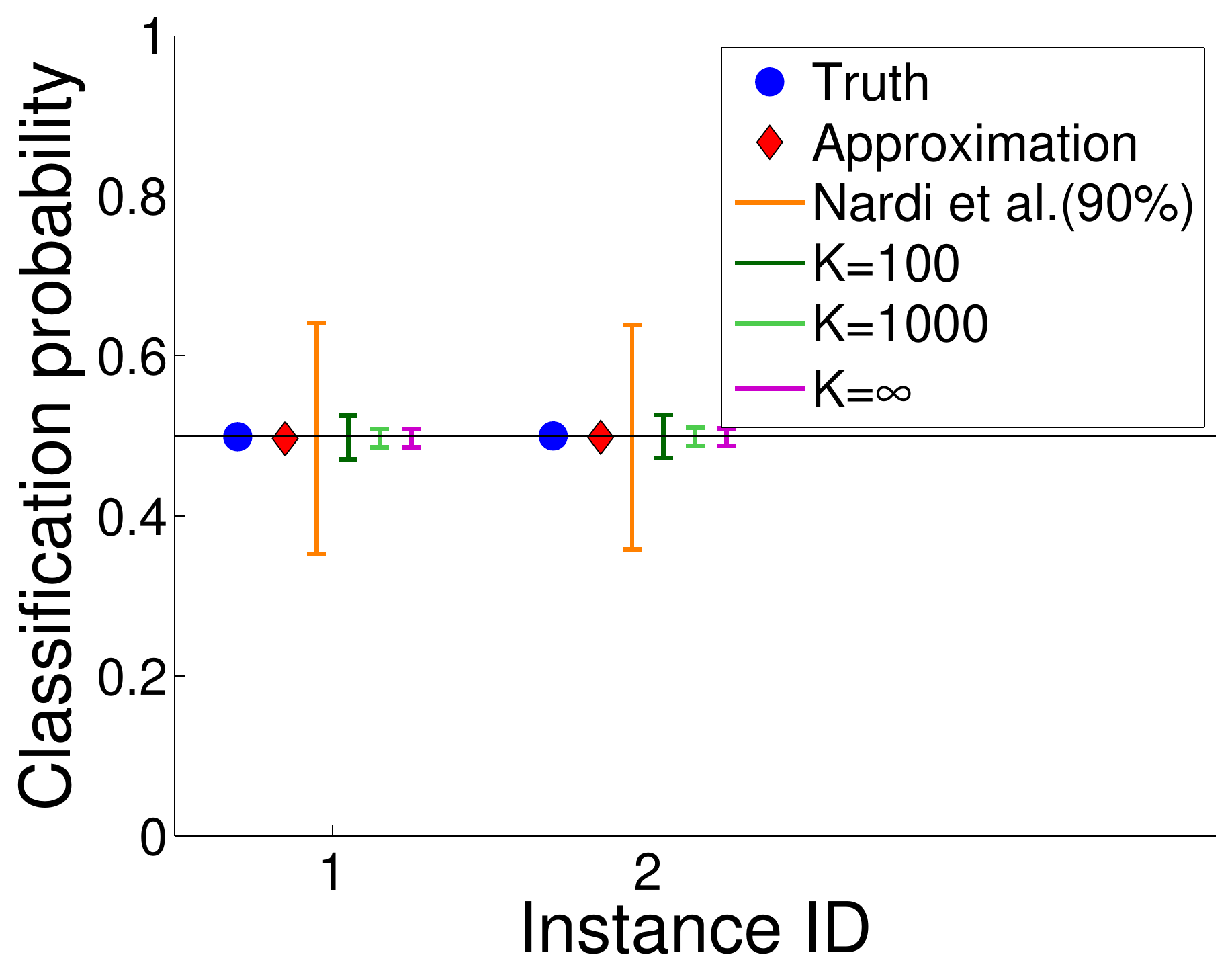}

OLD, $\lambda = 10$
\end{center}
\end{minipage}
\end{tabular}
\end{center}
\caption{The result of proposed bounds for some test instances}
\label{fig:exp1-all}
\end{figure*}

\begin{figure*}[tp]
\begin{center}
\begin{tabular}{ccc}
\begin{minipage}{0.3\hsize}
\begin{center}
\includegraphics[width = \textwidth]{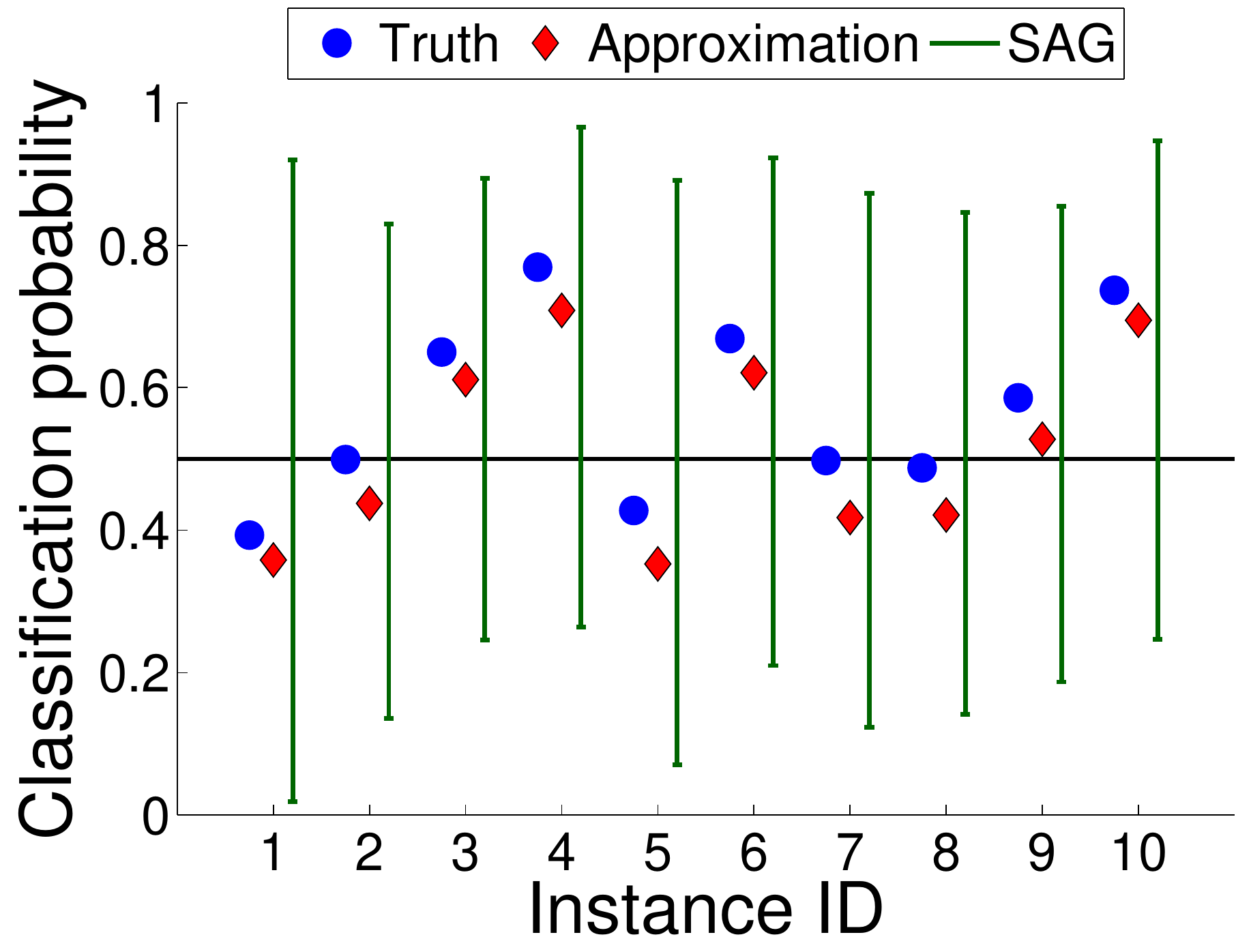}

Musk, $L=10$
\end{center}
\end{minipage}
&
\begin{minipage}{0.3\hsize}
\begin{center}
\includegraphics[width = \textwidth]{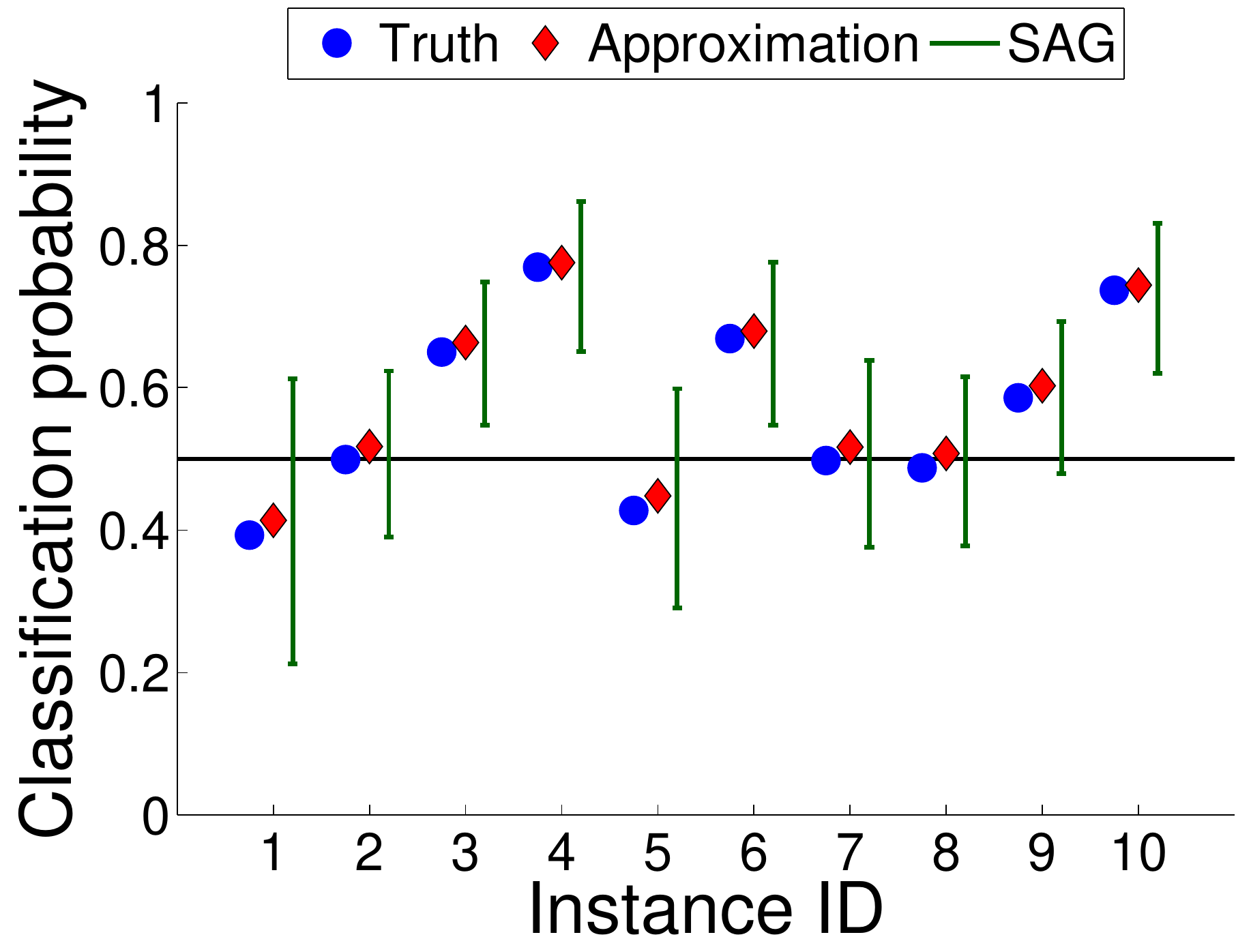}

Musk, $L=100$
\end{center}
\end{minipage}
&
\begin{minipage}{0.3\hsize}
\begin{center}
\includegraphics[width = \textwidth]{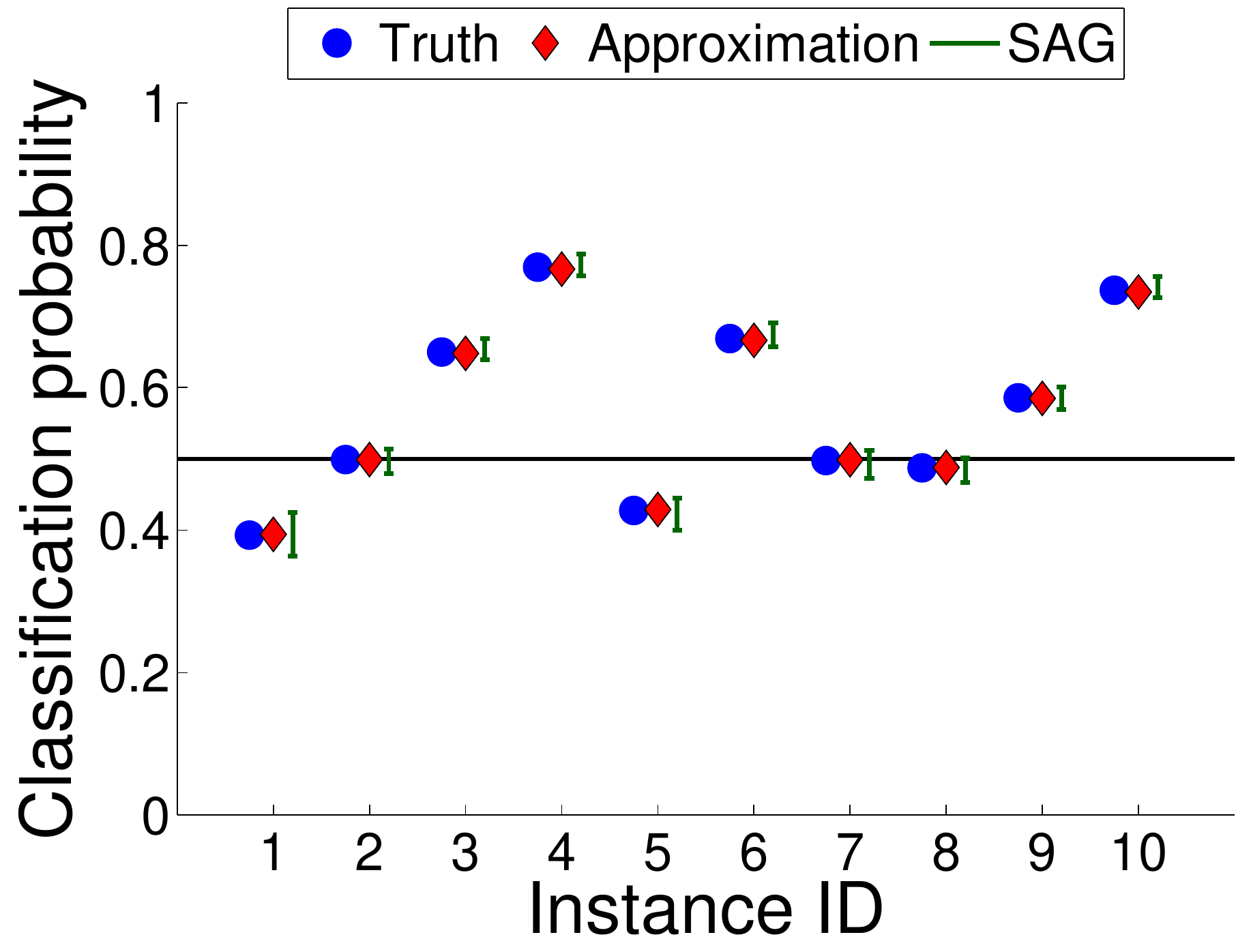}

Musk, $L=1000$
\end{center}
\end{minipage}
\\
\begin{minipage}{0.3\hsize}
\begin{center}
\includegraphics[width = \textwidth]{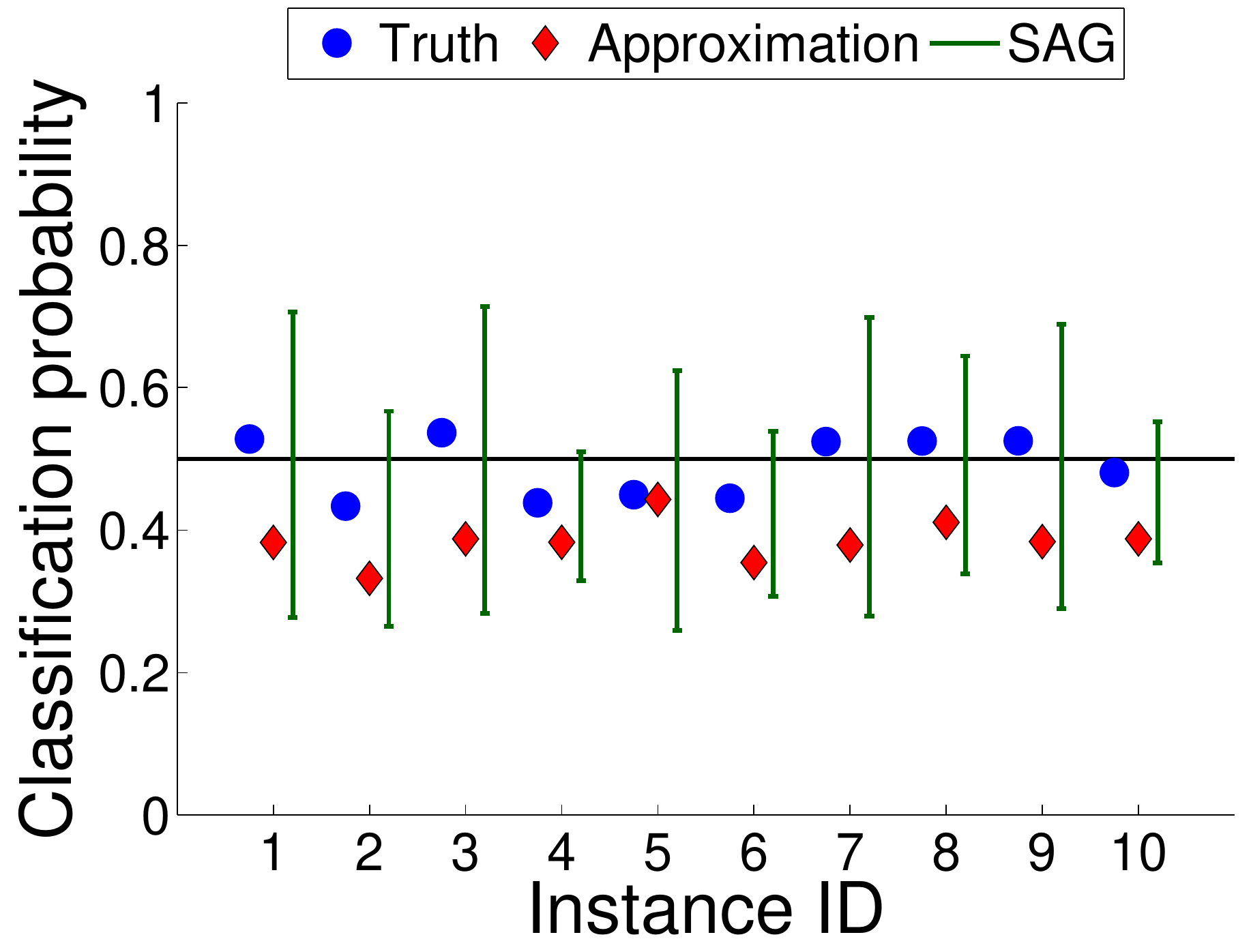}

MGT, $L=10$
\end{center}
\end{minipage}
&
\begin{minipage}{0.3\hsize}
\begin{center}
\includegraphics[width = \textwidth]{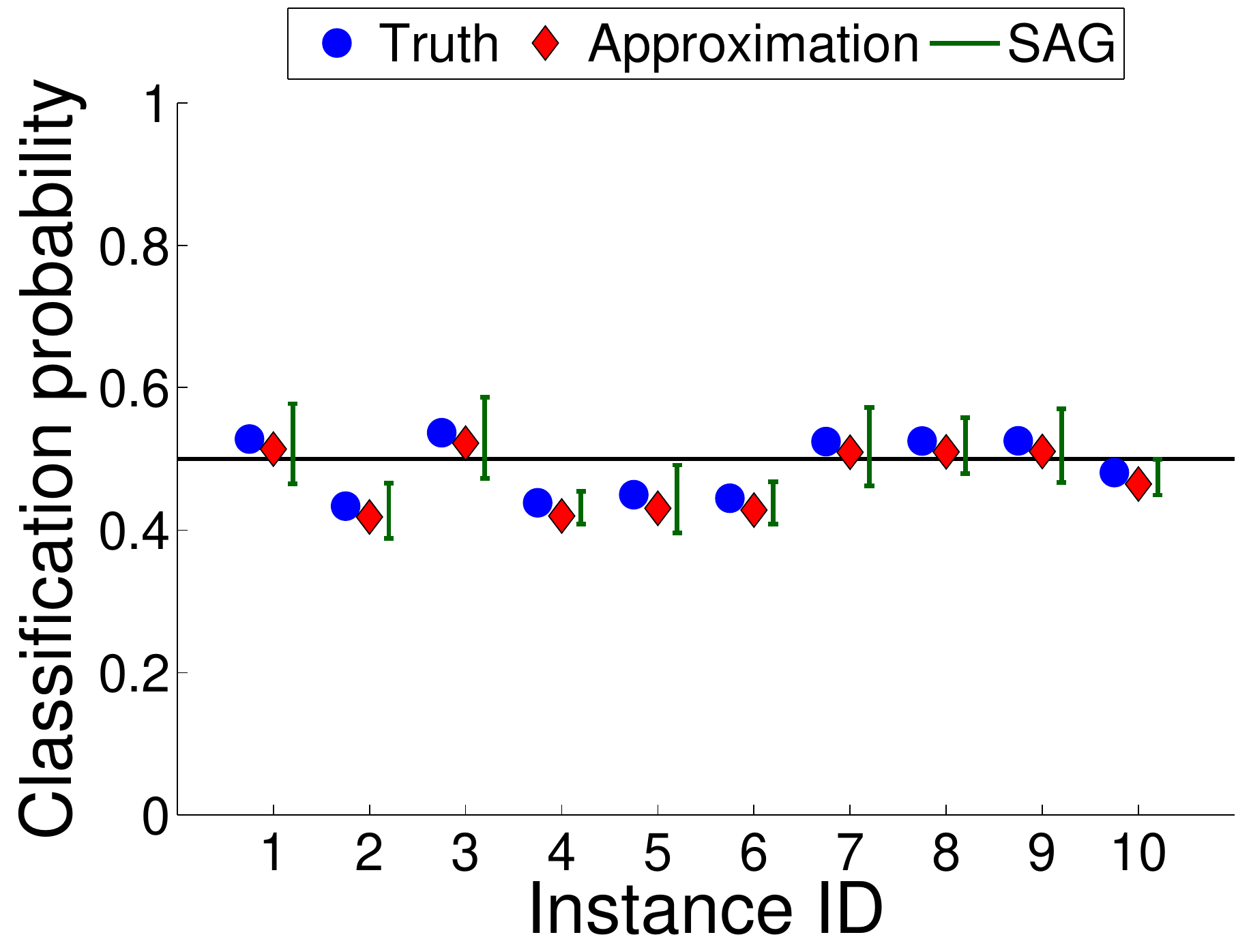}

MGT, $L=100$
\end{center}
\end{minipage}
&
\begin{minipage}{0.3\hsize}
\begin{center}
\includegraphics[width = \textwidth]{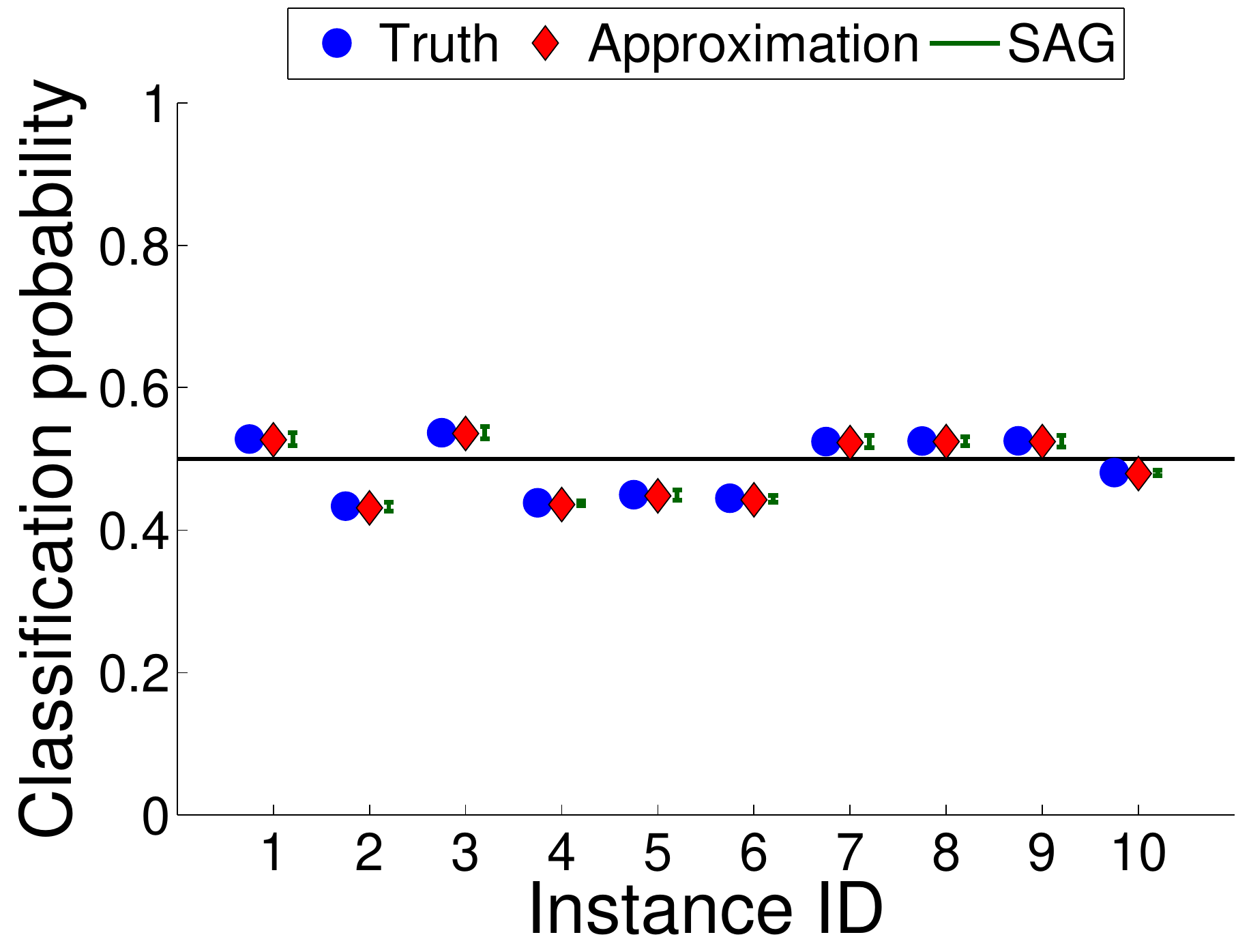}

MGT, $L=1000$
\end{center}
\end{minipage}
\\
\begin{minipage}{0.3\hsize}
\begin{center}
\includegraphics[width = \textwidth]{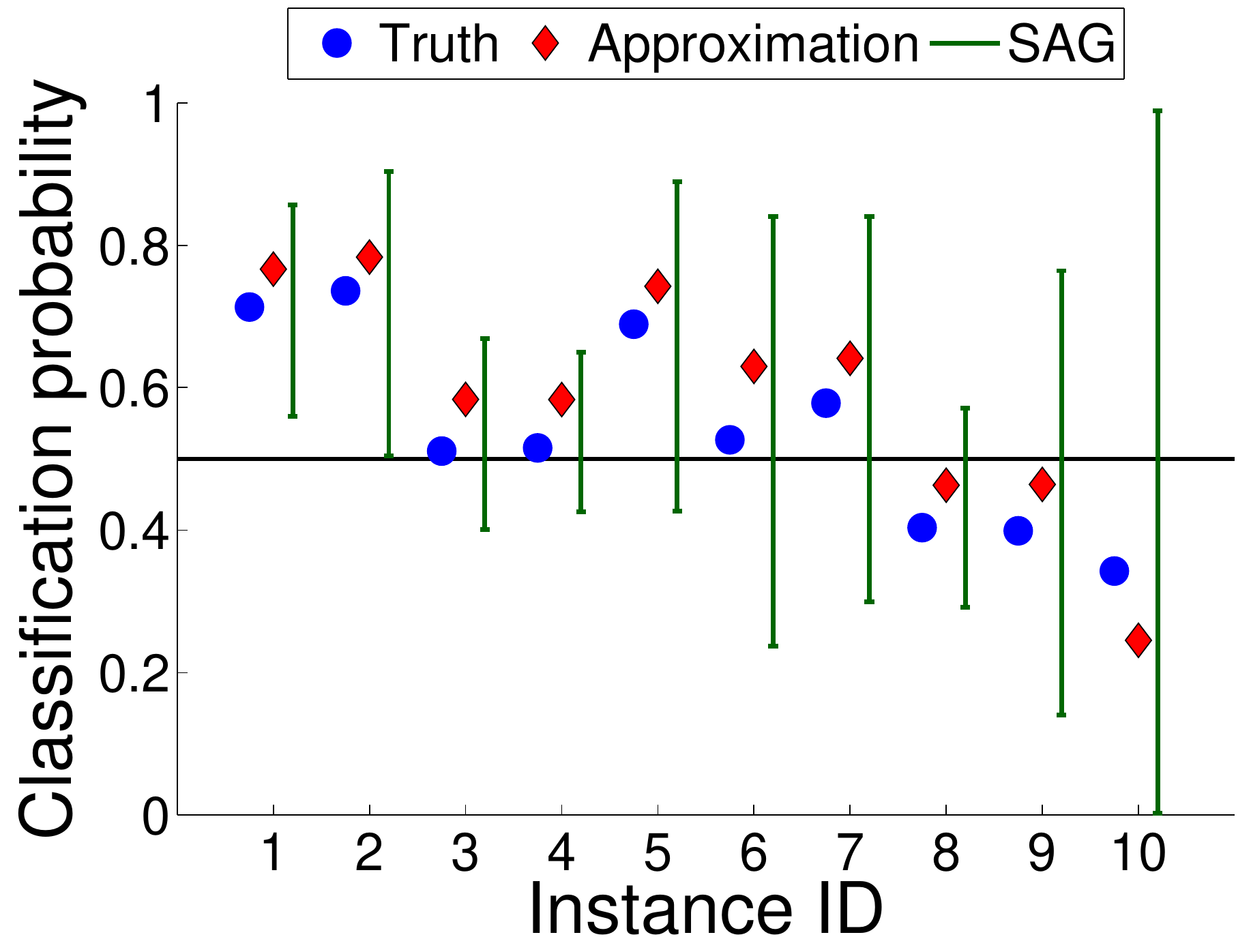}

Spambase, $L=10$
\end{center}
\end{minipage}
&
\begin{minipage}{0.3\hsize}
\begin{center}
\includegraphics[width = \textwidth]{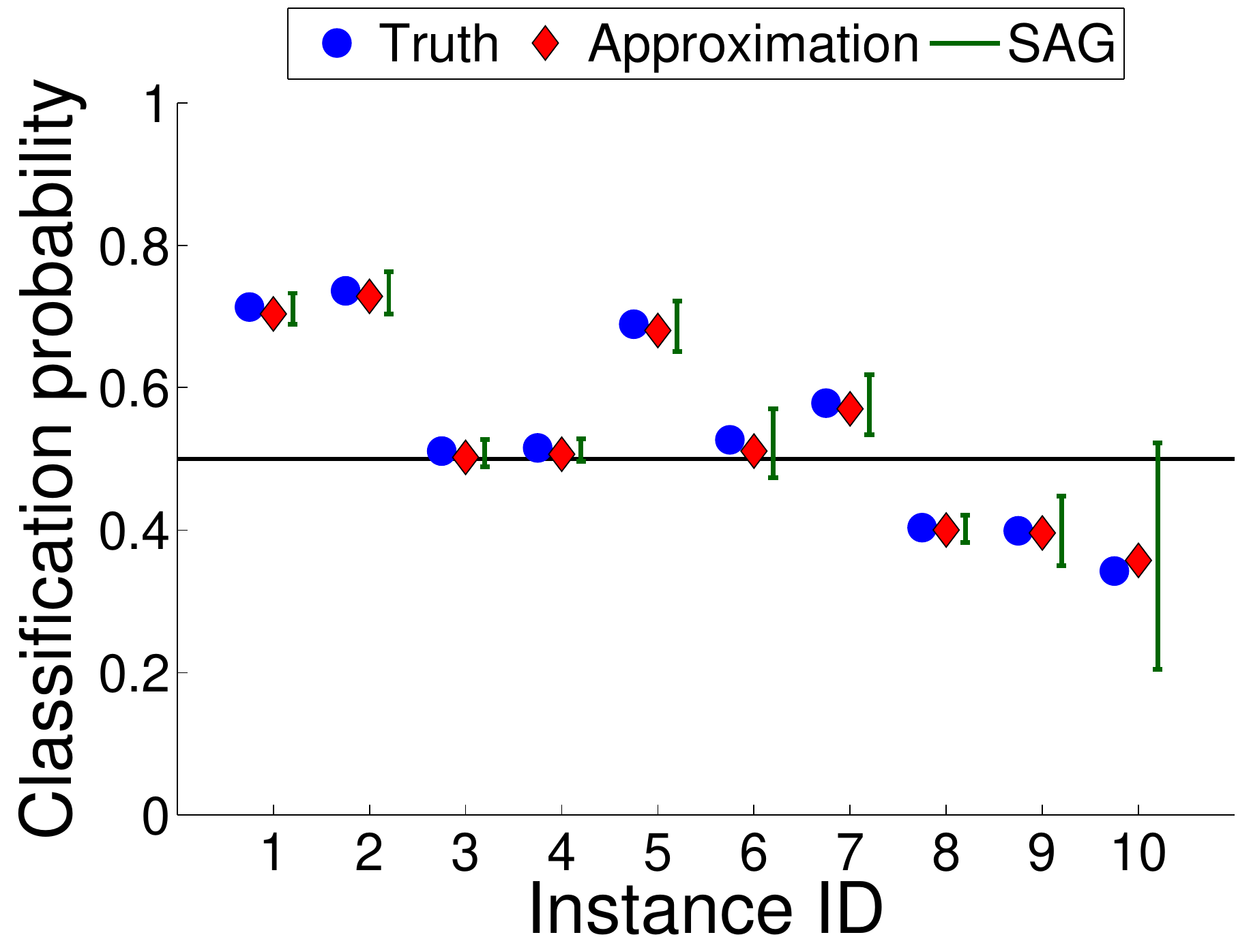}

Spambase, $L=100$
\end{center}
\end{minipage}
&
\begin{minipage}{0.3\hsize}
\begin{center}
\includegraphics[width = \textwidth]{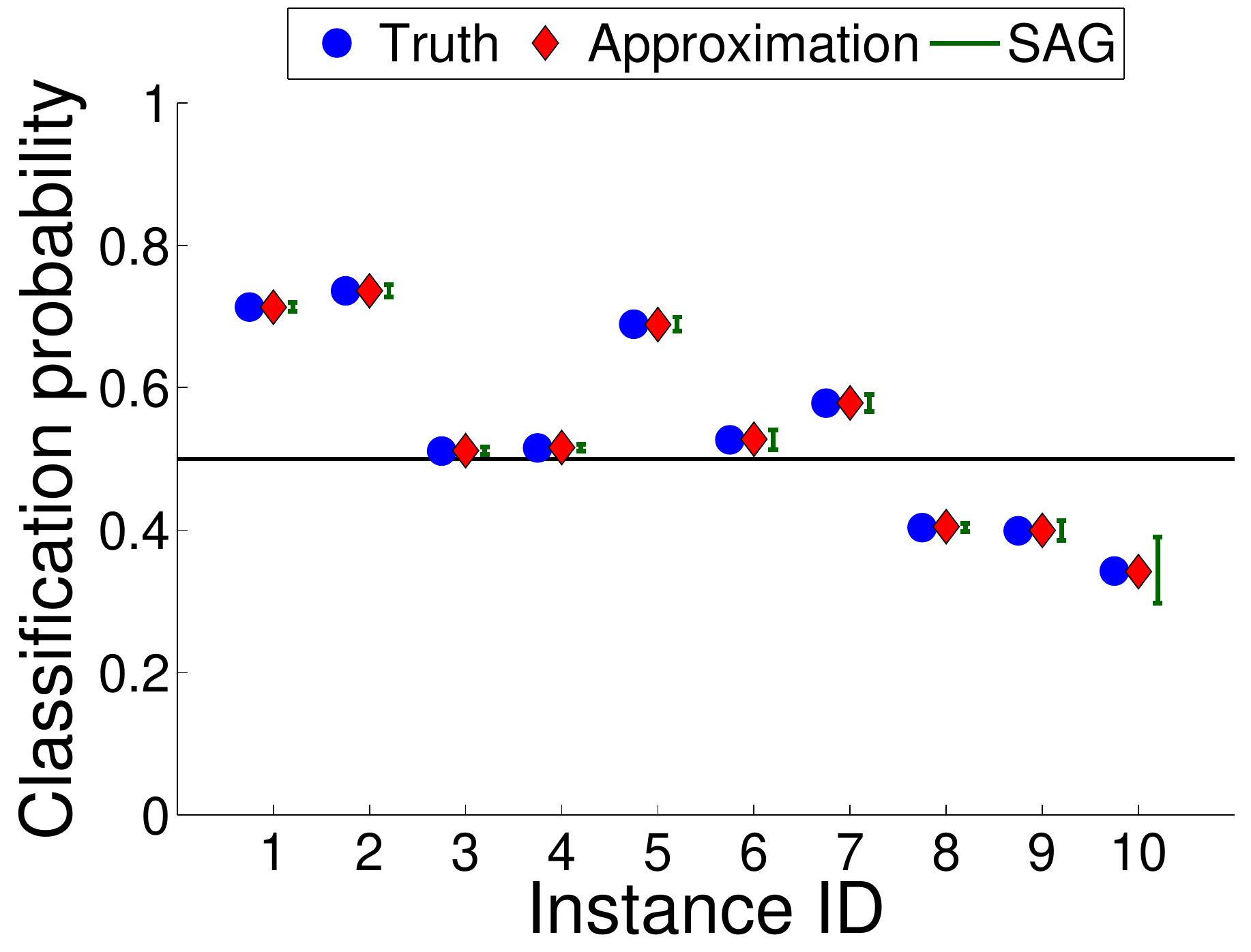}

Spambase, $L=1000$
\end{center}
\end{minipage}
\\
\begin{minipage}{0.3\hsize}
\begin{center}
\includegraphics[width = \textwidth]{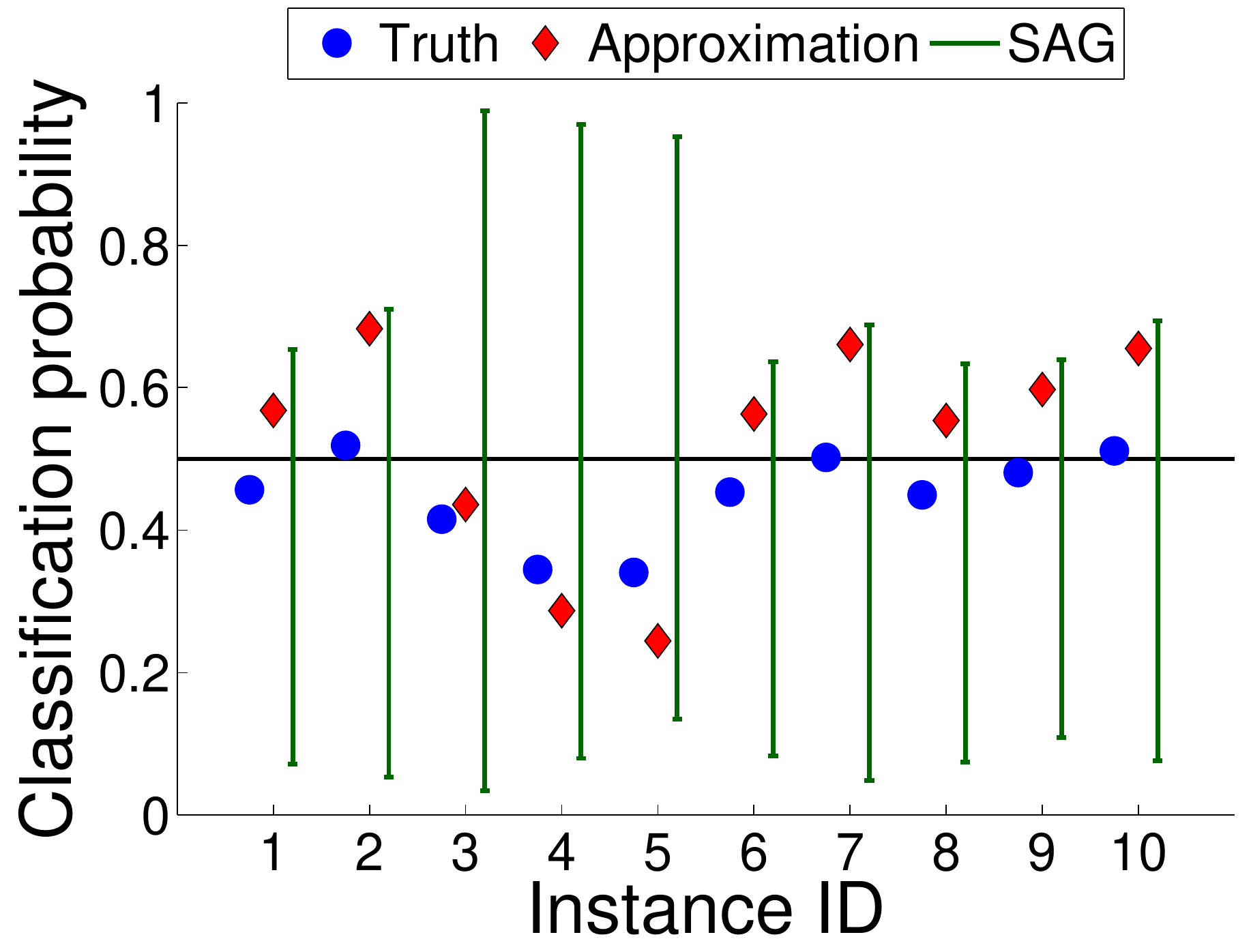}

OLD, $L=10$
\end{center}
\end{minipage}
&
\begin{minipage}{0.3\hsize}
\begin{center}
\includegraphics[width = \textwidth]{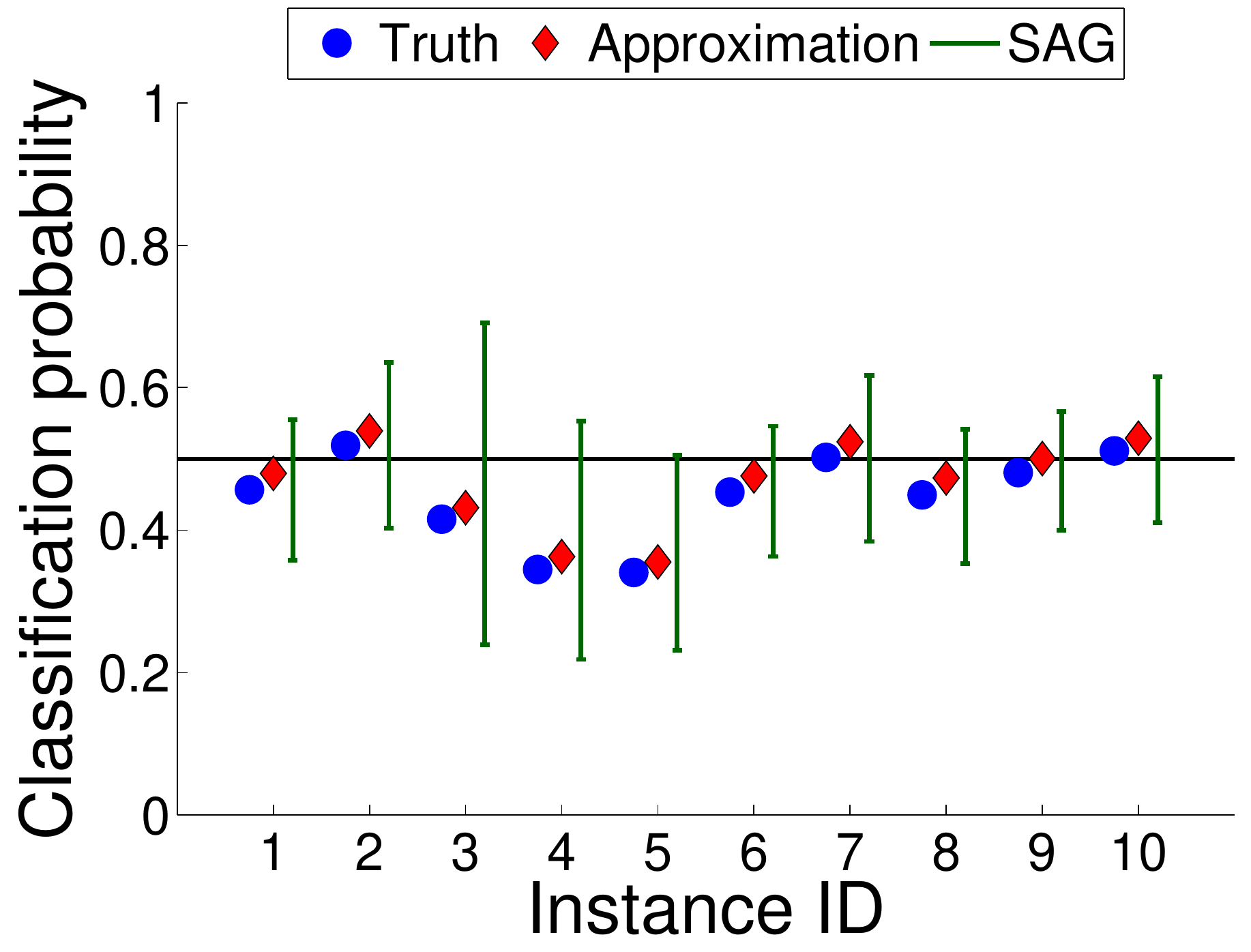}

OLD, $L=100$
\end{center}
\end{minipage}
&
\begin{minipage}{0.3\hsize}
\begin{center}
\includegraphics[width = \textwidth]{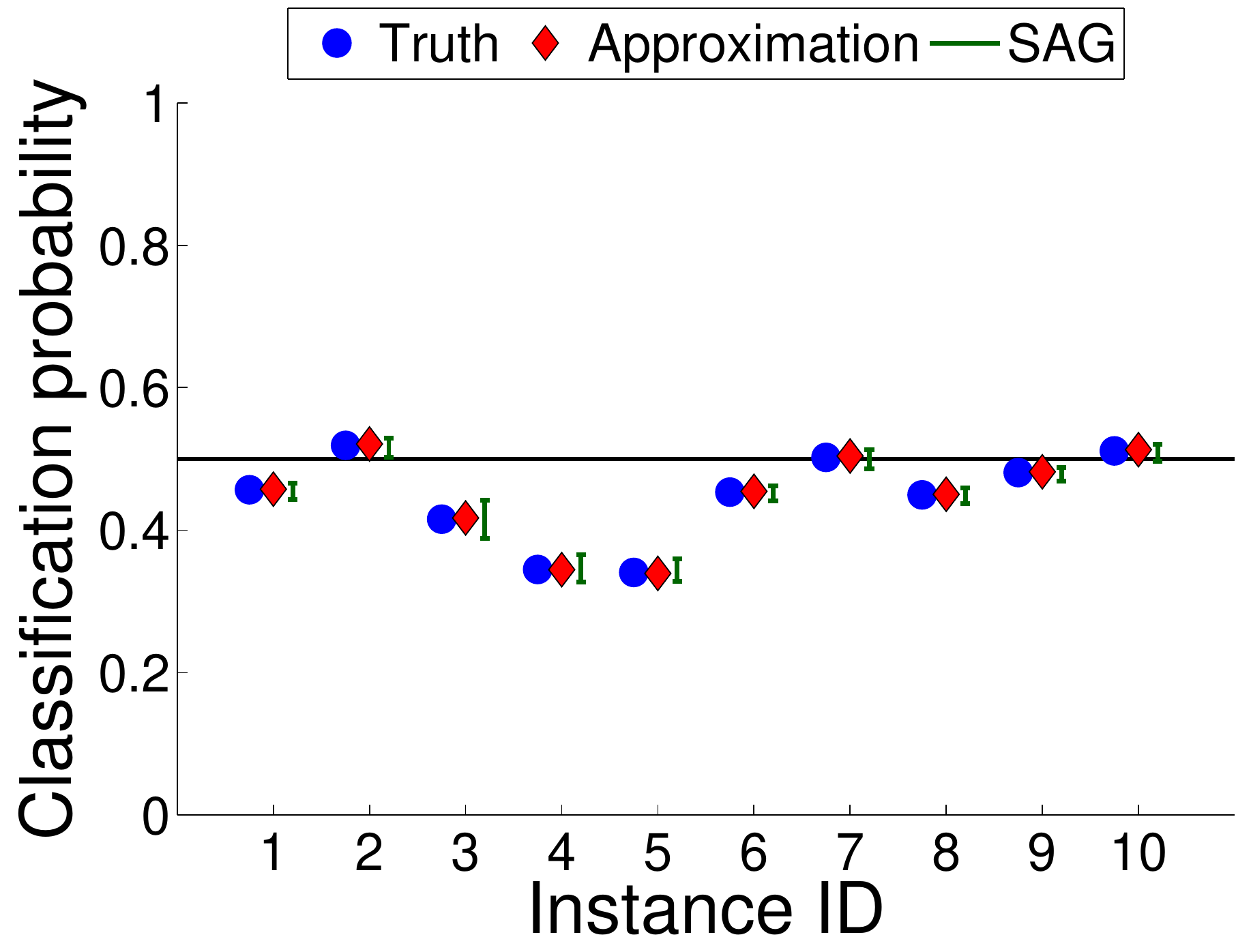}

OLD, $L=1000$
\end{center}
\end{minipage}
\end{tabular}
\end{center}
\caption{Change of bounds for the change of the accuracy of the approximated solution $\vhat{w}$}
\label{fig:change-L-all}
\end{figure*}

\begin{figure*}[tp]
\begin{center}
\begin{tabular}{ccc}
\begin{minipage}{0.3\hsize}
\begin{center}
\includegraphics[width = \textwidth]{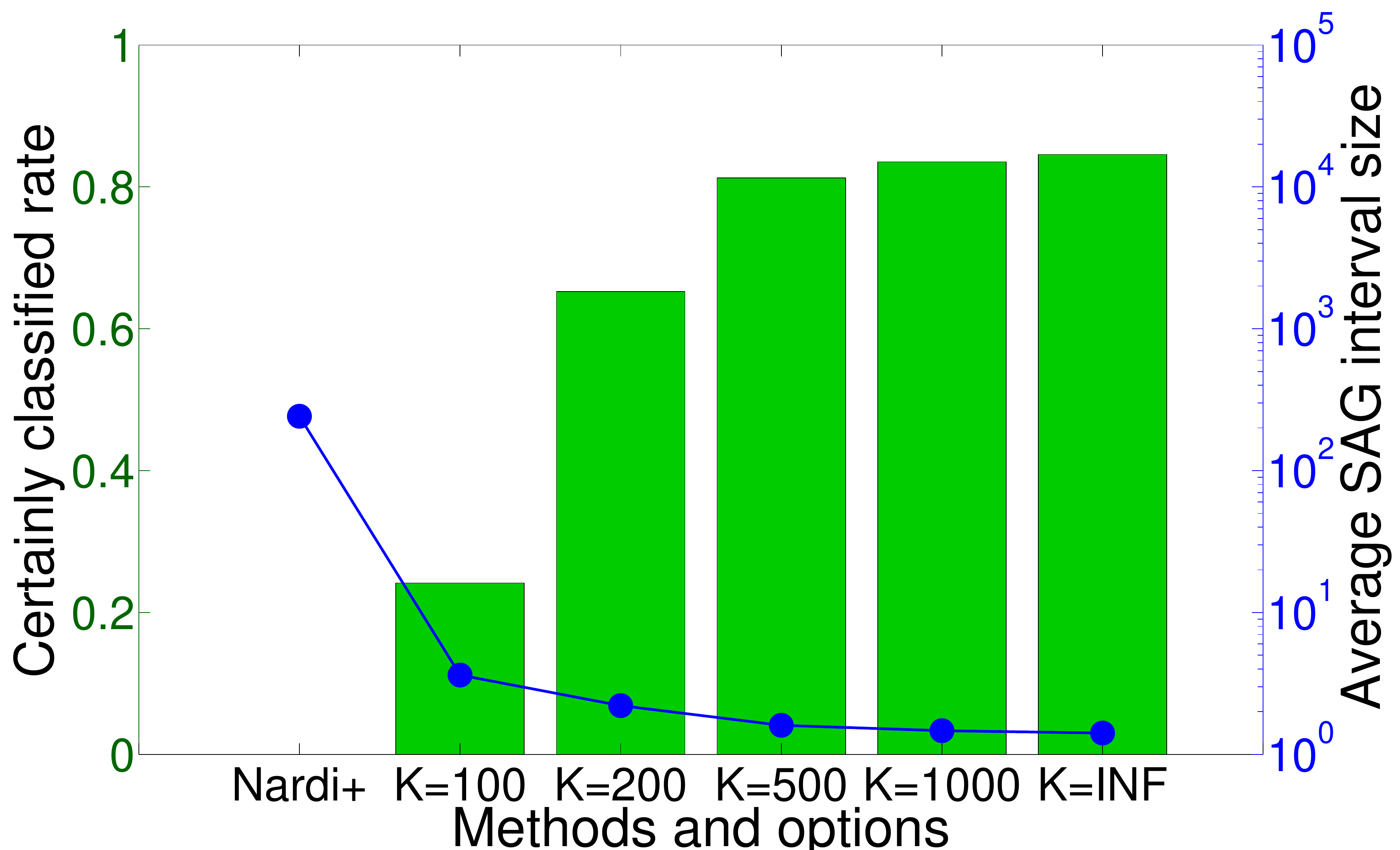}

Musk, $\lambda = 0.1$
\end{center}
\end{minipage}
&
\begin{minipage}{0.3\hsize}
\begin{center}
\includegraphics[width = \textwidth]{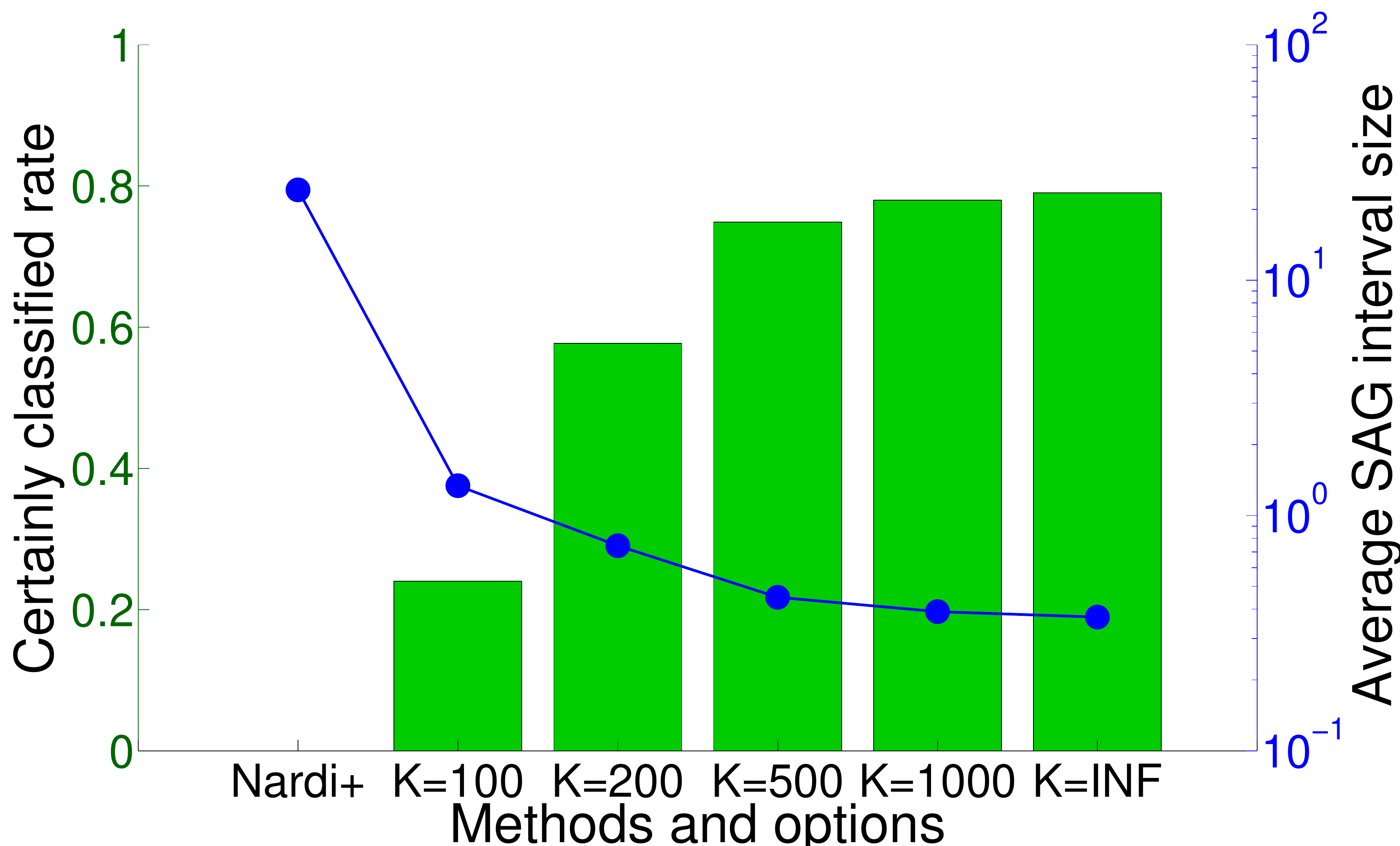}

Musk, $\lambda = 1$
\end{center}
\end{minipage}
&
\begin{minipage}{0.3\hsize}
\begin{center}
\includegraphics[width = \textwidth]{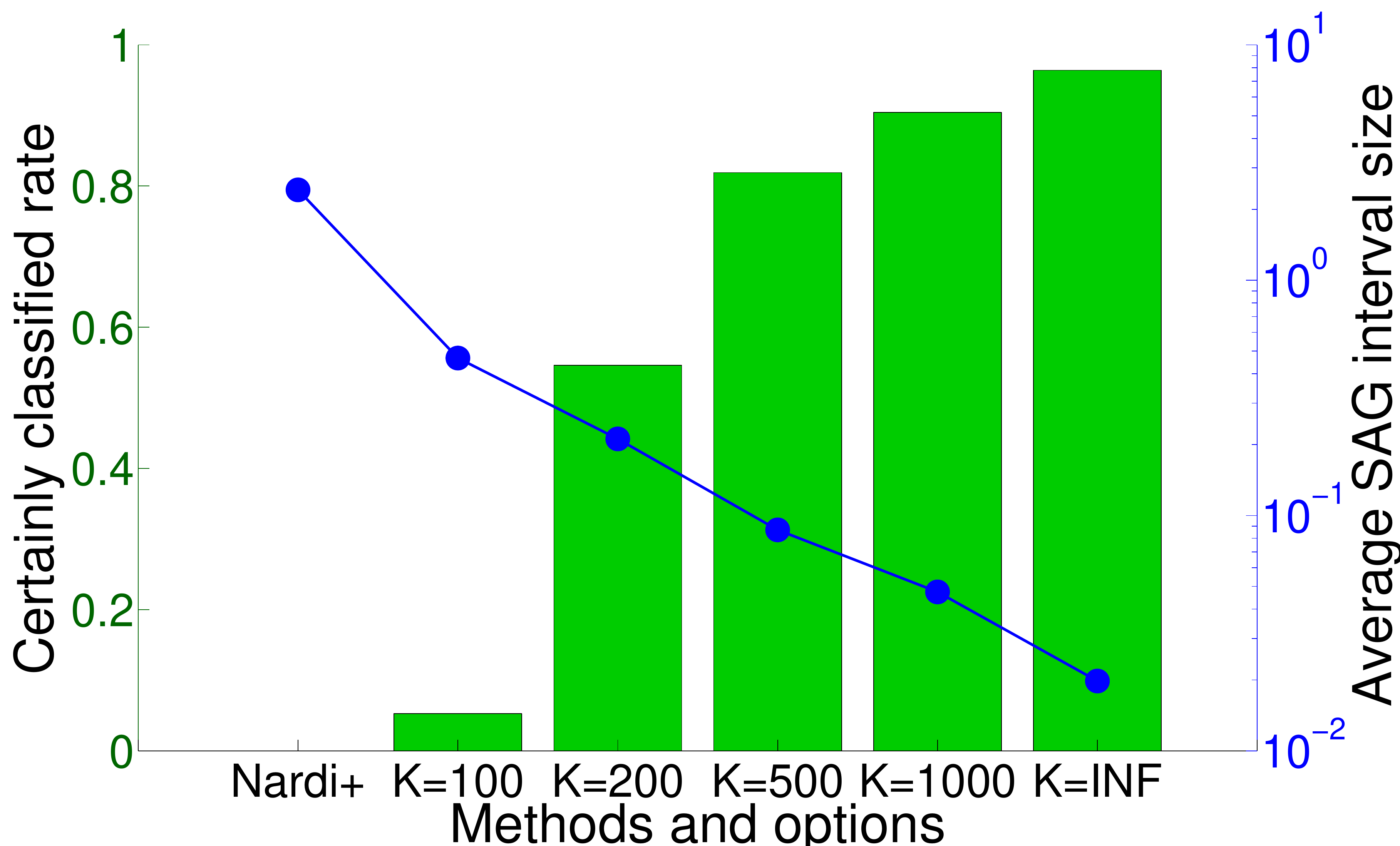}

Musk, $\lambda = 10$
\end{center}
\end{minipage}
\\
\begin{minipage}{0.3\hsize}
\begin{center}
\includegraphics[width = \textwidth]{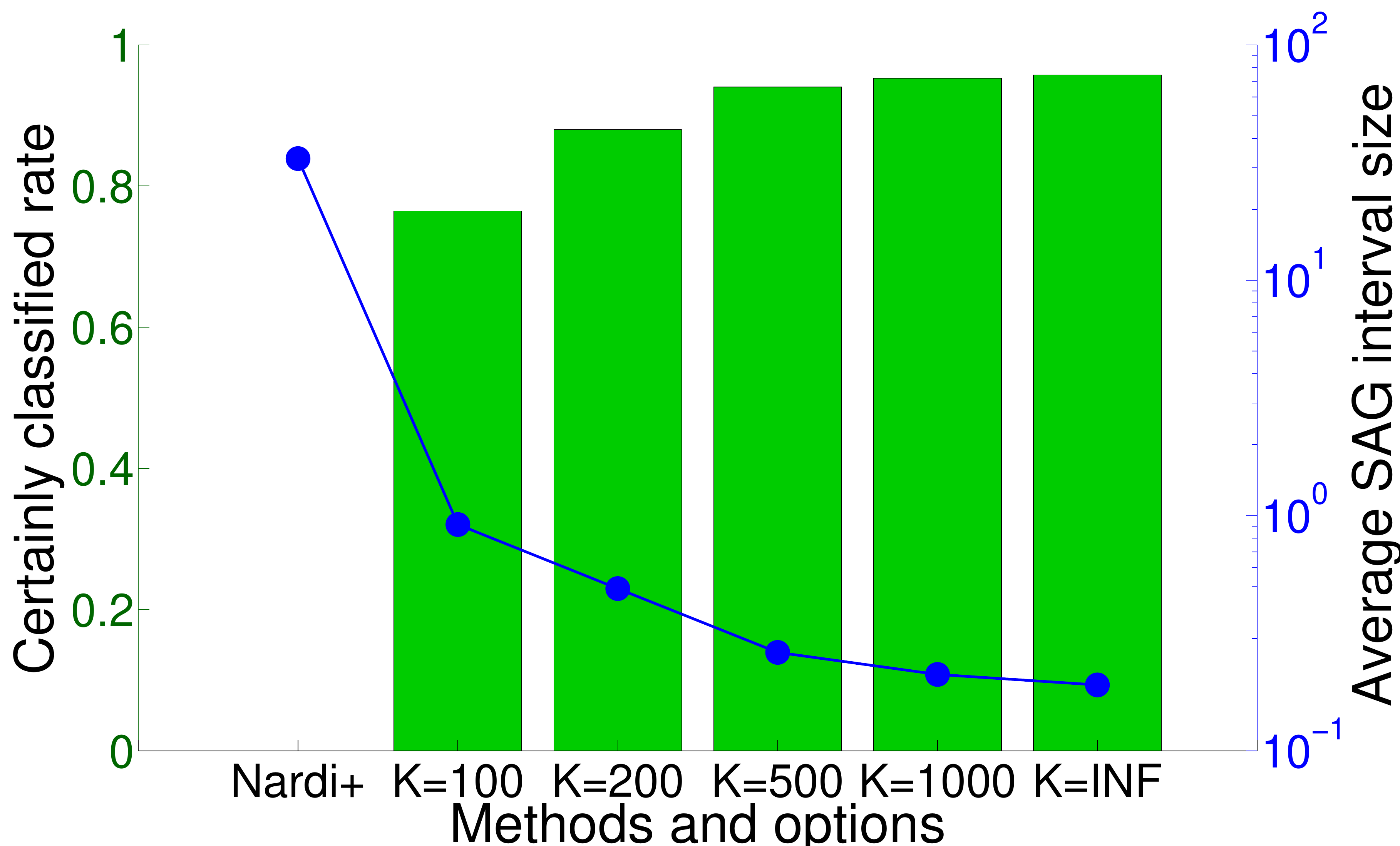}

MGT, $\lambda = 0.1$
\end{center}
\end{minipage}
&
\begin{minipage}{0.3\hsize}
\begin{center}
\includegraphics[width = \textwidth]{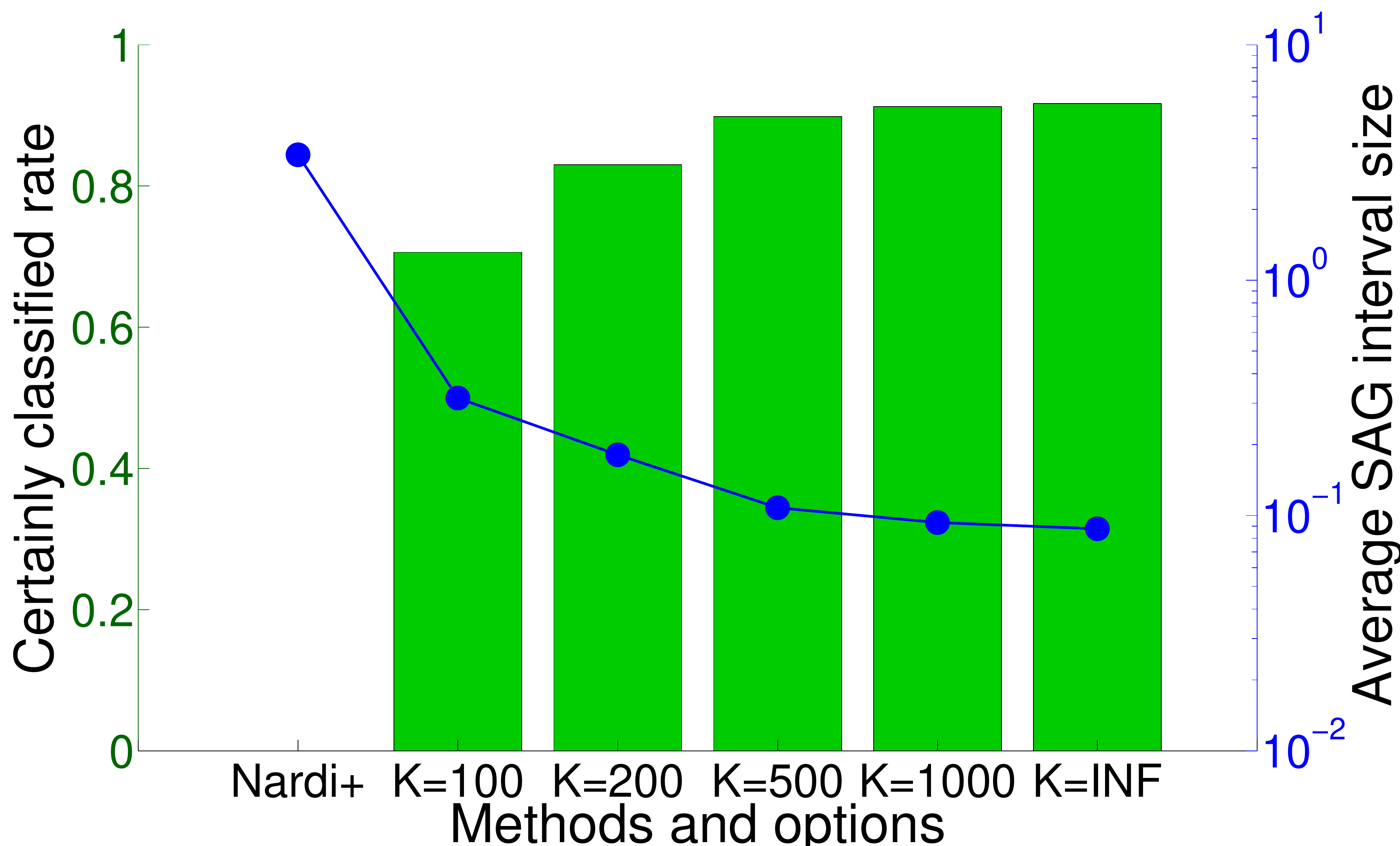}

MGT, $\lambda = 1$
\end{center}
\end{minipage}
&
\begin{minipage}{0.3\hsize}
\begin{center}
\includegraphics[width = \textwidth]{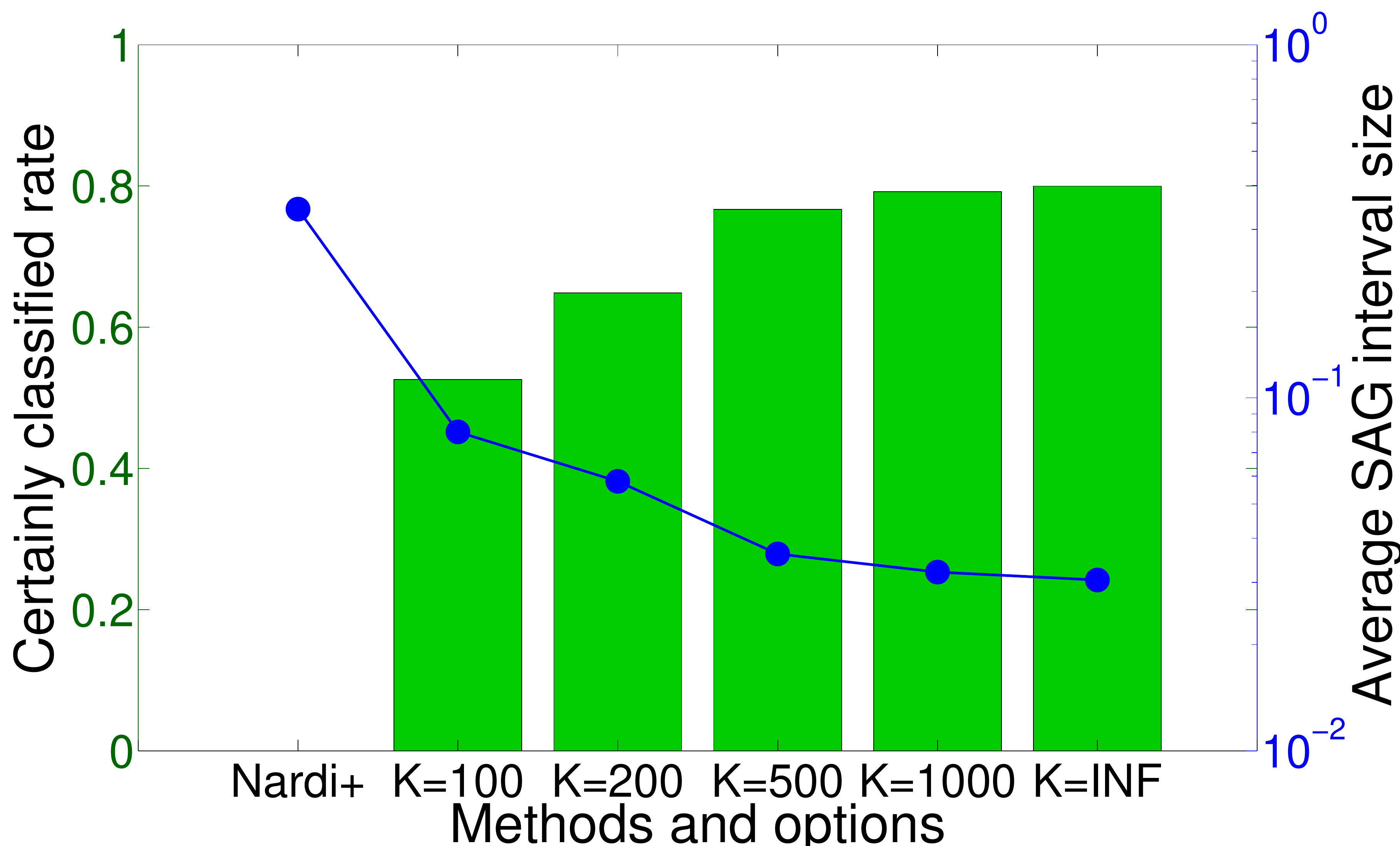}

MGT, $\lambda = 10$
\end{center}
\end{minipage}
\\
\begin{minipage}{0.3\hsize}
\begin{center}
\includegraphics[width = \textwidth]{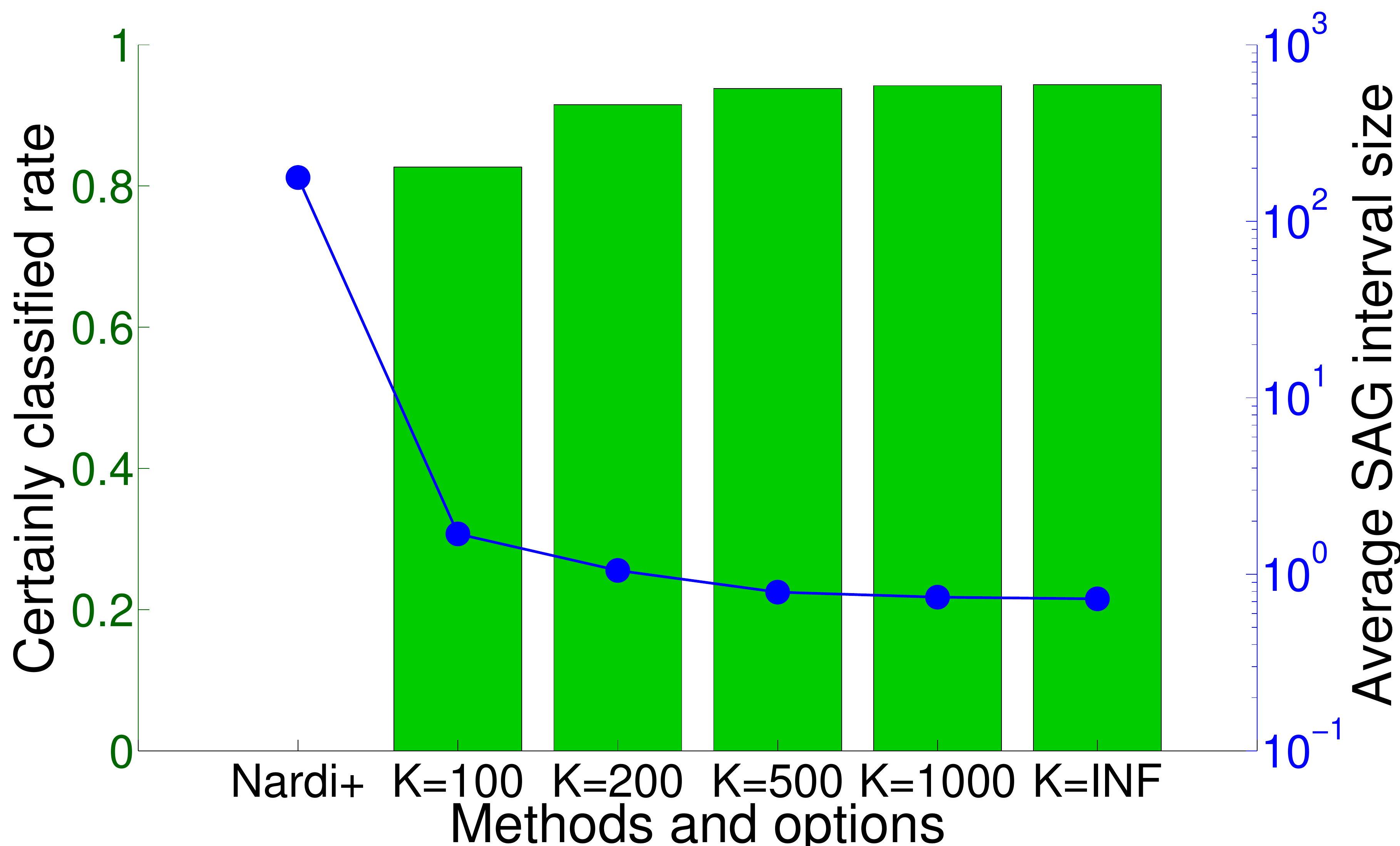}

Spambase, $\lambda = 0.1$
\end{center}
\end{minipage}
&
\begin{minipage}{0.3\hsize}
\begin{center}
\includegraphics[width = \textwidth]{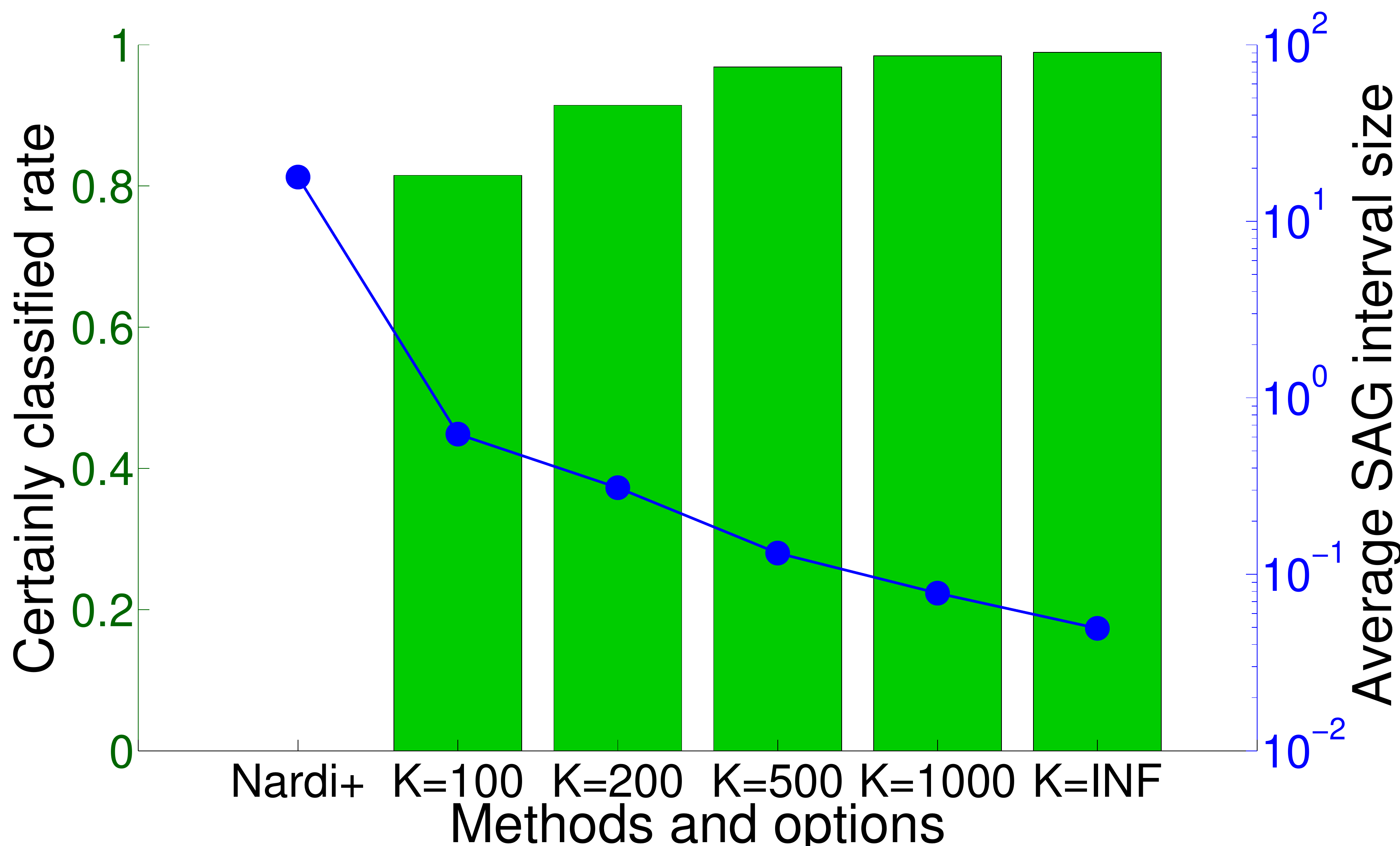}

Spambase, $\lambda = 1$
\end{center}
\end{minipage}
&
\begin{minipage}{0.3\hsize}
\begin{center}
\includegraphics[width = \textwidth]{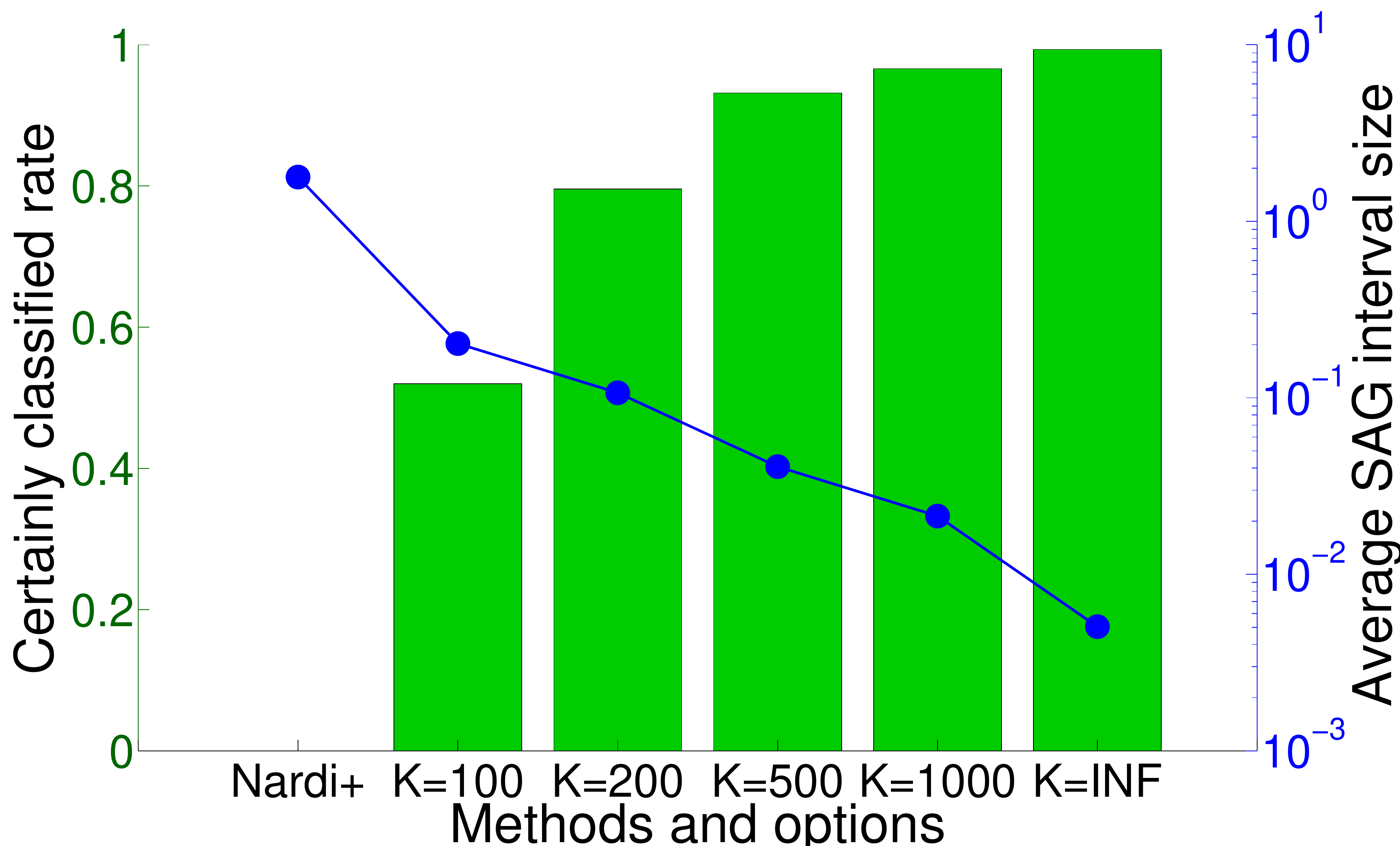}

Spambase, $\lambda = 10$
\end{center}
\end{minipage}
\\
\begin{minipage}{0.3\hsize}
\begin{center}
\includegraphics[width = \textwidth]{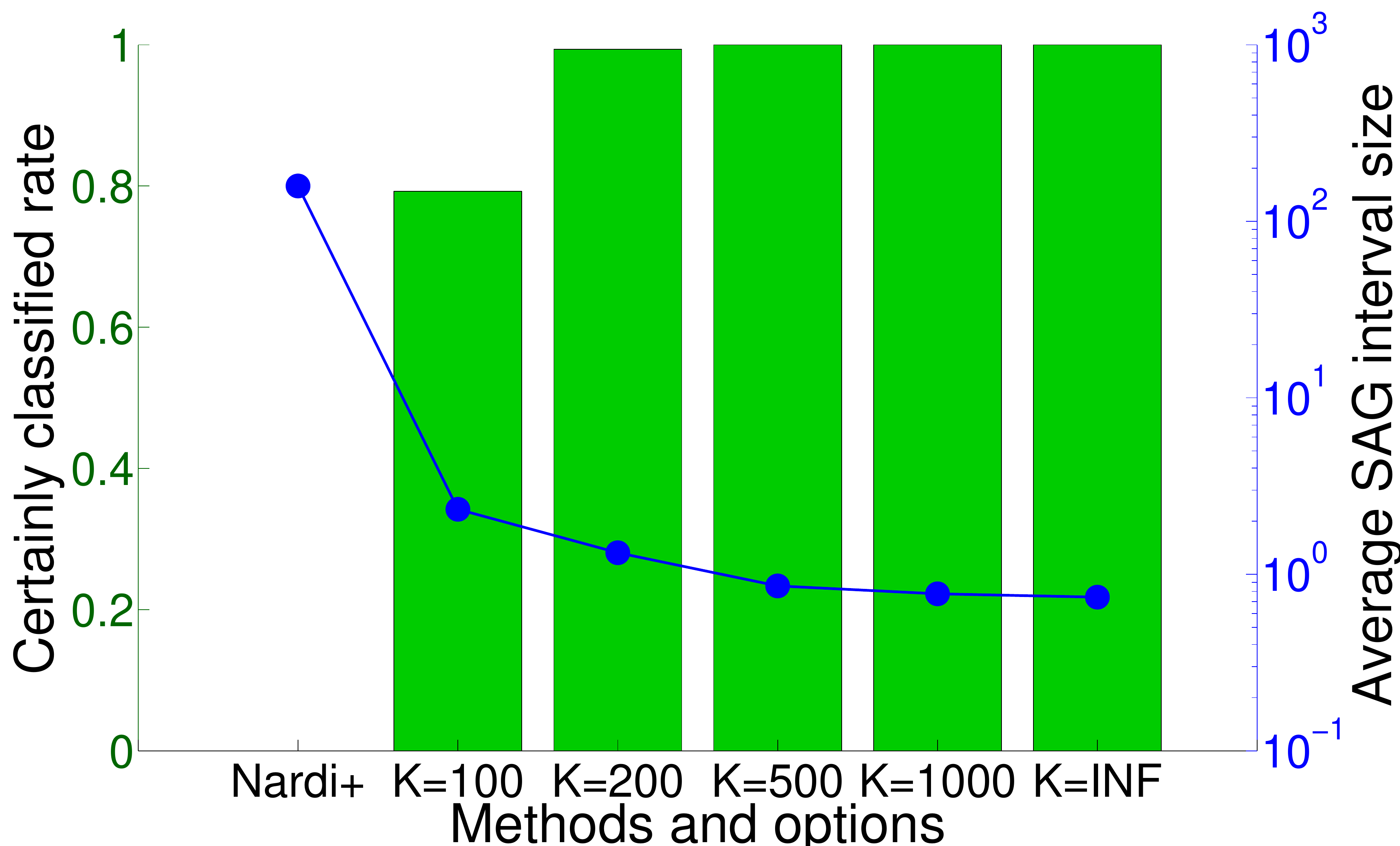}

OLD, $\lambda = 0.1$
\end{center}
\end{minipage}
&
\begin{minipage}{0.3\hsize}
\begin{center}
\includegraphics[width = \textwidth]{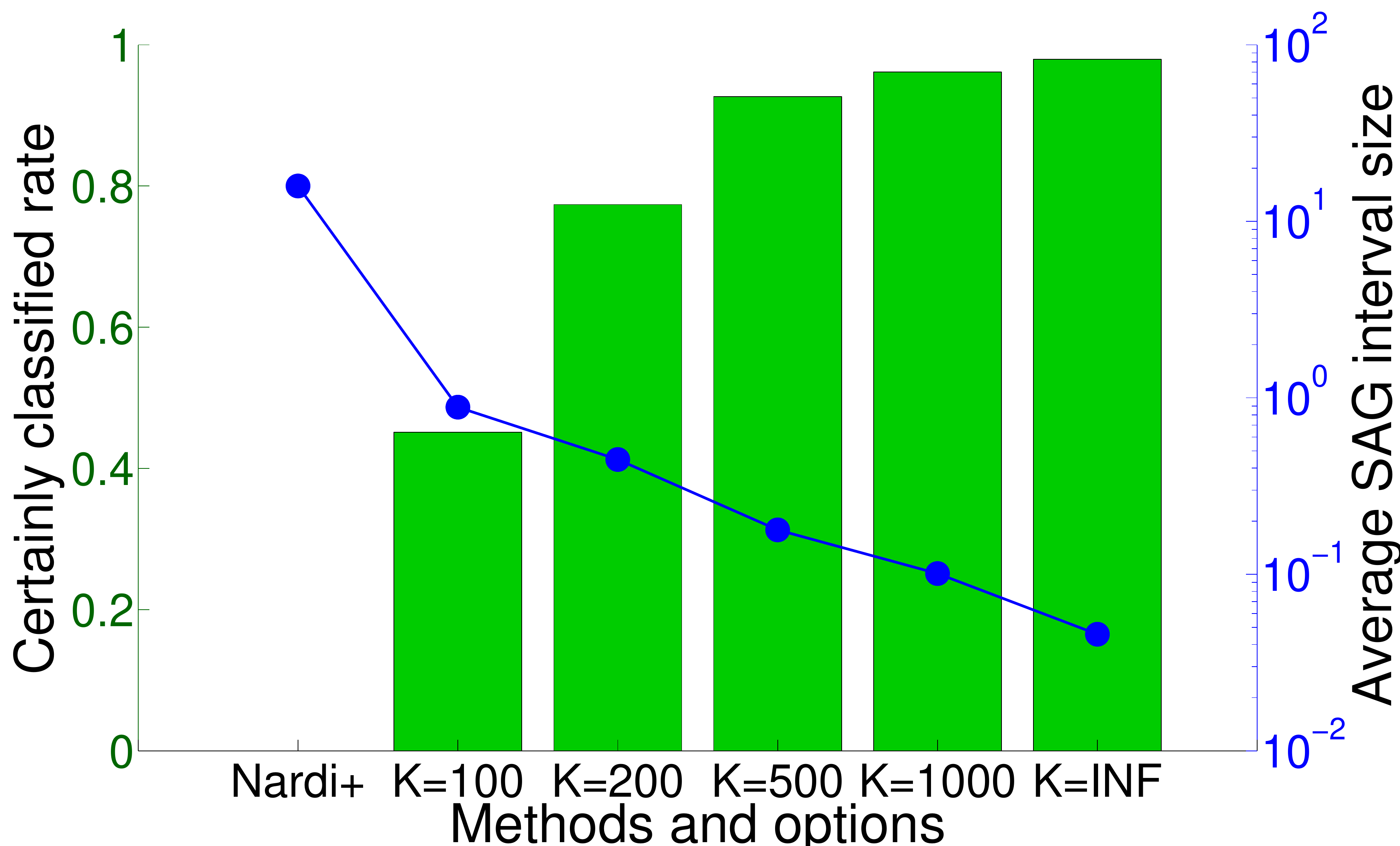}

OLD, $\lambda = 1$
\end{center}
\end{minipage}
&
\begin{minipage}{0.3\hsize}
\begin{center}
\includegraphics[width = \textwidth]{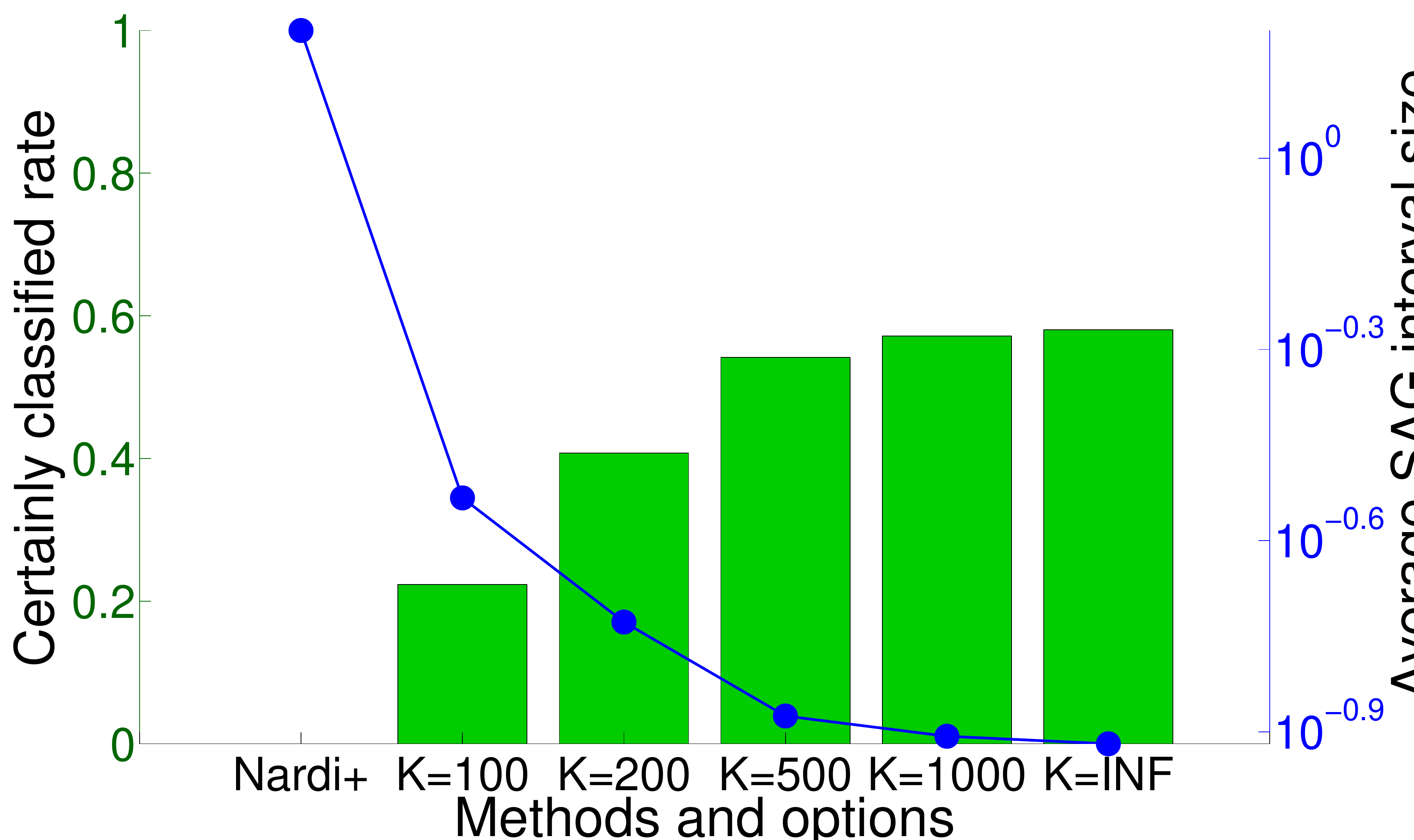}

OLD, $\lambda = 10$
\end{center}
\end{minipage}
\end{tabular}
\end{center}
\caption{Rate of successfully classified test instances and the average of size of bounds by different bound calculations (Nardi's, $K\in\{100, 200, 500, 1000, \infty\}$)}
\label{fig:clasify_mean-all}
\end{figure*}

We first investigated
several properties
of the SAG method for the logistic regression \eqref{eq:Logisic-loss}
by applying it to four benchmark datasets
summarized in Table~\ref{tab:dataset}.

First,
in \figurename~\ref{fig:exp1-all},
we compared the tightness of the bounds 
on the predicted classification probabilities for two randomly chosen validation instances
$\bm x_i$
defined as
$p(\bm x_i) := 1/(1 + \exp(- \bm x_i^\top \bm w^*))$, 
$i = 1, 2$. 
In the figure,
four types of intervals are plotted. 
The orange ones
are
Nardi {\it et al.}'s probabilistic bounds \cite{nardi2012achieving} with the probability 90\% (see \eqref{eq:nardi-bound}). 
The blue, green and purple ones
were obtained
by the SAG method
with $K = 100, 1000 \text{ and } \infty$,
respectively,
where
$K$ is the number of pieces in the piecewise-linear approximations. 
Here,
$K=\infty$ means that
the true loss function $\ell$ was used
as the two surrogate loss functions $\phi$ and $\psi$. 
The results clearly indicate that
bounds obtained by the SAG method are clearly tighter than those by Nardi et al.'s approach
despite that
the latter is probabilistic and cannot be securely computed in practice. 
When comparing the results with different $K$,
we can confirm that large $K$ yields tighter bounds.
The results with $K=1000$ are almost as tight as those obtained with the true loss function ($K=\infty$). 

\figurename~\ref{fig:change-L-all} also shows similar plots.
Here,
we investigated how the tightness of the SAG bounds changes
with the quality of the approximate solution
$\hat{\bm w}$. 
In order to consider approximate solutions with different levels of quality, 
we computed three approximate solutions with
$L = 10, 100 \text{ and } 1000$
in Nardi et al.'s approach,
where $L$ is the sample size used for approximating the logistic function
(see \S\ref{sect:existing}). 
The results clearly indicate that
tighter bounds are obtained
when the quality of the approximate solution is higher (i.e., larger $L$). 

\figurename~\ref{fig:clasify_mean-all} illustrates
how the SAG bounds can be useful in binary classification problems.
In binary classification problems,
if a lower bound of the classification probability is greater than 0.5,
the instance would be classified to positive class. 
Similarly, 
if an upper bound of the classification probability is smaller than 0.5,
the instance would be classified to negative class. 
The green histograms in the figure indicate 
how many percent of the validation instances
can be certainly classified as positive or negative class
based on the SAG bounds.
The blue lines indicate the average length of the SAG intervals,
i.e., 
the difference between the upper and the lower bounds. 
The results clearly indicate that,
as the number of pieces $K$ increases in the SAG method, 
the tighter bounds are obtained,
and more validation instances can be certainly classified.
On the other hand,
probabilistic bounds in Nardi et al.'s approach
cannot provide certain classification results
because their bounds are too loose. 

\begin{table}[tp]
\caption{Computation Time for obtaining bounds per instance}
\label{tab:time}
\begin{center}
\begin{tabular}{c|rrrr}\hline
$K$&100&200&500&1000\\\hline
Time(s)&381.089&790.674&1877.176&3717.569\\\hline
\end{tabular}
\end{center}
\end{table}

Finally, 
we examined the computation time
for computing the SAG bounds. 
Table~\ref{tab:time} 
shows the computation time per instance with
$K=\{100,200,500,1000\}$.
The results suggest that 
the computational cost is almost linear in $K$,
meaning that 
the computation of piecewise-linear functions dominates the cost.
Although this task can be completely parallelized per instance, 
further speed-up would be desired when $K$ is larger than 1000.

\subsection{Poisson and exponential regressions}

\begin{figure*}[tp]
\begin{center}
\begin{tabular}{cc}
\begin{minipage}[t]{0.45\hsize}
\begin{center}
\includegraphics[width = 0.8\textwidth]{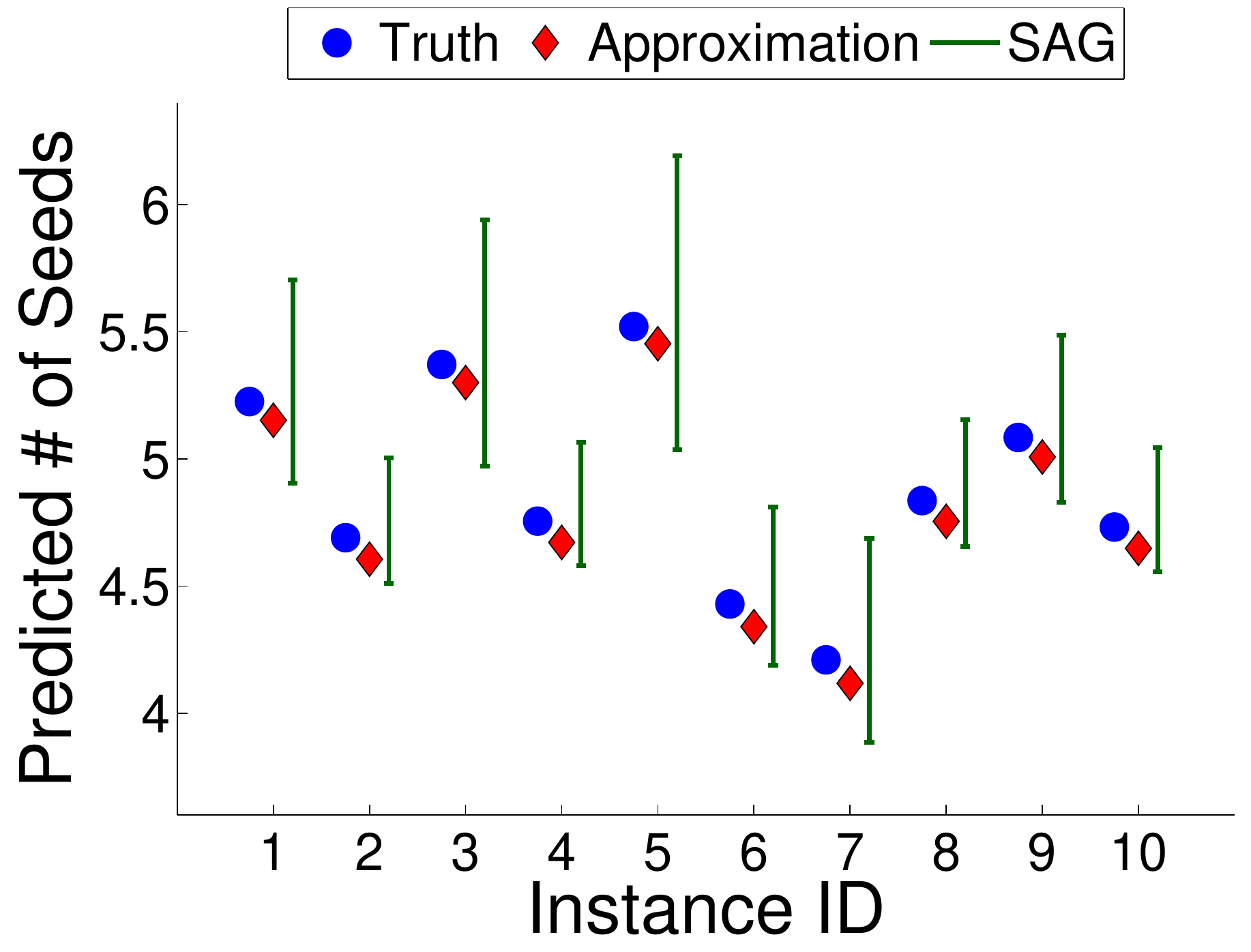}

(A) Poisson regression
\end{center}
\end{minipage}
&
\begin{minipage}[t]{0.45\hsize}
\begin{center}
\includegraphics[width = 0.8\textwidth]{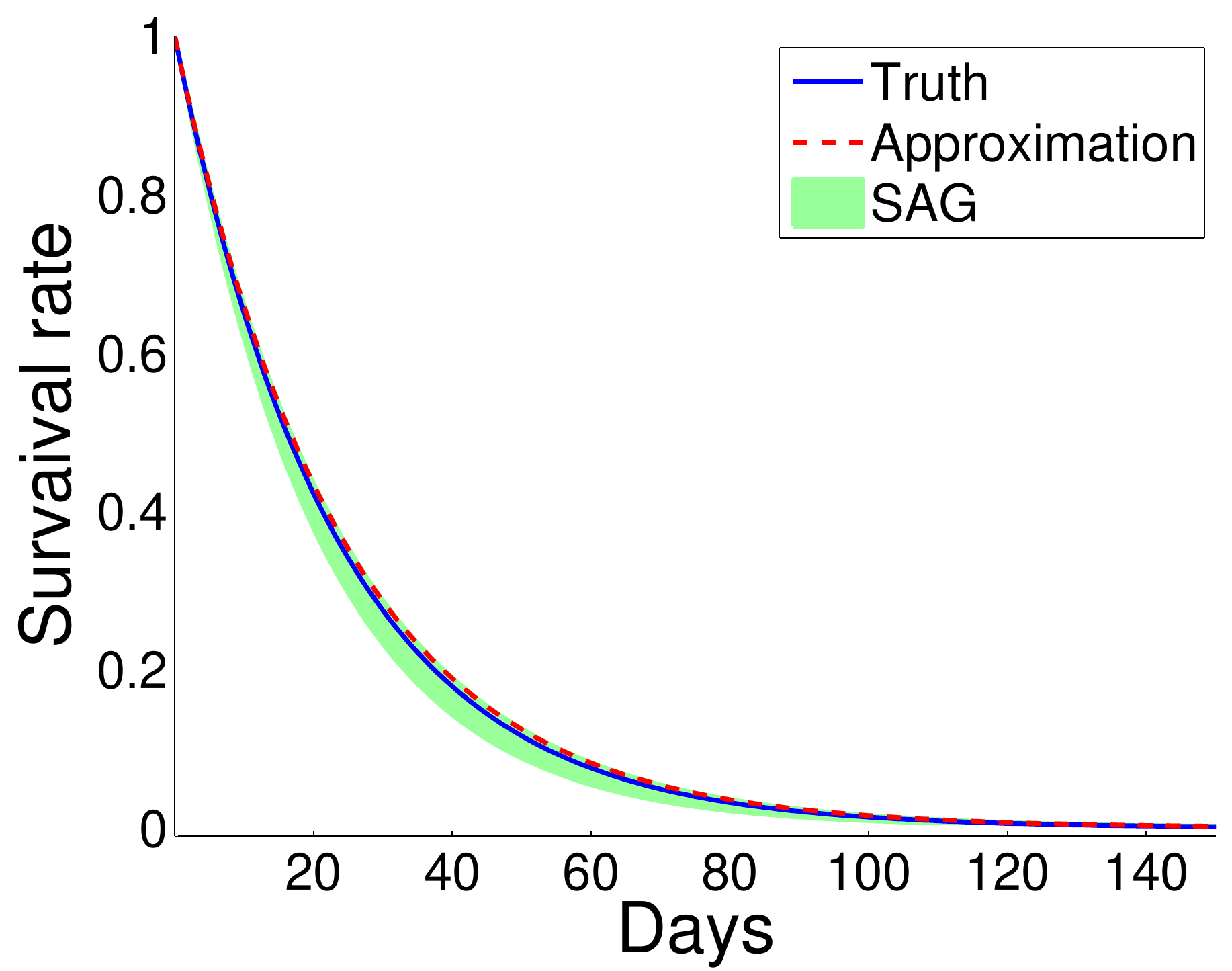}

(B) Exponential regression
\end{center}
\end{minipage}
\end{tabular}
\end{center}
\vspace{-1em}
\caption{Proposed bounds for Poisson and exponential regressions}
\label{fig:Poisson-Weibull}
\end{figure*}

We applied the SAG method to
Poisson regression \eqref{eq:Poisson-loss}
and
exponential regression \eqref{eq:Exponential-loss}. 
Poisson regression was applied to a problem 
for predicting the number of produced seeds
\footnote{http://hosho.ees.hokudai.ac.jp/\~{}kubo/stat/2015/Fig/\allowbreak{}poisson/data3a.csv}.
Exponential regression was applied to a problem 
for predicting survival time of lung cancer patients
\footnote{http://help.xlstat.com/customer/portal/kb\_article\_attachments/\allowbreak{}60040/original.xls}.
The results are shown
in \figurename~\ref{fig:Poisson-Weibull}.
The left plot (A) shows the result of Poisson regression,
where the SAG intervals on the predicted number of seeds are plotted for several randomly chosen instances. 
The right plot (B) shows the SAG bounds on the predicted survival probability curve,
in which we can confirm that the true survival probability curve is included in the SAG bound. 

\subsection{Privacy-preserving logistic regression to genomic and clinical data analysis}

\begin{figure}[tp]
\begin{center}
\begin{tabular}{cc}
\begin{minipage}[t]{0.45\hsize}
\begin{center}
\includegraphics[width = 0.8\textwidth]{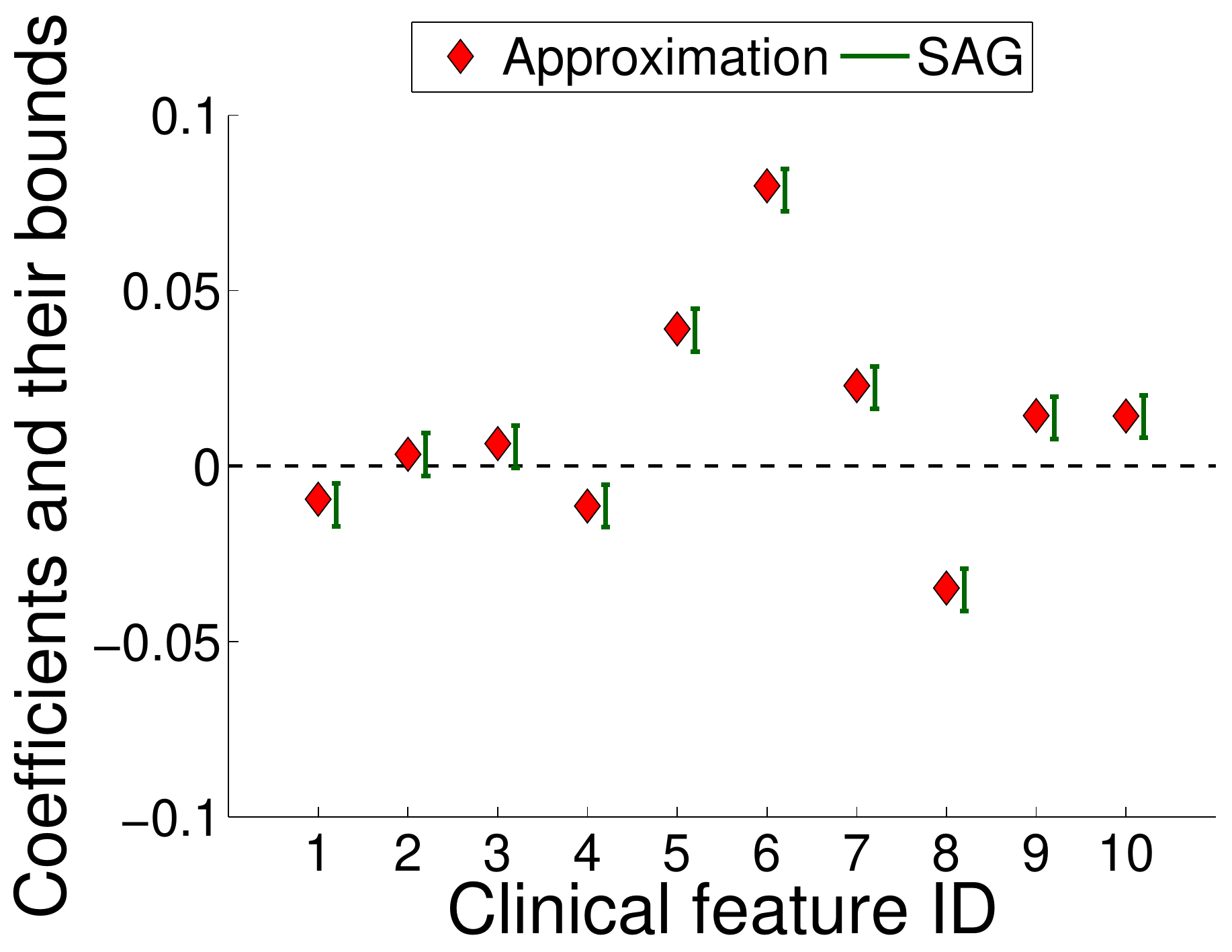}
\end{center}
\end{minipage}
&
\begin{minipage}[t]{0.45\hsize}
\begin{center}
\includegraphics[width = 0.8\textwidth]{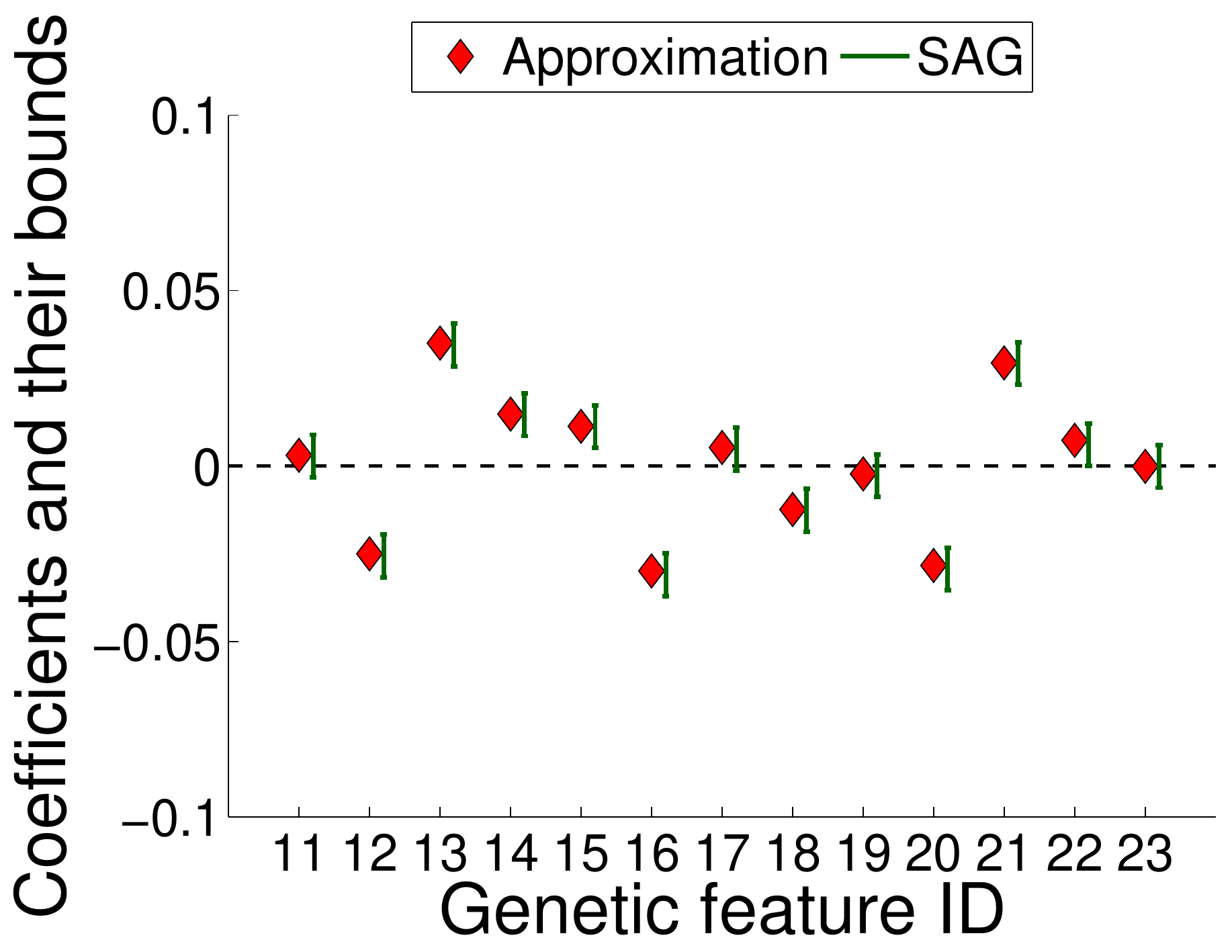}
\end{center}
\end{minipage}
\end{tabular}
\end{center}
\vspace{-1.5em}
\caption{Bounds of coefficients for disease risk evaluation}
\label{fig:disease-risk}
\end{figure}

Finally,
we apply
the SAG method
to a logistic regression 
on
a genomic and clinical data analysis,
which is the main motivation of this work ({\S}1).
In this problem,
we are interested in modeling the risk of a disease
based on genomic and clinical information of potential patients.
The difficulty of this problem is that 
genomic information were collected in a research institute,
while
clinical information were collected in a hospital,
and both institutes do not want to share their data to others.
However,
since the risk of the disease is dependent both on genomic and clinical features, 
it is quite valuable to use both types of information
for the risk modeling. 
Our goal is to find genomic and clinical features that highly affect the risk of the disease. 
To this end,
we use the SAG method for computing the bounds of coefficients of the logistic regression model
as described in \S\ref{sec:secure-approximation-guarantee}. 

In this experiment,
13 genomic (SNP) and 10 clinical features of 134 potential patients
are provided from a research institute and a hospital,
respectively
\footnote{
Due to confidentiality reasons, 
we cannot describe the details of the dataset.
Here,
we only analyzed a randomly sampled small portion of the datasets 
just for illustration purpose. 
}.
The SAG bounds on each of these 23 coefficients are plotted in
\figurename~\ref{fig:disease-risk}. 
Although we do not know the true coefficient values,
we can at least identify
features that positively/negatively correlated with the disease risk 
(note that, if the lower/upper bound is greater/smaller than 0, the feature is guaranteed to have positive/negative coefficient in the logistic regression model).

\section{Conclusions}
\label{sec:conclusions}
We studied empirical risk minimization (ERM) problems
under secure multi-party computation (MPC) frameworks.
We developed a novel technique called
secure approximation guarantee (SAG) method 
that can be used when only an approximate solution is available
due to the difficulty of secure non-linear function evaluations. 
The key property of the SAG method is that
it can securely provide the bounds on the true solution, 
which is practically valuable 
as we illustrated
in benchmark data experiments and 
in our motivating problem on genomic and clinical data.

\clearpage
\bibliography{paper}
\bibliographystyle{unsrt}
\clearpage
\newpage
\section*{Appendix}
\subsection*{Proofs of Theorem \ref{th:bounds} and Corollary \ref{coro:output-bound} (bounds of $\bm{w}^*$ from $\vhat{w}$)}
First
we present the following proposition
which will be used for proving 
Theorem~\ref{th:bounds}.
\begin{prop}
 \label{prop:VI}
 Consider the following general problem:
 \begin{align}
  \label{eq:general.convex.constrained.problem}
  \min_z ~ g(z) ~~~ {\rm s.t.} ~ z \in \cZ, 
 \end{align}
 where
 $g: \cZ \to \RR$ 
 is a subdifferentiable convex function and $\cZ$ is a convex set. 
 Then a solution $z^*$ is the optimal solution of
 \eq{eq:general.convex.constrained.problem}
 if and only if
 \begin{align*}
  \nabla g(z^*)^\top (z^* - z) \le 0 ~~~ \forall ~ z \in \cZ,
 \end{align*}
 where $\nabla g(z^*)$ is the subgradient vector of $g$ at $z = z^*$.
\end{prop}

See,
for example,
Proposition B.24 in \cite{bertsekas1999nonlinear}
for the proof of
Proposition~\ref{prop:VI}.

\begin{proof}[Proof of Theorem~\ref{th:bounds}]
 Using a slack variable
 $\xi \in \RR$, 
 let us first rewrite the minimization problem \eqref{eq:optimization-problem} as 
 \begin{eqnarray}
  \min_{\bm{w}\in\mathbb{R}^d, \xi\in\mathbb{R}}
   J(\bm{w}, \xi)
   \defequal
   \xi
   +
   \frac{\lambda}{2}\|\bm{w}\|^2
   \text{\quad s.t.\quad}
   \xi\geq \frac{1}{n}\sum_{i\in[n]}\ell(y_i, \bm{x}_i^\top\bm{w}).
   \label{ex:problem-slack}
 \end{eqnarray}
 Note that the optimal solution of the problem 
 \eq{ex:problem-slack}
 is
 $\bm w = \bm w^*$
 and
 $\xi = \xi^* \defequal \frac{1}{n} \sum_{i \in [n]} \ell(y_i, \bm{x}_i^\top\bm{w}^*)$.
 Using the definitions of
 $\psi$
 and
 $\Psi$, 
 we have 
 $
 \frac{1}{n}\sum_{i \in [n]} \ell(y_i, \bm{x}_i^\top\hat{\bm{w}})
 \le 
 \frac{1}{n}\sum_{i \in [n]} \psi(y_i, \bm{x}_i^\top\hat{\bm{w}})
 =
 \Psi(\hat{\bm w})
 $.
 It means that 
 $(\hat{\bm w}, \Psi(\hat{\bm w}))$
 is a feasible solution of 
 the problem
 \eq{ex:problem-slack}. 
 Applying this fact into 
 Proposition~\ref{prop:VI},
 we have 
 \begin{align}
  \nabla J(\bm w^*, \xi^*)^\top
  \left(
  \mtx{c}{\bm w^* \\ \xi^*}
  -
  \mtx{c}{\hat{\bm w} \\ \Psi(\hat{\bm w})}
  \right)
  \le 0, 
  \label{eq:gradient-bound-a}
 \end{align}
 where
 $\nabla J(\bm w^*, \xi^*) \in \RR^{d + 1}$
 is the gradient of the objective function in 
 \eq{ex:problem-slack}
 evaluated at
 $(\bm w^*, \xi^*)$.
Since
$J(\bm w, \xi)$
is a quadratic function
of $\bm w$ and $\xi$,
we can write 
$\nabla J(\bm w^*, \xi^*)$
explicitly, 
and 
\eq{eq:gradient-bound-a}
is written as
\begin{align}
 &
 \lambda \|\bm w^2\| + \xi^* - \lambda \bm w^{*\top} \hat{\bm w} - \Psi(\hat{\bm w}) \le 0
 \nonumber
 \\
 \Leftrightarrow~
 &
 \lambda \|\bm w^{*2}\| + \frac{1}{n}\sum_{i \in [n]} \ell(y_i, \bm{x}_i^\top\bm{w}^*) 
 - \lambda \bm w^{*\top} \hat{\bm w} - \Psi(\hat{\bm w}) \le 0
 \label{eq:gradient-bound-b}
\end{align}
From the definition of
$\phi$
and
$\Phi$, we have
\begin{align*}
\frac{1}{n}\sum_{i \in [n]} \ell(y_i, \bm{x}_i^\top\bm{w}^*)
\ge
\frac{1}{n}\sum_{i \in [n]} \phi(y_i, \bm{x}_i^\top\bm{w}^*)
=
\Phi(\bm w^*).
\end{align*}
Plugging this into \eq{eq:gradient-bound-b}, we have 
\begin{align}
 \label{eq:gradient-bound-c}
 \lambda \|\bm w^{*2}\| + \Phi(\bm w^*) - \lambda \bm w^{*\top} \hat{\bm w} - \Psi(\hat{\bm w}) \le 0
\end{align}
Furthermore, 
noting that
$\phi$
and
$\Phi$
are convex with respect to $\bm w$, by the definition of convex functions we get
\begin{align}
 \label{eq:gradient-bound-d}
 \Phi(\bm w^*) \ge \Phi(\hat{\bm w}) + \nabla \Phi(\hat{\bm w})^\top (\bm w^* - \hat{\bm w}). 
\end{align}
By plugging
\eq{eq:gradient-bound-d}
into
\eq{eq:gradient-bound-c},
\begin{align}
 \lambda \|\bm w^{*2}\| + \Phi(\hat{\bm w}) + \nabla \Phi(\hat{\bm w})^\top (\bm w^* - \hat{\bm w}) - \lambda \bm w^{*\top} \hat{\bm w} - \Psi(\hat{\bm w}) \le 0
 \label{eq:gradient-bound-e}
\end{align}
Noting that
\eq{eq:gradient-bound-e}
is a quadratic function of
$\bm w^*$,
we obtain
\begin{align*}
\left\|\bm w^*
 -
\frac{1}{2}\left(\hat{\bm{w}} - \frac{1}{\lambda}\nabla\Phi(\hat{\bm{w}}) \right)
 \right\|^2
\le
 \left\|\frac{1}{2}\left(\hat{\bm{w}} + \frac{1}{\lambda}\nabla\Phi(\hat{\bm{w}}) \right)\right\|^2 + \frac{1}{\lambda}\left(\Psi(\hat{\bm{w}}) - \Phi(\hat{\bm{w}})\right).
\end{align*}
It means that
the optimal solution
$\bm w^*$
is within a ball
with the center
$\bm m(\hat{\bm w})$
and
the radius
$r(\hat{\bm w})$,
which completes the proof. 
\end{proof}
Next, we prove Corollary~\ref{coro:output-bound}.
\begin{proof}[Proof of Corollary~\ref{coro:output-bound}]
 We show that
 the lower bound of the linear model output value 
 $\bm w^{*\top} \bm x$
 is
 $\bm x^\top \bm m(\hat{\bm w}) - \|\bm x\|r(\hat{\bm w})$
 under the constraint that
 \begin{align*}
  \|\bm w^* - \bm m(\hat{\bm w})\| \le r(\hat{\bm w}).
 \end{align*}
 To formulate this,
 let us consider the following constrained optimization problem
 \begin{align}
  \label{eq:coro-proof}
  \min_{\bm w \in \RR^d}
  ~
  \bm w^\top \bm x
  ~~~\text{s.t.}~~~
  \|\bm w - \bm m(\hat{\bm w})\|^2 \le r(\hat{\bm w})^2.
 \end{align}
 Using a Lagrange multiplier
 $\mu > 0$, 
 the problem \eq{eq:coro-proof}
 is rewritten as
  \begin{align*}
  &
  \min_{\bm w \in \RR^d}
  ~
  \bm w^\top \bm x
  ~~~\text{s.t.}~~~
  \|\bm w - \bm m(\hat{\bm w})\|^2 \le r(\hat{\bm w})^2,
  \\
  =
  &
  \min_{\bm w \in \RR^d}
  \max_{\mu > 0}
  \big(
  \bm w^\top \bm x + \mu (\|\bm w - \bm m(\hat{\bm w})\|^2 - r(\hat{\bm w})^2)
  \big)
  \\
  =
  &
  \max_{\mu > 0}
  \big(
  - \mu r(\hat{\bm w})^2
  +
  \min_{\bm w}
  \big(
  \mu \|\bm w - \bm m(\hat{\bm w})\|^2 + \bm w^\top \bm x
  \big)
  \big)
  \\
  =
  &
  \max_{\mu > 0}
  ~
  H(\mu) := 
  \big(
  - \mu r(\hat{\bm w})^2
  - \frac{\|\bm x\|^2}{4 \mu}
  +
  \bm x^\top \bm m(\hat{\bm w})
  \big),
 \end{align*}
 where
 $\mu$
 is strictly positive 
 because the constraint
 $\|\bm w - \bm m(\hat{\bm w})\|^2 \le r(\hat{\bm w})^2$
 is strictly active
 at the optimal solution.
 By letting
 $\partial H(\mu)/\partial \mu = 0$,
 the optimal
 $\mu$
 is written as
 \begin{align*}
  \mu^* := \frac{\|\bm x\|}{2 r(\hat{\bm w})}
  =
  \arg \max_{\mu > 0} ~ H(\mu).
 \end{align*}
 Substituting
 $\mu^*$
 into
 $H(\mu)$,
 \begin{align*}
  \bm x^\top \bm m(\hat{\bm w}) - \|\bm x\| r(\hat{\bm w}) = \max_{\mu > 0}~H(\mu).
 \end{align*}
 The upper bound part can be shown similarly. 
\end{proof}
\subsection*{Proof of Theorem \ref{th:plf-secure-computability} (Protocol evaluating piecewise linear function and its subderivative securely)}

First we explain the outline of the protocol of secure comparison by Veugen {\it et al.} \cite{veugen2011comparing}. The protocol returns the result of comparison $E\pkB(I_{q > 0})$ (given to party A) for the encrypted values $E\pkB(q)$ (owned by party A) with the following two steps:
\begin{itemize}
\item Party A and B obtain $q_A := R$ and $q_B := q + R$, respectively, where $R$ is a random value, and
\item Party A and B compare $q_A$ and $q_B$ with the implementation of bit-wise comparison with Paillier cryptosystem (see the original paper).
\end{itemize}
Let us denote the protocol of the latter by $SC(q_A, q_B) \rightarrow (E\pkB(I_{q_A > q_B}), E\pkA(I_{q_A > q_B}))$, that is, $SC$ is a protocol comparing two private, unencrypted values owned by two parties $q_A$, $q_B$.

The protocol for Theorem \ref{th:plf-secure-computability} is as follows:
\begin{enumerate}
\item \label{th:plf-secure-computability:unify}
	Party A computes $E\pkB(s) = E\pkB(s_A + s_B)$ from $E\pkB(s_A)$ and $E\pkA(s_B)$ as follows:
	\begin{itemize}
	\item Party B generates a random value $R\in \ZZ_{N/2}$ ($N$ is defined in {\S}\ref{sect:secure-multi-party}), then sends $E\pkA(s_B-R) = E\pkA(s_B)^{-R}$ and $E\pkB(R)$ to party A.
	\item Party A decrypts $E\pkA(s_B - R)$ and computes $E\pkB(s_A + s_B)$ as: $E\pkB(s_A + s_B) = E\pkB(s_A + s_B - R + R) = E\pkB(s_A + s_B - R)E\pkB(R) = E\pkB(s_A)^{s_B - R}E\pkB(R)$.
	\end{itemize}
	See \cite{veugen2014encrypted} for the security of the part.
\item \label{th:plf-secure-computability:randomize}
	With the similar protocol to Veugen {\it et al.}'s, party A and B obtains $p_A$ and $p_B$, respectively, where $p_A$ and $p_B$ are randomized and satisfy $p_A + p_B = s$.
\item Compute $t_j = I_{p_A + p_B > T_j}$ securely with $SC$: 
	\begin{align}
	SC(p_A, T_j-p_B) \rightarrow (E\pkB(t_j), E\pkA(t_j)), \label{eq:secure-comparison}
	\end{align}
\item Party A computes $E\pkB(o_j)$ from $E\pkB(t_j)$:
	\begin{align*}
	E\pkB(o_j) = E\pkB(t_{j-1} - t_j) = E\pkB(t_{j-1})\cdot E\pkB(t_j)^{-1}
	\end{align*}
	Party B similarly computes for $E\pkA$. The idea is shown in Figure \ref{fig:pwlp_sample}.
\item Party A computes $g_{Aj} := \alpha_j p_A + \beta_j$, and party B $g_{Bj} := \alpha_j p_B$ for all $j\in[K]$. Note that $g_{Aj} + g_{Bj} = \alpha_j s + \beta_j$.
\item Compute encrypted $g_A$ and $g_B$. Because $g_A + g_B = g(s)$ after taking $g_A = \sum_{j\in[K]} o_j g_{Aj}$ and $g_B = \sum_{j\in[K]} o_j g_{Bj}$ (see \eqref{eq:location-of-PLF-sum} in {\S}4), party A computes $E\pkB(g_A)$ as
	\begin{align*}
	E\pkB(g_A) = E\pkB\left(\sum_{j\in[K]}o_j g_{Aj}\right)
	= \prod_{j\in[K]}E\pkB(o_j)^{g_{Aj}}.
	\end{align*}
	Party B similarly computes $E\pkA(g_B)$.
\end{enumerate}

\begin{figure}
\centering
\includegraphics[width=.4\textwidth]{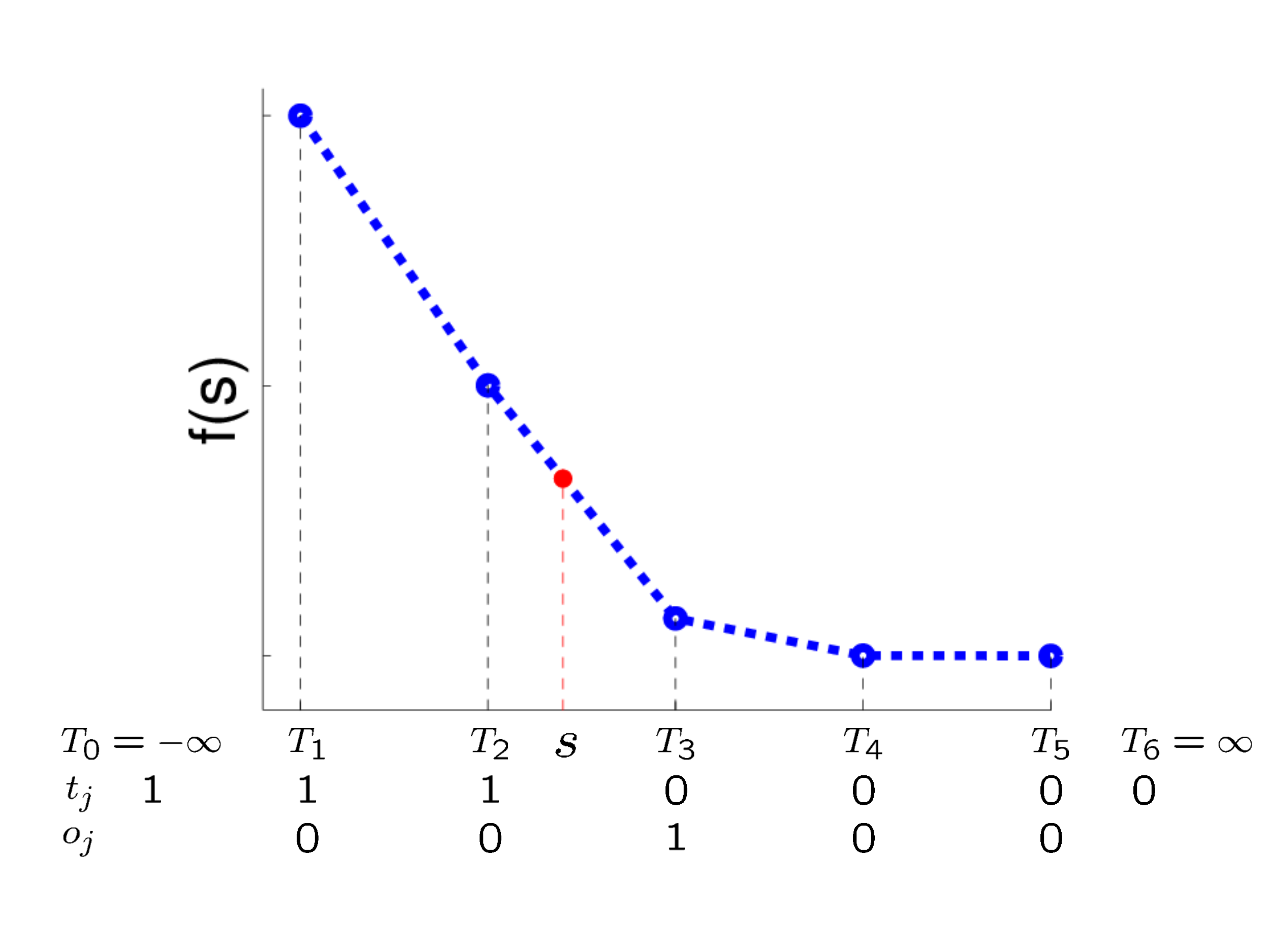}
\vspace{-2em}
\caption{Computing $o_j$ from $t_j$ in the protocol $SPLC$}
\label{fig:pwlp_sample}
\end{figure}

To obtain the subderivative $g'(s) = \sum_{j\in[K]}o_j\alpha_j$, during the protocol for Theorem \ref{th:plf-secure-computability}, party A computes $E\pkB(g'(s))$ as
	\begin{align*}
	E\pkB(g'(s)) = E\pkB\left(\sum_{j\in[K]}o_j\alpha_j\right)
	= \prod_{j\in[K]}E\pkB(o_j)^{\alpha_j}.
	\end{align*}
Party B similarly computes $E\pkA(g'(s))$.

\subsection*{Proof of Theorem \ref{th:bound_evaluation} (Protocol evaluating the upper and the lower bounds)}

Protocol \ref{pro:bound-evaluation} securely evaluates the upper bound $UB$ and the lower bound $LB$, where $\overline{sqrt}$ is an upper bound of the square root function implemented as a piecewise linear function. Note that taking $r$ larger does not lose the validity of the bounds (looser bounds are obtained) as
\begin{align*}
\tilde{\bm{x}}^\top\bm{m}
-
\overline{\|\tilde{\bm{x}}\|r}
\leq
\tilde{\bm{x}}^\top\bm{m}
-
\|\tilde{\bm{x}}\|r,
\\
\tilde{\bm{x}}^\top\bm{m}
+
\overline{\|\tilde{\bm{x}}\|r}
\geq
\tilde{\bm{x}}^\top\bm{m}
+
\|\tilde{\bm{x}}\|r.
\end{align*}

The security is proved as follows: all techniques used in the protocol are secure with the same discussions as previous. The remaining problem is that whether party A can guess $\tilde{\bm{x}}_B$ from $UB$ and $LB$. Party A can know $UB + LB = \bm{\tilde{x}}^\top\bm{m} = \bm{\tilde{x}}_A^\top\bm{m}_A + \bm{\tilde{x}}_B^\top\bm{m}_B$ and $UB - LB = \overline{\|\tilde{\bm{x}}\|r} = \overline{sqrt}((\|\tilde{\bm{x}}_A\|^2 + \|\tilde{\bm{x}}_B\|^2)(r_A^2 + r_B^2))$. However, because party A does not know $\bm{m}_A$, $\bm{m}_B$, $r_A^2$ or $r_B^2$, party A cannot guess $\tilde{\bm{x}}_B$ either\footnote{If this protocol is conducted for many enough $\tilde{\bm{x}}$, because party A knows $\tilde{\bm{x}}_A$, party A can also know $\bm{m}_A$ by solving a system of linear equations, and thus know $\bm{\tilde{x}}_B^\top\bm{m}_B$. This, however, does not lead party A to guess separate $\bm{\tilde{x}}_B$ or $\bm{m}_B$ because they are both private for party B.}.

\floatname{algorithm}{Protocol}
\begin{algorithm*}[tp]
\caption{Protocol for bound computation}
\begin{description}
\item[public] $\overline{sqrt}$
\item[Input of A] $E_{pk_B}(\bm{m}_A),E_{pk_B}(r_A^2),\tilde{\bm{x}}_A$
\item[Input of B] $E_{pk_B}(\bm{m}_B),E_{pk_B}(r_B^2),\tilde{\bm{x}}_B$
\item[Output of A] $UB,LB$
\item[Output of B] $\emptyset$
\begin{enumerate}
\item[Step 1.] Party A and B computes:

		party A:$
		E_{pk_B}(\tilde{\bm{x}}_A^\top\bm{m}_A),
		E_{pk_A}(\|\tilde{\bm{x}}_A\|^2)
		$
		
		party B:$
		E_{pk_A}(\tilde{\bm{x}}_A^\top\bm{m}_B),
		\|\tilde{\bm{x}}_B\|^2
		$
\item[step 2.] Party A sends $
		E_{pk_B}(\tilde{\bm{x}}_A^\top\bm{m}_A),
		E_{pk_A}(\|\tilde{\bm{x}}_A\|^2),
		E_{pk_B}(r_A^2)
		$ to B.
		\\
		Party B obtains $
		E_{pk_A}(\tilde{\bm{x}}^\top\bm{m}),
		E_{pk_A}(\|\tilde{\bm{x}}\|^2),
		E_{pk_A}(r^2)
		$.
		\\ // The similar manner to the protocol for Theorem \ref{th:plf-secure-computability}, step \ref{th:plf-secure-computability:unify}
\item[step 2.] Compute $E_{pk_B}(\|\tilde{\bm{x}}\|^2 r^2)$ using the protocol for multiplication in \cite{nissim2006communication}.
\item[step 3.] Compute $\overline{\|\tilde{\bm{x}}\|r}$ with $SPL$: \\
		$SPL(E_{pk_B}(0),E_{pk_B}(\|\tilde{\bm{x}}\|^2r^2))\to(q_A,q_B)$
		//$q_A+q_B = \overline{\|\tilde{\bm{x}}\|r}$
\item[step 4.] Party A sends $E_{pk_B}(q_B)$ to B.
		\\
		Party B obtains $q_B$ and thus $E_{pk_A}(\overline{\|\tilde{\bm{x}}\|r}) = E_{pk_A}(q_A)^{q_B}$.
		\\
		Party B computes the followings and sends to A.\\
		$
		E_{pk_A}(UB)
		\leftarrow
		E_{pk_A}(\tilde{\bm{x}}^\top\bm{m}
		+
		\overline{\|\tilde{\bm{x}}\|r})
		$
		\\
		$
		E_{pk_A}(LB)
		\leftarrow
		E_{pk_A}(\tilde{\bm{x}}^\top\bm{m}
		-
		\overline{\|\tilde{\bm{x}}\|r})
		$		
\item[step 5.] Party A obtains $UB,LB$ by decrypting them.
\end{enumerate}

\end{description}
\label{pro:bound-evaluation}
\end{algorithm*}

\subsection*{Example Protocol for the Logistic Regression}

We show the detailed implementation of secure ball computation ($SBC$, Theorem \ref{th:SAG}) in {\S}4 for the logistic regression, including how to use the secure computation of piecewise linear functions ($SPL$, Theorem \ref{th:plf-secure-computability}).

For the logistic regression ({\S}\ref{sect:problem-statement}), $\cY = \{-1, +1\}$, and we take $u(s) = \log(1 + \exp(-s))$, $s = \bm{x}^\top\bm{w}$ and $v(y, \bm{x}^\top\bm{w}) = -y\bm{x}^\top\bm{w}$ in Theorem \ref{th:SAG}.

To apply this for $SPL$, we set $E\pkB(s_A) := E\pkB(\bm{x}_A^\top\vhat{w}_A)$ and $E\pkA(s_B) := E\pkA(\bm{x}_B^\top\vhat{w}_B)$ since we assume party A and B knows $E\pkB(\vhat{w}_A)$ and $E\pkA(\vhat{w}_B)$, respectively. Note that $s_A + s_B = s$ because $\bm{x} = [\bm{x}_A^\top ~ \bm{x}_B^\top]^\top$ and $\bm{w} = [\bm{w}_A^\top ~ \bm{w}_B^\top]^\top$. Take piecewise linear functions $\underline{u}(s)$ and $\overline{u}(s)$ as lower and upper bounds of $u(s)$, respectively. With it, we can compute $\phi$, $\psi$ and $\nabla\phi$ in $SAG$ as follows:
\begin{align*}
\psi|_{\bm{w}=\vhat{w}} - &\phi|_{\bm{w}=\vhat{w}} = \overline{u}(\bm{x}^\top\vhat{w}) - \underline{u}(\bm{x}^\top\vhat{w}), \\
\nabla\phi|_{\bm{w}=\vhat{w}} &= \underline{u}'(\bm{x}^\top\vhat{w})\left.\frac{\partial s}{\partial\bm{w}}\right|_{\bm{w}=\vhat{w}} + \left.\frac{\partial v}{\partial\bm{w}}\right|_{\bm{w}=\vhat{w}} \\
	&= \underline{u}'(\bm{x}^\top\vhat{w})\left.\frac{\partial}{\partial\bm{w}}\bm{x}^\top\bm{w}\right|_{\bm{w}=\vhat{w}} + \left.\frac{\partial}{\partial\bm{w}}y\bm{x}^\top\bm{w}\right|_{\bm{w}=\vhat{w}} \\
	&= (\underline{u}'(\bm{x}^\top\vhat{w}) + y)\bm{x},
\end{align*}
which are all computable with $SPL$.

After these preparations, we can conduct the protocol $SBC$ as {\bf Protocol} \ref{pro:SAG}.

\begin{spacing}{1}
\floatname{algorithm}{Protocol}
\begin{algorithm*}[tp]
\caption{Secure Ball Computation protocol (SBC)}
\begin{description}
\item[Public] $\phi := \underline{u}(s) - y\bm{x}^\top\bm{w}$, $\psi := \overline{u}(s) - y\bm{x}^\top\bm{w}$
\item[Input from A] $\{\bm{x}_{iA}\}_{i\in[n]},E\pkB(\vhat{w}_A)$
\item[Input from B] $\{\bm{x}_{iB},y_i\}_{i\in[n]},E\pkA(\vhat{w}_B)$
\item[Output to A] $E\pkB(\bm{m}_A),E\pkB(r_A^2)$
\item[Output to B] $E\pkA(\bm{m}_B),E\pkA(r_B^2)$ (where $r_A^2+r_B^2 = r^2$)
\end{description}
\begin{enumerate}
\item[Step1] Party B sends $E\pkB(\bm{y})$ to party A.
\item[Step2] Party A and B compute encrypted $\Phi$, $\Psi$ and $\nabla\Phi$ at $\bm{w} = \vhat{w}$.
\item[] Party A does:
	\begin{enumerate}
	\item[]
		for $i = 1$ to $n$:
		\\
		\quad $SPLC(E\pkB(\bm{x}_{iA}^\top\vhat{w}_A), E\pkA(\bm{x}_{iB}^\top\vhat{w}_B))\rightarrow 
		\left(E\pkB(\underline{u}^*_{iA}), E\pkA(\underline{u}^*_{iB})\right)$,
		\\
		\quad $SPLC(E\pkB(\bm{x}_{iA}^\top\vhat{w}_A), E\pkA(\bm{x}_{iB}^\top\vhat{w}_B))\rightarrow
		\left(E\pkB(\overline{u}^*_{iA}), E\pkA(\overline{u}^*_{iB})\right)$
	\item[]
		// Note: $E(a)^{1/n}$ is in reality computed as $E(a)^{M/n}$, \\
		// \quad where $M$ is the magnification constant. \\
		// Note: $E(a)^{\bm{\eta}}$ ($\bm{\eta}$: a vector) means $[E(a)^{\eta_1} ~ E(a)^{\eta_2} ~ \cdots]^\top$.
	\item[]
		$E\pkB(\Psi_A - \Phi_A)$\\
		$
		\gets E\pkB\left(
			\frac{1}{n}\sum_{i\in [n]}[\overline{u}^*_{iA} - y_i\bm{x}_{iA}^\top\vhat{w}_A]
			- \frac{1}{n}\sum_{i\in [n]}[\underline{u}^*_{iA} - y_i\bm{x}_{iA}^\top\vhat{w}_A]
			\right)
		$\\
		$= E\pkB\left(\frac{1}{n}\sum_{i\in [n]}[\overline{u}^*_{iA} - \underline{u}^*_{iA}]\right)$\\
		$
		= \left[
			\prod_{i\in [n]}E\pkB\left(\overline{u}^*_{iA}\right)
			\right]^{1/n}
			\left[
			\prod_{i\in [n]}E\pkB\left(\underline{u}^*_{iA}\right)
			\right]^{-1/n}
		$ \\
		// $\Phi_A+\Phi_B = \Psi$, $\Psi_A+\Psi_B = \Psi$
	\item[]
		$E\pkB(\nabla\Phi_A)
		\leftarrow E\pkB\left(\frac{1}{n}\sum_{i\in [n]}[\underline{u}'(\bm{x}_i^\top\bm{w}) - y_i]\bm{x}_{iA}\right)$
		\\
		\quad$= \left[\prod_{i\in [n]}E\pkB(\underline{u}'(\bm{x}_i^\top\bm{w}))E\pkB(y_i)^{-1}\right]^{(1/n)\bm{x}_{iA}}$\\
		// $[\nabla\Phi_A^\top,\nabla\Phi_B^\top]^\top =\nabla\Phi$
	\item[]
		// Note: $\underline{\alpha}_j$ and $\underline{o}_j$ means $\alpha_j$ and $o_j$ for $\underline{u}$ \\
		// \quad (see the subderivative computation in Theorem \ref{th:plf-secure-computability}).
	\end{enumerate}
\item[] Party B does the similar.
\item[Step3] Party A and B compute encrypted $\bm{m}$ and $r$.
\item[] Party A does:
	\begin{enumerate}
	\item[] $E\pkB(\bm{m}_A) \leftarrow E\pkB(\vhat{w}_A-\frac{1}{\lambda}\nabla\Phi_A)^{1/2}$
	\item[] Compute $E\pkB(\|\frac{1}{2}(\vhat{w}_A+\frac{1}{\lambda}\nabla\Phi_A)\|^2)$ from $E\pkB(\frac{1}{2}(\vhat{w}_A+\frac{1}{\lambda}\nabla\Phi_A)) = [E\pkB(\vhat{w}_A)E\pkB(\nabla\Phi_A)^{1/\lambda}]^{1/2}$ using the multiplication protocol in \cite{nissim2006communication}.
	\item[] $E\pkB(r^2_A) \leftarrow E\pkB(\|\frac{1}{2}(\vhat{w}_A+\frac{1}{\lambda}\nabla\Phi_A)\|^2) \cdot (E\pkB(\Psi_A)\cdot E\pkB(\Phi_A)^{-1})^{1/\lambda}$
	\end{enumerate}
\item[] Party B does the similar.
\end{enumerate}
\label{pro:SAG}
\end{algorithm*}
\end{spacing}

{\rema
In the description of the protocol, we omitted the maginification constant $M$ ({\S}\ref{sect:secure-multi-party}) for simplicity. We have to notice that, summing two values magnified by $M^a$ and $M^b$, we get a value magnified by $M^{\max\{a, b\}}$. Similarly, multiplying two values magnified by $M^a$ and $M^b$, we get a value magnified by $M^{a+b}$. In the protocol, when the original data is magnified by $M$, then the final result is magnified by $M^{12}$. So we have to adjust $M$ so that $M^{12}$ times the final result does not exceed the domain of Paillier cryptosystem $\ZZ_N$.
}


\end{document}